\documentclass[sigconf]{aamas}

\usepackage{balance} % for balancing columns on the final page
\usepackage{multirow}
\usepackage{subcaption}
\usepackage{enumitem}

\usepackage{bm}

\def\equaladvising{%
  \ifnum\value{eqfn}=0%
    \footnote{Equal advising.}%
    \setcounter{eqfn}{\value{footnote}}%
  \else%
    \footnotemark[\value{eqfn}]%
  \fi%
}%

\setcopyright{ifaamas}
\acmConference[AAMAS '26]{Proc.\@ of the 25th International Conference
on Autonomous Agents and Multiagent Systems (AAMAS 2026)}{May 25 -- 29, 2026}
{Paphos, Cyprus}{C.~Amato, L.~Dennis, V.~Mascardi, J.~Thangarajah (eds.)}
\copyrightyear{2026}
\acmYear{2026}
\acmDOI{}
\acmPrice{}
\acmISBN{}

\acmSubmissionID{1493}

\title[AAMAS-2026 Formatting Instructions]{Offline Safe Policy Optimization From Heterogeneous Feedback}

\author{Ze Gong}
\affiliation{
  \institution{Shenzhen Institute of Advanced Technology (SIAT), CAS}
  \city{Shenzhen}
  \country{China}}
\email{ze.gong@siat.ac.cn}

\author{Pradeep Varakantham}
\authornote{Equal advising.}
\affiliation{
  \institution{Singapore Management University}
  \city{Singapore}
  \country{Singapore}}
\email{pradeepv@smu.edu.sg}

\author{Akshat Kumar}
\authornotemark[1]
\affiliation{
  \institution{Singapore Management University}
  \city{Singpaore}
  \country{Singapore}}
\email{akshatkumar@smu.edu.sg}

\begin{abstract}
  Offline Preference-based Reinforcement Learning (PbRL) learns rewards and policies aligned with human preferences without the need for extensive reward engineering and direct interaction with human annotators. However, ensuring safety remains a critical challenge across many domains and tasks. Previous works on safe RL from human feedback (RLHF) first learn reward and cost models from offline data, then use constrained RL to optimize a safe policy. While such an approach works in the contextual bandits settings (LLMs), in long horizon continuous control tasks, errors in rewards and costs accumulate, leading to impairment in performance when used with constrained RL methods. To address these challenges, (a) instead of indirectly learning policies (from rewards and costs), we introduce a framework that learns a policy directly based on pairwise preferences regarding the agent's behavior in terms of rewards, as well as binary labels indicating the safety of trajectory segments; (b) we propose \textsc{PreSa} (Preference and Safety Alignment), a method that combines preference learning module with safety alignment in a constrained optimization problem. This optimization problem is solved within a Lagrangian paradigm that directly learns reward-maximizing safe policy \textit{without explicitly learning reward and cost models}, avoiding the need for constrained RL; (c) we evaluate our approach on continuous control tasks with both synthetic and real human feedback. Empirically, our method successfully learns safe policies with high rewards, outperforming state-of-the-art baselines, and offline safe RL approaches with ground-truth reward and cost.
\end{abstract}


\newcommand{\BibTeX}{\rm B\kern-.05em{\sc i\kern-.025em b}\kern-.08em\TeX}

\begin{document}

%%% The following commands remove the headers in your paper. For final 
%%% papers, these will be inserted during the pagination process.

\pagestyle{fancy}
\fancyhead{}

%%% The next command prints the information defined in the preamble.

\maketitle 

%%%%%%%%%%%%%%%%%%%%%%%%%%%%%%%%%%%%%%%%%%%%%%%%%%%%%%%%%%%%%%%%%%%%%%%%

\section{Introduction}

To align the intelligent agents with human values, preference-based reinforcement learning (PbRL)~\cite{wirth2017survey} (also known as Reinforcement Learning from Human Feedback (RLHF)) has emerged as a powerful learning paradigm by training the agent's policy from human pairwise preference over agent behaviors, without the need for explicit reward engineering.
Offline PbRL~\cite{shin2021offline} further improves feedback efficiency by avoiding costly online interactions with human annotators. It has recently achieved remarkable success when applied to policy learning in control tasks~\cite{christiano2017deep,lee2021pebble,park2022surf,huquery,hejnacontrastive} and finetuning large language models (LLMs)~\cite{rafailov2024direct,zhao2023slic,ethayarajh2024kto}, demonstrating the ability to learn rewards and policies that are consistent with human preferences.
However, ensuring safety remains a significant challenge for PbRL. For instance, a robotic agent must avoid collisions and operate safely around humans during a control task. In LLM fine-tuning, the agent must avoid generating harmful or socially inappropriate content.
These examples highlight that alignment with human preferences alone is insufficient. Agents must also be explicitly aligned with human safety considerations. Hence, in this paper, we aim to learn a policy that is not only consistent with human pairwise preferences but also {\em aligned with human safety considerations}.
% \begin{itemize}
%     \item  consistent with human pairwise preferences, and
%     \item  aligned with human safety considerations.
% \end{itemize}

Recently, safe RLHF~\cite{daisafe} has been proposed to fine-tune LLMs for generating helpful and harmless responses. It assumes access to pairwise human preferences for both rewards and costs, as well as safety labels for each response. By learning reward and cost models from human feedback, it then optimizes a policy via constrained RL. 
In this paper, a key distinction is that we focus on continuous control tasks. While LLM fine-tuning is typically framed as a contextual bandit problem, continuous control tasks are inherently sequential decision-making problems in RL. Unlike contextual bandits, where decisions are made in isolation, continuous control requires reasoning over multi-step long-horizon dependencies, balancing immediate and future rewards, and handling compounding errors. These challenges make direct adaptation of existing safe RLHF methods inadequate, necessitating more sophisticated techniques to ensure safety and optimality in sequential environments. 
Additionally, safe RLHF method typically follows the two-phase PbRL learning paradigm: in the first phase, reward and cost models are learned from human feedback; in the second phase, constrained RL is applied to optimize a policy based on the learned models. Such decoupling, however, introduces inherent sources of suboptimality. Inaccuracies in the reward and cost learning phase can propagate to the policy, while rollout sampling or bootstrapping in the constrained RL phase may introduce further optimization challenges.

% Such differences make it infeasible to directly apply the safe RLHF approach to a continuous control task. 

To overcome these limitations, we introduce the framework of \textit{Offline Safe Policy Optimization from Heterogeneous Feedback (Offline Safe POHF)}, in which a policy is learned directly from offline datasets (without first learning explicit reward and cost models) that include two types of feedback: (a) \textit{pairwise preferences} over agent behaviors with respect to rewards, and (b) \textit{binary safety labels} indicating whether each trajectory segment is safe or not. Notably, unlike prior work~\cite{daisafe}, our approach does not require pairwise preferences over the agent behaviors on cost, as such annotations are often scarce and costly to collect in practice~\cite{ethayarajh2024kto}. 
To address the aforementioned challenges, we propose \textsc{PreSa} (Preference and Safety Alignment), a novel algorithm that learns policies directly from heterogeneous feedback without explicit reward or cost modeling. Our approach begins by developing a safety alignment module that learns to generate safe behaviors using only binary safety labels. This component draws inspiration from Human-Aware Losses (HALOs)~\cite{ethayarajh2024kto}, adapting their principles of desirable decision-making to continuous control settings.
Next, we transform the safety alignment objective into an optimization constraint by showing that it implicitly defines a feasible set of policies. This safety alignment module is then integrated with the preference alignment module, and the overall policy is optimized using the Lagrangian method directly on the offline dataset. Consequently, we derive a fully supervised learning objective that eliminates the need for explicit reward and cost model learning, as well as the conventional constrained RL phase.

%To evaluate our approach, we first constructed a new dataset based on the well-established offline safe RL benchmark, DSRL~\cite{liu2023datasets}. We synthesized human feedback using the ground truth rewards and costs provided in the original offline dataset as they are not accessible during training in our setting. The dataset includes 29 continuous control tasks across two prevailing domains, providing a robust platform for testing Offline Safe POHF and facilitating further study. Additionally, we conduct experiments with real human feedback that is collected under various simulated autonomous driving scenarios~\cite{highway-env}. We compare our method against baselines from Offline Safe POHF approaches, including a variant of safe RLHF, as well as offline safe RL~\cite{liu2023datasets,xu2022constraints,liu2023constrained}, which use ground truth reward and cost. The results suggest that, compared to the baselines, our method effectively learns a policy that achieves high rewards while adhering to constraints implicitly encoded in the human feedback. 

The contributions of this paper are threefold. First, we introduce the \textit{Offline Safe POHF} framework for continuous control tasks, enabling policy learning from offline datasets that include both pairwise behavioral preferences with respect to rewards and binary safety annotations. Second, we propose \textsc{PreSa}, a unified optimization formulation that integrates preference and safety alignment through a constrained objective, solved efficiently using the Lagrangian method. This formulation removes the dependency on reward and cost model estimation, as well as the need for an additional constrained RL stage. Third, we conduct comprehensive evaluations on both synthetic and real human feedback datasets, demonstrating that \textsc{PreSa} achieves effective alignment and robust safety performance in continuous control tasks.

%%%%%%%%%%%%%%%%%%%%%%%%%%%%%%%%%%%%%%%%%%%%%%%%%%%%%%%%%%%%%%%%%%%%%%%%

\section{Related Work}

\subsection{Preference-based Reinforcement Learning}
To avoid reward engineering which requires expert knowledge and may suffer from negative effects of reward misspecification~\cite{pan2022effects}, Preference-based Reinforcement Learning (PbRL) (also known as Reinforcement Learning from Human Feedback (RLHF))  provides a promising paradigm to learn a policy from human feedback~\cite{wirth2017survey}. There have been significant advances recently, both for control tasks~\cite{christiano2017deep,lee2021pebble,park2022surf,liu2022meta,shin2021offline,hejna2024inverse,kang2023beyond,hejnacontrastive} and LLM finetuning~\cite{ziegler2019fine,stiennon2020learning,ouyang2022training,bai2022training,zhao2023slic,rafailov2024direct,ethayarajh2024kto}. However, these works have primarily focussed on preference alignment and have not considered safety in policy learning. 
Safe RLHF~\cite{daisafe} was recently proposed for fine-tuning LLMs to generate helpful and avoid harmful responses. However, it is designed for contextual bandit settings, where decisions are made independently, and does not directly handle the long-horizon dependencies and compounding errors inherent in RL tasks. Moreover, Safe RLHF relies on explicitly learned reward and cost models for constrained RL, which can introduce inaccuracies that hinder policy optimization.
% In this paper, we introduce Offline Safe Policy Optimization from Heterogeneous Feedback (Offline Safe POHF), a framework tailored for continuous control tasks. Offline Safe POHF learns safe policies directly from pairwise preferences and binary safety labels, avoiding explicit reward or cost modeling and addressing the challenges of safety in sequential decision-making
% Safe RLHF~\cite{daisafe} was recently introduced specifically for fine-tuning LLMs to generate helpful and avoid harmful responses. However, this approach is tailored to contextual bandit settings, where decisions are made independently at each step, rather than reinforcement learning (RL) settings, which involve multi-step sequential decision-making. As a result, Safe RLHF does not directly address the challenges posed by long-horizon dependencies and compounding errors inherent in RL-based tasks. Additionally, Safe RLHF explicitly learns reward and cost models, which are then used to optimize policies via constrained RL. This reliance on learned models introduces potential inaccuracies that can negatively impact policy optimization.
% In this paper, we introduce a new framework, Offline Safe Policy Optimization from Heterogeneous Feedback (POHF), designed specifically for continuous control tasks. Unlike Safe RLHF, our approach aims to learn a safe policy directly from pairwise preference and binary safety labels, circumventing the need for explicit reward and cost modeling and addressing the unique challenges of safety in sequential decision-making.

\subsection{Offline Safe Reinforcement Learning}
Offline safe RL provides a practical framework for learning safe policies using pre-collected datasets~\cite{liu2023datasets}. Previous work has addressed this problem through various approaches such as sequential modeling~\cite{liu2023constrained}, distribution correction estimation~\cite{leecoptidice}, Q-learning~\cite{xu2022constraints}, feasible region identification~\cite{zhengsafe} and trajectory classification~\cite{gong2024offline}. These offline datasets typically consist of agent rollouts with ground truth rewards and costs for each timestep, often designed by experts to ensure the data quality for policy learning. However, in many complex real-world domains and tasks, it is difficult to manually design reward and cost functions that accurately reflect human values. Therefore, with the advances of PbRL, we propose to learn policies from pre-collected human feedback, replacing ground truth rewards and costs with human pairwise preferences regarding agent behavior for rewards and binary labels indicating whether the behavior is safe or not.

%%%%%%%%%%%%%%%%%%%%%%%%%%%%%%%%%%%%%%%%%%%%%%%%%%%%%%%%%%%%%%%%%%%%%%%%

\section{Problem Definition}
We formulate our problem within the framework of a Markov Decision Process (MDP), denoted as a tuple $\mathcal{M}=(\mathcal{S}, \mathcal{A}, \mathcal{P}, r, \rho_0, \gamma)$, where $\mathcal{S}$ denotes the state space, $\mathcal{A}$ the action space, and $\mathcal{P}(s^\prime|s,a)$ the transition dynamics. The function $r: \mathcal{S} \times \mathcal{A} \rightarrow \mathbb{R}$ represents the reward function. $\rho_0$ is the initial state distribution, and $\gamma$ is the discount factor. 

To incorporate safety considerations, Safe RL extends the MDP to a constrained MDP (CMDP), expressed as $\mathcal{M}\cup \mathcal{C}$, where $\mathcal{C}=\{(c_i, b_i)\}_{i=0}^m$. Each $c_i$ represents a cost function associated with a safety constraint, and $b_i$ denotes the corresponding cost threshold. 
The discounted cumulative reward under a policy $\pi$ is defined as $V_\pi^r(s)= \mathbb{E}_\pi \left[\sum_{t=0}^\infty \gamma^t r(s_t, a_t)|s_0=s \right]$, and the discounted cumulative cost for constraint $i$ is defined similarly as $V_\pi^{c_i}(s)= \mathbb{E}_\pi \left[\sum_{t=0}^\infty \gamma^t c_i(s_t, a_t)|s_0=s \right]$. The objective of a standard constrained RL problem can thus be expressed as:
\begin{equation}
    \max_\pi \mathbb{E}_{s\sim\rho_0} \left[V_\pi^r(s) \right] \quad \text{s.t.} \quad \mathbb{E}_{s\sim\rho_0} \left[V_\pi^{c_i}(s) \right] \leq b_i, \forall i
    \label{eq:srl}
\end{equation}

While this formulation provides a principled way to ensure safety through explicit constraints, its practical application often relies on access to accurate reward and cost functions as well as extensive online interaction for policy optimization. In many real-world domains, however, such assumptions are unrealistic. Rewards and costs are rarely observable, and online data collection can be unsafe or prohibitively expensive.

\begin{figure}
    \centering
    \includegraphics[width=\columnwidth]{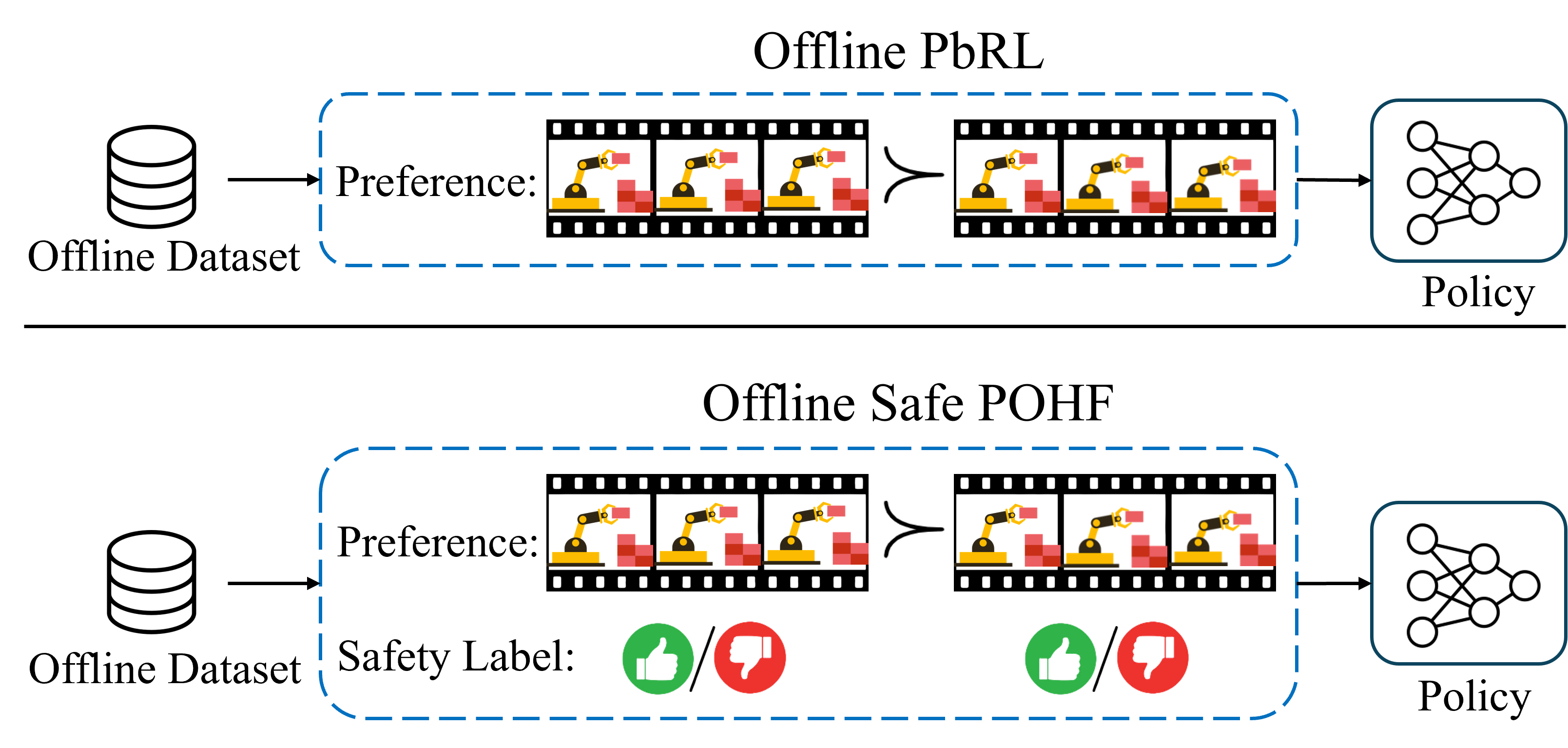}
    \caption{Offline Safe POHF versus Offline PbRL for control tasks. Besides pairwise preference between agent's trajectory segments, the dataset of Offline Safe POHF additionally includes binary safety labels of each segment, which is used to align the policy with implicit safety constraints.}
    \label{fig:problem_setting}
\end{figure}

\subsection{Offline Safe Policy Optimization from Heterogeneous Feedback (Offline Safe POHF)}

In this paper, we consider the offline setting, where the agent has access only to a static dataset collected from previous interactions, without access to the true reward or cost functions. 
% Instead, the dataset provides contains heterogeneous feedback, consisting of: a) \textit{pairwise preferences} between trajectory segments reflecting reward-related desirability, and b) \textit{binary safety labels} indicating whether each trajectory segment is safe or unsafe. 
This setup differs from standard offline PbRL by incorporating both preference and safety signals, we formalize it as follows.
\begin{definition}[Offline Safe Policy Optimization from Heterogeneous Feedback, (Offline Safe POHF)]
    Given a CMDP with unknown reward and cost functions, the Offline Safe POHF problem is defined by an offline dataset, 
    \begin{equation}
        D = \left\{ \left(\sigma^+, y^+, \sigma^-, y^- \right) \right\},
    \end{equation}
    where $\sigma=(s_1, a_1, s_2, a_2, \cdots, s_k, a_k)$ is a trajectory segment of length $k$. The dataset contains:
    \begin{enumerate}[label=(\alph*)]
        \item \textit{Pairwise preferences}: $\sigma^+ \succ \sigma^-$, indicating that $\sigma^+$ is preferred to $\sigma^-$ with respect to reward; and
        \item \textit{Binary safety labels}: $y \in \{-1, +1\}$, where $y = +1$ denotes a safe segment and $y = -1$ denotes an unsafe one. $y^+$ and $y^-$ are the safety labels associated with $\sigma^+$ and $\sigma^-$, respectively.
    \end{enumerate}
    The goal is to learn a policy that maximizes expected reward while satisfying safety constraints, relying solely on the feedback information contained in $D$.
\end{definition}

% Unlike Safe RLHF~\cite{daisafe}, the Offline Safe POHF setting assumes that human preferences are provided only with respect to reward, without requiring pairwise comparisons for safety or cost, which are typically difficult and expensive to obtain~\cite{ethayarajh2024kto}.
An overview of this setting is illustrated in Figure~\ref{fig:problem_setting}.
To solve the problem, we decompose the overall learning objective into two components: (a) a \textit{preference alignment module} for learning reward-consistent behavior, and (b) a \textit{safety alignment module} for constraining unsafe behavior.
We then integrate these components into a unified constrained optimization framework, which serves as the foundation of our proposed Offline Safe POHF algorithm.
The following sections present each module in detail and describe how they are combined into a single, unified learning objective.

\section{Preliminaries}

\subsection{Contrastive Preference Learning (CPL)}

To learn a policy that aligns with human pairwise preferences, we consider a preference dataset $D_\text{pref}=\left\{\left(\sigma^+, \sigma^-\right)\right\}$. The conventional PbRL learning paradigm offers a class of methods that typically involve two phases~\cite{christiano2017deep}. In the first phase, it assumes the human preference model is distributed according to cumulative rewards and a reward model is learned by optimizing the negative log-likelihood of human preferences. In the second phase, a policy is trained by using the learned reward model. However, there exist two main challenges of applying this conventional PbRL paradigm. First, the assumption that human preferences are reward-based has been criticized as incompetent at capturing true human preference~\cite{knox2024models}. Second, the RL training in the second phase would suffer from substantial computational difficulties~\cite{hejnacontrastive,rafailov2024direct}.

To address the issues, Contrastive Preference Learning (CPL) \cite{hejnacontrastive} has been proposed as a method to learn a policy directly without the need for reward learning and RL. In CPL, human preference is modeled using regret instead of reward~\cite{knox2024models}. By leveraging the equivalence of negated regret and the discounted sum of optimal advantages, the regret-based preference model is given by:
\begin{equation}
    P\left[\sigma^+ \succ \sigma^-\right] = 
    \frac{\exp \sum_{\sigma^+} \gamma^t A_r^*\left(s_t^+, a_t^+\right)}{\exp \sum_{\sigma^+} \gamma^t A_r^*\left(s_t^+, a_t^+\right) + \exp \sum_{\sigma^-} \gamma^t A_r^*\left(s_t^-, a_t^-\right)} \label{eq:cplscore}
\end{equation}
where $A_r^*\left(s_t, a_t\right)$ is the optimal advantage function of a single timestep $(s_t, a_t)$ with respect to a reward model $r$; shorthand ``$+$" and ``$-$" index the states and actions of segments $\sigma^+$ and $\sigma^-$. According to the principle of maximum entropy~\cite{ziebart2010modeling,hejnacontrastive}, the optimal advantage function $A_r^*(s,a)$ for a Kullback–Leibler divergence (KL)-regularized RL problem can be expressed in terms of the optimal policy $\pi^*$ as:
\begin{equation}
    A_r^*(s,a) = \alpha \log \frac{\pi^*(a|s)}{\pi_\text{ref}(a|s)}
    \label{eq:adv}
\end{equation}
where $\pi_\text{ref}$ is a reference policy used to regularize $\pi^*$, and $\alpha$ is a temperature parameter that determines the extent to which the reference policy $\pi_\text{ref}$ influences $\pi^*$. Consequently, the loss function for learning a parameterized policy $\pi_\theta$ (i.e., an approximation of $\pi^*$) is formulated by optimizing the negative log-likelihood of human preferences: 
$L_\text{CPL-KL}\left(\pi_\theta, D_\text{pref}\right) = \mathbb{E}_{(\sigma^+, \sigma^-) \sim  D_\text{pref}} \left[ - \log P_{\pi_\theta}\left[\sigma^+ \succ \sigma^-\right] \right]$.
% \begin{equation}
%     L_\text{CPL-KL}(\pi_\theta, D_\text{pref}) =
%     \mathbb{E}_{(\sigma^+, \sigma^-) \sim  D_\text{pref}} \left[ - \log P_{\pi_\theta}[\sigma^+ \succ \sigma^-] \right]
% \end{equation}
With Equation~\ref{eq:adv}, the loss function $L_\text{CPL-KL}\left(\pi_\theta, D_\text{pref}\right)$ presents a closed-form formulation for directly learning a policy that aligns with human preferences.

\subsection{Prospect Theory and Human-Aware Losses}

Prospect theory is a behavioral economics framework that explains how individuals evaluate gains and losses in uncertain events, often in an asymmetric manner. There are three key principles when modeling human decision-making through the lens of prospect theory, (a) the use of a reference point to determine relative gains or losses, (b) concavity in relative gains (i.e., diminishing sensitivity as they move farther from the reference point); and (c) loss aversion, meaning individuals are more sensitive to losses compared to gains.

Building on prospect theory, Ethayarajh et al.~\cite{ethayarajh2024kto} introduced Human-Aware Losses (HALOs) to better understand how RLHF fine-tunes large language models (LLMs). HALOs provide a structured way to guide model training by distinguishing between gains and losses in human preference data.
A function $f$ is classified as a HALO if it follows this form:
\begin{equation}
    f(\pi_\phi, \pi_\text{ref}^\text{LLM}) = 
    \mathbb{E}_{x,y\sim D_\text{LLM}} \left[a_{x,y} v\left(r_{\pi_\phi}(x,y) - 
    \mathbb{E}_Q \left[r_{\pi_\phi} \left(x,y^\prime \right)\right]\right)\right] + C_{D_\text{LLM}} \label{eq:llmhalo}
\end{equation}
where {$\pi_\phi$} is the model being fine-tuned, which generates a response $y$ based on an input $x$.
$\pi_{ref}^{LLM}$ is a reference model, used as a baseline for comparison. $a_{x,y}$ is $+1$ for desirable responses and $-1$ for undesirable ones, capturing whether a response is perceived as a gain or a loss. The implied reward $r_{\pi_\phi}(x,y)$ is computed using the likelihood ratio of the response under the fine-tuned model versus the reference model, scaled by a normalizing factor.
% $l(y)$.  
$v(\cdot)$ is a function that ensures gains are weighted differently from losses, aligning with human preference biases. $Q(Y'|x)$ represents a distribution of alternative responses to define a reference point for comparison. $D_{LLM}$ is the dataset of human feedback, and $C_{D_{LLM}}$ is a dataset-specific constant. 

As the LLM is fine-tuned by minimizing this loss function, it learns to assign higher rewards to desirable responses and lower rewards to undesirable ones. This means that over time, the model becomes more likely to generate responses that align with human preferences. Empirically, HALO-based approaches match or outperform traditional RLHF methods across different LLM scales~\cite{ethayarajh2024kto}.
In this work, we focus on multi-step continuous control tasks, where directly applying HALO is non-trivial due to the challenges of temporal credit assignment. In the next section on safety alignment, we derive a tailored loss function suitable for continuous control settings.

%%%%%%%%%%%%%%%%%%%%%%%%%%%%%%%%%%%%%%%%%%%%%%%%%%%%%%%%%%%%%%%%%%%%%%%%

\section{Methodology}

In the Offline Safe POHF setting, the objective is to learn a policy given pairwise preferences and safety annotations $y\in \{-1, +1\}$ for each trajectory segment. To achieve this objective, we make the following novel contributions:
(a) Provide an algorithm for ensuring safety alignment for continuous control tasks in Section~\ref{sec:safety_alignment}.
(b) Combine the safety alignment with preference optimization in an interpretable and principled manner in Section~\ref{sec:integration}.
% \begin{enumerate}
%     \item We first provide an algorithm for ensuring safety alignment for continuous control tasks (Section~\ref{sec:safety_alignment}).
%     \item We then combine the safety alignment with preference optimization in an interpretable and principled manner in Section~\ref{sec:integration}.
% \end{enumerate}
%Setting aside pairwise preferences, we can frame safety alignment as policy optimization based solely on binary feedback. %~\citet{ethayarajh2024kto} recently introduced a novel approach (i.e., KTO) for finetuning LLMs using binary safety feedback. However, KTO is tailored specifically for contextual bandit settings. In~\cref{sec:safety_alignment}, we show how to extend the ideas behind KTO, such as prospect theory~\cite{tversky1992advances} and human aware losses (HALOs), to our sequential decision making setting. Additionally, CPL~\cite{hejnacontrastive} is developed for learning from pairwise preference feedback, it does not take into account the safety aspect. We cannot optimize safety and reward preferences independently. Therefore, one of our key contributions is the principled integration of safety modeling from prospect theory perspective with preference optimization in a sequential safe RL setting, which will be discussed in detail in~\cref{sec:integration}.
\subsection{Safety Alignment}
\label{sec:safety_alignment}

Although HALOs are introduced for the LLM setting (i.e., the contextual bandit framework), they offer a novel perspective for understanding Safe RLHF. For continuous control tasks which involve multi-step sequential decision-making, we derive a loss function over trajectory segments specifically for safety alignment. 

\subsubsection{Trajectory Segment Utility.}

We begin by defining the utility function of a trajectory segment $\sigma$ as:
\begin{equation}
    u(\sigma)\triangleq \psi_\pi(\sigma) - z_\text{ref}
    \label{eq:utility}
\end{equation}
where $\psi_\pi(\sigma)$ produces a scalar score of $\sigma$ based on policy $\pi$, and $z_\text{ref}$ serves as a reference point that determines the relative gain or loss when evaluating the outcome of $\sigma$ against all possible trajectory segments; score function $\psi_\pi(\sigma)$ is analogous to $r$ in Equation~\ref{eq:llmhalo}.
% \red{Does $\psi$ in Eq 6 is analogous to $r$ in Eq 5? If so, write ``utility $u$ is analogous to the $r$ in~\eqref{eq:llmhalo}"}.
% Different from the original prospect theory, \cref{eq:utility} ignores the weight placed on $\psi_\pi(\sigma)$ as humans do not see the full probability distribution of trajectory segments with respect to a policy $\pi$ within the offline dataset, and we will focus only on the score function \cite{ethayarajh2024kto}. 
% As stated in CPL \cite{hejnacontrastive}, the proposed preference learning framework shares similarities with contrastive learning approaches and the discounted sum of advantage function (i.e., $\sum_{\sigma} \gamma^t A_r^*(s_t, a_t)$) can be considered as a segment’s score from the Noise Conntrastive Estimation perspective. 
Inspired by the definition of HALOs, we can express $\psi_{\pi_\theta}(\sigma)$ with a parameterized policy $\pi_\theta$ \footnote{To maintain consistency, we use the same $\theta$ as in the preference alignment module.} as follows:
\begin{equation}
    \psi_{\pi_\theta}(\sigma) = \sum_{t=0}^T \gamma^t \beta \log \frac{\pi_\theta (a_t | s_t)}{\pi_\text{ref}(a_t | s_t)}.
    \label{eq:score}
\end{equation}
It is the cumulative logarithm of $\pi_\theta$ along the trajectory segment, regularized by the reference policy $\pi_\text{ref}$, with the hyperparameter $\beta$ governing the degree of influence $\pi_\text{ref}$ has on $\pi_\theta$. Accordingly, we define $z_\text{ref}$ as the expected score of all segments $\sigma$ that humans have encountered under policy $\pi_\theta$ using offline data, which serves as a biased estimate of the ground truth reference point, i.e., $\hat{z}_\text{ref} = \mathbb{E}_\sigma \left[ \sum_{t=0}^T \gamma^t \beta \log \frac{\pi_\theta(a_t | s_t)}{\pi_\text{ref}(a_t | s_t)} \right]$. 
To ensure stable training, we do not backpropagate through $\hat{z}_\text{ref}$, it exists solely to regulate the loss saturation.
% According to prospect theory, \citet{tversky1992advances} proposed a functional form for human value and we adapt it to our problem setting as follows:
% \begin{equation}
%     h(v_\pi, z_\text{ref}; \lambda, \beta)=
%     \begin{cases}
%         (v_\pi(\sigma) - z_\text{ref})^\beta & \text{if } v_\pi(\sigma) > z_\text{ref} \\
%         - \lambda(z_\text{ref} - v_\pi(\sigma))^\beta & \text{if } v_\pi(\sigma) < z_\text{ref}
%     \end{cases}
%     \label{eq:human_value}
% \end{equation}
% where $\beta$ controls how quickly utility changes and $\lambda$ controls the degree of loss aversion. Thus, a policy can be learned by optimizing the human value in \cref{eq:human_value}. 

\subsubsection{Loss Function for Safety Alignment.}

To develop a feasible objective function, we use \texttt{sigmoid} function as $v$ in HALO, as it aligns with the principle of prospect theory by being concave in gains and convex in losses. Additionally, we introduce two weight values, $\lambda_s$ and $\lambda_u$, for safe and unsafe segments respectively. These weights indicate the importance of $\sigma$ during policy training and reflect the concept of loss aversion in prospect theory~\cite{tversky1992advances}. Therefore, the loss function for safety alignment is written as:
% \begin{multline}
%     L_\text{safety}(\pi_\theta, D) = \lambda_s \mathbb{E}_{\sigma\sim D_\text{safe}} \left[ 1 - \texttt{sigmoid}\left( \psi_{\pi_\theta}(\sigma) - z_\text{ref} \right) \right] \\
%     + \lambda_u \mathbb{E}_{\sigma\sim D_\text{unsafe}} \left[ 1 - \texttt{sigmoid}\left( z_\text{ref} - \psi_{\pi_\theta}(\sigma) \right) \right]
%     \label{eq:safe_obj}
% \end{multline}
\begin{multline}
    L_\text{safety}(\pi_\theta, D) = \lambda_s \mathbb{E}_{\sigma\sim D_\text{safe}} \left[ 1 - \texttt{sigmoid}\left( u(\sigma) \right) \right] \\
    + \lambda_u \mathbb{E}_{\sigma\sim D_\text{unsafe}} \left[ 1 - \texttt{sigmoid}\left( -u(\sigma) \right) \right]
    \label{eq:safe_obj}
\end{multline}
The offline dataset $D$ is divided into two datasets. $D_\text{safe}$ contains all the safe segments while all the unsafe segments are in $D_\text{unsafe}$. The weights $\lambda_s$ and $\lambda_u$ are determined by the ratio of the number of safe segments $n_s$ to the number of unsafe segments $n_u$ in the offline dataset, specifically, $\frac{\lambda_s n_s}{\lambda_u n_u} = \eta$.
The hyperparameter $\eta$ regulates the relative importance of safe and unsafe segments during optimization. By minimizing Equation~\ref{eq:safe_obj}, the model assigns higher scores to safe segments and lower scores to unsafe ones. Accordingly, based on the definition of $\psi_{\pi_\theta}(\sigma)$ in Equation~\ref{eq:score}, the learned policy $\pi_\theta$ is encouraged to generate safe trajectories with high probability while avoiding unsafe behaviors.
To justify this formulation, we present the following result (Lemma~\ref{lemma1}), which establishes the connection between the trajectory utility function and its probabilistic representation. The proof is provided in Appendix A.1.
\begin{lemma}
\label{lemma1}
For a trajectory segment $\sigma$ and a given policy $\pi$, the utility function $u(\sigma;\pi)$ is directly proportional to the difference between the log probability of $\sigma$ under $\pi$ and the reference point $z_{\text{ref}}$, which is given by:
$u(\sigma;\pi)=\log p(\sigma; \pi) - z_{\text{ref}}(\pi)+ \text{constants}$.
\end{lemma}

% Furthermore, to ensure that the safety alignment achieved by \textsc{\textsc{PreSa}} on the offline dataset generalizes reliably to unseen scenarios, we establish a generalization result that formally bounds the discrepancy between empirical and true safety performance, as stated in Lemma~\ref{lemma2}. 
% This result provides the theoretical foundation for the feasibility guarantee, with the proof included in Appendix A.2.
Furthermore, a critical consideration in offline learning is that strong performance on the static dataset does not guarantee safety under the learned policy's state distribution. To address this, we provide a theoretical analysis that formally bounds the generalization error of our safety alignment module. Specifically, Lemma~\ref{lemma2} establishes that the empirical safety performance on the offline dataset closely approximates the true expected safety performance, with high probability. Crucially, our bound exhibits a favorable $\mathcal{O}(1/\sqrt{N})$ dependence on the dataset size $N$, meaning that larger offline datasets yield tighter generalization guarantees and more reliable safety assurance. This result provides a formal foundation for scaling safety assurance with data collection, ensuring that a policy which is empirically safe according to our alignment module will, with high confidence, remain safe when deployed. The proof is provided in Appendix A.2.
\begin{lemma}[Feasibility generalization bound]
\label{lemma2}
Let $\mathcal{D}=\{(\sigma_i,y_i)\}_{i=1}^N$ be drawn i.i.d.\ from an unknown distribution $\mathcal P$ over trajectory segments and binary safety labels. Define the function class
\begin{equation}
    \mathcal G = \left\{ g_\theta:(\sigma,y)\mapsto \texttt{sigmoid} \left(y\left(\psi_{\pi_\theta}(\sigma)-z_{\text{ref}} \right)\right)\;:\;\theta\in\Theta\right\},
\end{equation}
and the empirical and population safety scores
\begin{equation*}
    \hat F_{\mathcal D}(\pi_\theta)=\frac{1}{N}\sum_{i=1}^N g_\theta(\sigma_i,y_i),\qquad
    F(\pi_\theta)=\mathbb{E}_{(\sigma,y)\sim\mathcal P} \left[g_\theta(\sigma,y) \right].
\end{equation*}
Assume each $g_\theta$ takes values in $[0,1]$ (true for \texttt{sigmoid}). Then for any confidence level $\tau\in(0,1)$, with probability at least $1-\tau$ the following holds simultaneously for all $\pi_\theta\in\Pi_\Theta$:
\begin{equation}\label{eq:two-sided-bound}
    \big|F(\pi_\theta)-\hat F_{\mathcal D}(\pi_\theta)\big|
    \le 2\,\mathfrak R_N(\mathcal G) + \sqrt{\frac{\ln(2/\tau)}{2N}},
\end{equation}
where $\mathfrak R_N(\mathcal G)$ denotes the Rademacher complexity of $\mathcal G$.
\end{lemma}

\subsection{Integrated Preference and Safety Alignment (\textsc{PreSa})}
\label{sec:integration}

In this section, we propose a method to integrate pairwise preferences and safety alignment in an interpretable and principled manner into a single objective function that satisfies both criteria. We refer to our approach as \textit{\textsc{PreSa}} (Preference and Safety Alignment).

\subsubsection{Defining Feasible Policy Set}

We now convert safety alignment as a constraint on policy set. The two components in Equation~\ref{eq:safe_obj} address safe and unsafe segments separately, corresponding to the safety labels for each segment. We can rewrite the equation by combining these two components using the binary safety labels:
\begin{equation}
    L_\text{safety}(\pi_\theta, D) = 
    \mathbb{E}_{\sigma\sim D} \left[ w(y_\sigma) (1 - \texttt{sigmoid}\left( y_\sigma (\psi_{\pi_\theta}(\sigma) - z_\text{ref})) \right) \right]
    \label{eq:safe_obj_v1}
\end{equation}
where $y_\sigma$ denotes the safety label associated with segment $\sigma$, and $w(y_\sigma)$ represents the corresponding weight assigned to safe and unsafe segments. Specifically, $w(y_\sigma) = \lambda_s \, \mathbb{I}[y_\sigma = +1] + \lambda_u \, \mathbb{I}[y_\sigma = -1]$, where $\mathbb{I}[\cdot]$ is the indicator function.
% \begin{equation}
%     w(y_\sigma) = 
%     \begin{cases}
%         \lambda_s & \text{if } y_\sigma = +1 \\
%         \lambda_u & \text{if } y_\sigma = -1
%     \end{cases}
% \end{equation}
We observe that $y_\sigma (\psi_{\pi_\theta}(\sigma) - z_\text{ref})$ serves as a scalar score of segment $\sigma$ and the output of \texttt{sigmoid} function can be interpreted as the probability of predicting the label $y_\sigma$ for the corresponding segment $\sigma$ under policy $\pi_\theta$,
\begin{equation}
    p(Y=y_\sigma | \sigma; \pi_\theta) \triangleq \texttt{sigmoid}\left( y_\sigma (\psi_{\pi_\theta}(\sigma) - z_\text{ref}) \right)
\end{equation}
Therefore, in the context of a typical binary classification problem, minimizing the loss function in Equation~\ref{eq:safe_obj_v1} is equivalent to maximizing the probability of correctly classifying each segment with respect to safety. This classification-like objective determines which segments are safe, and which are unsafe. Consequently, we find that the safety alignment objective can be transformed to define a feasible policy set as follows,
\begin{equation}
    \Pi = \left\{ \pi | p(Y=y_\sigma | \sigma; \pi) \geq \delta, \forall \sigma \right\}
\end{equation}
where $\delta$ is a predefined parameter that controls the stringency with which we accept a segment as being correctly classified according to the safety labels provided by humans. When $\delta$ is close to $1$, we expect each segment to be classified correctly with a very high probability. In this case, if the ground truth label is $+1$, the segment's score should be high, and the learned policy should assign high probabilities to this segment according to Equation~\ref{eq:score}. If the ground truth label is $-1$, the segment's score should be low, and the learned policy should assign low probabilities to it. When $\delta$ is more relaxed, we allow for greater tolerance, which may lead to a policy that occasionally violates the implicit safety constraints. 

\subsubsection{Unified Objective of \textsc{PreSa}}

As the preference alignment module ensures consistency with human preferences and the safety alignment module implicitly defines a feasibility policy set, we can now integrate them into a single constrained optimization problem:
\begin{align}
\begin{split}
    \min_{\pi_\theta}\, &\mathbb{E}_{(\sigma^+, \sigma^-) \sim D} \left[ - \log P_{\pi_\theta}\left[ \sigma^+ \succ \sigma^- \right] \right] \\
    \text{s.t., } &\mathbb{E}_{(\sigma, y_\sigma) \sim D} \left[ p \left(Y=y_\sigma | \sigma; \pi_\theta \right) \right] \geq \delta
    \label{eq:rlhf-lag}
\end{split}
\end{align}
where the objective term for reward preference is defined analogously to the CPL objective in Equation~\ref{eq:cplscore} without explicitly learning the rewards. 
Compared to the objective function of safe RL in Equation~\ref{eq:srl}, the term of minimizing negative log-likelihood of preferences in Equation~\ref{eq:rlhf-lag} corresponds to the maximization of cumulative reward. The feasible policy set determined by the safety alignment module corresponds to the one defined by the cumulative cost constraints.
Given the above constrained objective, we expect to learn a policy that maximizes reward while adhering to the safety constraints implicitly encoded in the safety labels. Notably, \textsc{PreSa} does not require learning of reward and cost functions, and avoids conventional constrained RL. 

To solve this constrained optimization problem, we employ the Lagrangian method to convert the constrained primal problem into an unconstrained dual form:
\begin{align}
    \min_{\pi_\theta} \max_{\nu \geq 0}\, &L(\pi_\theta, \nu, D) = 
    \mathbb{E}_{(\sigma^+, \sigma^-) \sim D} \big[ - \log P_{\pi_\theta}\left[ \sigma^+ \succ \sigma^- \right] \big] \nonumber \\
    &+ \nu \cdot \left( \delta - \mathbb{E}_{(\sigma, y_\sigma) \sim D} \left[ p \left(Y=y_\sigma | \sigma; \pi_\theta \right) \right] \right)
\end{align}
where $\nu \geq 0$ is the Lagrange multiplier. 
% Theoretical results regarding the convergence of \textsc{PreSa} are provided in Appendix.
%Specifically, we prove the following theorem in Appendix A:
% \begin{theorem}
% Assume an unbounded number of preferences and safety labels generated from a noisy rational regret-based preference model with expert advantage function $A^*$. Any solution found by \textsc{PreSa} that satisfies the KKT conditions of its constrained optimization recovers the optimal policy $\pi^*$.
% \end{theorem}
By introducing the weights $w(y_\sigma)$ back, the above objective function is rewritten as,
\begin{multline}
    L(\pi_\theta, \nu, D) = 
    \mathbb{E}_{(\sigma^+, \sigma^-) \sim D} \left[ - \log P_{\pi_\theta}\left[ \sigma^+ \succ \sigma^- \right] \right] + \\
    \nu \cdot \mathbb{E}_{(\sigma, y_\sigma) \sim D} \left[ w(y_\sigma) \cdot \left( \delta -  p \left(Y=y_\sigma | \sigma; \pi_\theta \right) \right) \right]
    \label{eq:lag_obj}
\end{multline}
Interestingly, the optimization of preference learning may sometimes conflict with the objective of optimizing safety alignment, though they can also complement each other in learning a better policy. That is because preference and safety do not influence each other in a monotonic manner. For instance, some unsafe segments might be preferred while some unpreferred segments could be safe as well. Thus, in Equation~\ref{eq:lag_obj}, the Lagrange multiplier $\nu$ dynamically manages the mutual influence. To address the optimization problem, we iteratively update the policy parameter $\theta$ and the Lagrange multiplier $\nu$ using gradient descent, which helps avoid over-emphasizing one objective at the expense of the other due to a fixed optimization ratio.

% \begin{figure}
%     \centering
%     \includegraphics[width=0.9\columnwidth]{figures/safety_ratio.png}
%     \caption{\small Ratio of safe agents learned by different approaches.}
%     \label{fig:safe_ratio}
% \end{figure}

%%%%%%%%%%%%%%%%%%%%%%%%%%%%%%%%%%%%%%%%%%%%%%%%%%%%%%%%%%%%%%%%%%%%%%%%

\section{Experiments}

In this section, we present experiments with synthetic and real human feedback to evaluate the effectiveness of \textsc{PreSa} in achieving both high reward performance and safety adherence. We aim to address the following questions:
% \textit{a)} How does \textsc{PreSa} compare to Offline Safe POHF baselines and Offline Safe RL with ground truth reward and cost?
% \textit{b)} How does \textsc{PreSa} perform with varying segment lengths and offline dataset sizes?
% \textit{c)} What ingredients of \textsc{PreSa} are important for enhanced performance and safety alignment?
% \textit{d)} How does \textsc{PreSa} perform with real human feedback?
% \textit{e)} How important are safety labels for learning a safe policy from human feedback?
\begin{itemize}
    \item How does \textsc{PreSa} compare to Offline Safe POHF baselines and those from Offline Safe RL with access to ground truth reward and cost?
    \item How does the preference and safety alignment module perform, respectively?
    \item How does \textsc{PreSa} perform with varying trajectory segment lengths and offline dataset sizes?
    \item What ingredients of \textsc{PreSa} are important for enhanced performance and safety alignment?
    \item How does \textsc{PreSa} perform with real human feedback?
    \item How important are safety labels for learning a safe policy from human feedback?
\end{itemize}

Additional experimental results, including an evaluation of \textsc{PreSa}'s effectiveness under noisy feedback and extensive ablation studies, are provided in the Appendix.

% $0.14\scriptstyle {\pm 0.4}$
\begin{table*}[t]
  \caption{Evaluation results of normalized reward and cost. The $\uparrow$ symbol indicates that the higher reward, the better, while the $\downarrow$ symbol signifies that the lower cost (up to a threshold of $1$), the better. 
  % Each value is averaged over 3 distinct cost thresholds, 20 evaluation episodes, and 3 random seeds. 
  \textbf{Bold}: Safe agents whose normalized cost is below 1. \textcolor{blue}{\textbf{Blue}}: Safe agent with the highest reward among Offline Safe RL baselines. \textcolor{orange}{\textbf{Orange}}: Safe agent with the highest reward among Offline Safe POHF approaches. \textcolor{gray}{\textbf{Gray}}: Unsafe agent.}
  \label{tab:representative_result}
  % \centering
  \resizebox{\textwidth}{!}{
  \begin{tabular}{ccccccc|cccccccccc}
    \toprule
    \multirow{3}[2]{*}{Task} & \multicolumn{6}{c}{Offline Safe RL (w/ ground truth rewards and costs)}  &  \multicolumn{8}{c}{Offline Safe POHF} \\
    \cmidrule(lr){2-7}
    \cmidrule(lr){8-15}
        &   \multicolumn{2}{c}{BC-All} & \multicolumn{2}{c}{BC-Safe} & \multicolumn{2}{c}{CDT} & \multicolumn{2}{c}{Binary Alignment} & \multicolumn{2}{c}{BC-Safe-Seg} & \multicolumn{2}{c}{Safe-RLHF (CDT)} & \multicolumn{2}{c}{\textsc{PreSa} (Ours)}  \\ %(unnormalized)} \\
    \cmidrule(lr){2-3}
    \cmidrule(lr){4-5}
    \cmidrule(lr){6-7}
    \cmidrule(lr){8-9}
    \cmidrule(lr){10-11}
    \cmidrule(lr){12-13}
    \cmidrule(lr){14-15}
      &   reward$\uparrow$    &   cost$\downarrow$ &   reward$\uparrow$    &   cost$\downarrow$  &   reward$\uparrow$    &   cost$\downarrow$ &   reward$\uparrow$    &   cost$\downarrow$  &   reward$\uparrow$    &   cost$\downarrow$    &   reward$\uparrow$    &   cost$\downarrow$    &   reward$\uparrow$    &   cost$\downarrow$    \\
    \midrule
    PointButton1    &   \textcolor{gray}{0.1}   &   \textcolor{gray}{1.05}   &   \textcolor{blue}{\textbf{0.06}}   &   \textcolor{blue}{\textbf{0.52}}   &   \textcolor{gray}{0.53}   &   \textcolor{gray}{1.68}  &   \textbf{0.02}   &   \textbf{0.55}    &    \textbf{0.06}    &   \textbf{0.81}    & \textbf{0.06}    &   \textbf{0.78}    &   \textcolor{orange}{\textbf{0.09}}    &   \textcolor{orange}{\textbf{0.84}}    \\
    PointButton2    &   \textcolor{gray}{0.27}    &   \textcolor{gray}{2.02}    &   \textcolor{gray}{0.16}    &   \textcolor{gray}{1.1} &   \textcolor{gray}{0.46}    &   \textcolor{gray}{1.57}    &   \textcolor{orange}{\textbf{-0.03}}   &   \textcolor{orange}{\textbf{0.5}}    &    \textcolor{gray}{0.22}    &   \textcolor{gray}{1.57}    & \textcolor{gray}{0.18}    &   \textcolor{gray}{1.33}    &   \textbf{-0.1}   &   \textbf{0.74}    \\
    % PointCircle1    &   0.79    &   3.98    &   \textbf{0.41}    &   \textbf{0.16}    &   \textcolor{blue}{\textbf{0.59}}    &   \textcolor{blue}{\textbf{0.69}}    &   -0.23   &   1.21    &   0.32    &   1.09    &   0.37    &   2.97    &   \textcolor{orange}{\textbf{0.4}} &   \textcolor{orange}{\textbf{0.21}} \\
    % PointCircle2    &   0.66    &   4.17    &   \textcolor{blue}{\textbf{0.48}}    &   \textcolor{blue}{\textbf{0.99}}    &   0.64    &   1.05    &   -0.24   &   8.3    &   0.44    &   1.89    &  0.66    &   4.87    &   \textcolor{orange}{\textbf{0.16}} &   \textcolor{orange}{\textbf{0.96}}    \\
    PointGoal1  &   \textcolor{blue}{\textbf{0.65}}    &   \textcolor{blue}{\textbf{0.95}}    &   \textbf{0.43}    &   \textbf{0.54}    &   \textcolor{gray}{0.69}    &   \textcolor{gray}{1.12}    &   \textbf{0.31}    &   \textbf{0.42}    &   \textcolor{gray}{0.48}    &   \textcolor{gray}{1.17}    & \textbf{0.34}    &   \textbf{0.52}    &   \textcolor{orange}{\textbf{0.37}}    &   \textcolor{orange}{\textbf{0.73}}  \\
    PointGoal2  &   \textcolor{gray}{0.54}    &   \textcolor{gray}{1.97}    &   \textcolor{blue}{\textbf{0.29}}    &   \textcolor{blue}{\textbf{0.78}}    &   \textcolor{gray}{0.59}    &   \textcolor{gray}{1.34}    &   \textcolor{gray}{0.39}   &   \textcolor{gray}{1.15}    &   \textcolor{gray}{0.52}  &   \textcolor{gray}{2.08}    &    \textcolor{gray}{0.35}    &   \textcolor{gray}{2.5} &   \textcolor{orange}{\textbf{0.16}}    &   \textcolor{orange}{\textbf{0.96}}    \\
    % PointPush1  &   \textbf{0.19}    &   \textbf{0.61}    &   \textbf{0.13}    &   \textbf{0.43}    &   \textcolor{blue}{\textbf{0.24}}    &   \textcolor{blue}{\textbf{0.48}}    &   \textbf{0.14}   &   \textbf{0.51}    &   \textcolor{orange}{\textbf{0.19}} &   \textcolor{orange}{\textbf{0.6}} &   \textbf{0.1} &   \textbf{0.36}    &   \textbf{0.14}    &   \textbf{0.4}  \\
    % PointPush2  &   \textbf{0.18}    &   \textbf{0.91}    &   \textbf{0.11}    &   \textbf{0.8} &   \textcolor{blue}{\textbf{0.21}}    &   \textcolor{blue}{\textbf{0.65}}   &   0.17   &   1.69    &   \textcolor{orange}{\textbf{0.18}}    &   \textcolor{orange}{\textbf{0.8}}    &   \textbf{0.08}    &   \textbf{0.22}    &   \textbf{0.12}    &   \textbf{0.9}  \\
    CarButton1  &   \textcolor{gray}{0.03}    &   \textcolor{gray}{1.38}    &   \textcolor{blue}{\textbf{0.07}}    &   \textcolor{blue}{\textbf{0.85}}    &   \textcolor{gray}{0.21}    &   \textcolor{gray}{1.6}   &   \textcolor{gray}{-0.01}   &   \textcolor{gray}{2.52}    &   \textcolor{gray}{0.02}    &   \textcolor{gray}{1.42}    &   \textcolor{gray}{0.05}    &   \textcolor{gray}{3.96}    &   \textcolor{gray}{0.12}    &   \textcolor{gray}{1.87}   \\
    CarButton2  &   \textcolor{gray}{-0.13}   &   \textcolor{gray}{1.24}    &   \textcolor{blue}{\textbf{-0.01}}   &   \textcolor{blue}{\textbf{0.63}}    &   \textcolor{gray}{0.13}    &   \textcolor{gray}{1.58}    &   \textcolor{gray}{-0.06}   &   \textcolor{gray}{1.36}    &   \textcolor{gray}{-0.03}   &   \textcolor{gray}{1.01}    & \textcolor{gray}{0.02}  & \textcolor{gray}{1.77}    &   \textcolor{gray}{-0.04}   &   \textcolor{gray}{1.27}    \\
    % CarCircle1  &   0.72    &   4.39    &   0.37    &   1.38    &   0.6 &   1.73    &   -0.32   &   4.71    &   0.61    &   4.53    &   0.27    &   3.53    &   -0.26    &   2.86\\
    % CarCircle2  &   0.76    &   6.44    &   0.54    &   3.38    &   0.66    &   2.53    &   \textbf{-0.23}   &   \textbf{0.0}   &   0.63    &   4.23    &   0.5 &   3.91    &   \textcolor{orange}{\textbf{0.23}}    &   \textcolor{orange}{\textbf{0.22}}  \\
    CarGoal1    &   \textcolor{blue}{\textbf{0.39}}    &   \textcolor{blue}{\textbf{0.33}}    &   \textbf{0.24}    &   \textbf{0.28}    &   \textcolor{gray}{0.66}    &   \textcolor{gray}{1.21}    &   \textbf{0.29}    &   \textbf{0.38}    &   \textbf{0.25}    &   \textbf{0.3} &   \textcolor{orange}{\textbf{0.4}}    &   \textcolor{orange}{\textbf{0.61}}    &   \textbf{0.26}    &   \textbf{0.14}    \\
    CarGoal2    &   \textcolor{gray}{0.23}    &   \textcolor{gray}{1.05}    &   \textcolor{blue}{\textbf{0.14}}    &   \textcolor{blue}{\textbf{0.51}}    &   \textcolor{gray}{0.48}    &   \textcolor{gray}{1.25}    &   \textcolor{orange}{\textbf{0.18}}   &   \textcolor{orange}{\textbf{0.64}}    &   \textcolor{gray}{0.17} &   \textcolor{gray}{1.03}    &   \textcolor{gray}{0.18}    &   \textcolor{gray}{1.01}    &   \textbf{0.14}   &   \textbf{0.35}  \\
    % CarPush1    &   \textbf{0.22}    &   \textbf{0.36}    &   \textbf{0.14}    &   \textbf{0.33}    &   \textcolor{blue}{\textbf{0.31}}    &   \textcolor{blue}{\textbf{0.4}}   &   \textbf{0.16}   &   \textbf{0.34} &   \textcolor{orange}{\textbf{0.21}}    &   \textcolor{orange}{\textbf{0.51}}    & \textbf{0.17}    &   \textbf{0.96}    &   \textbf{0.15}    &   \textbf{0.56}  \\
    % CarPush2    &   \textcolor{blue}{\textbf{0.14}}    &   \textcolor{blue}{\textbf{0.9}} &   \textbf{0.05}    &   \textbf{0.45}    &   0.19    &   1.3 &   \textbf{0.07}   &   \textbf{0.69} &   \textbf{0.07}    &   \textbf{0.91}    &   0.1 &   1.81    &   \textcolor{orange}{\textbf{0.1}}    &   \textcolor{orange}{\textbf{0.52}}    \\
    SwimmerVelocity &   \textcolor{gray}{0.49}    &   \textcolor{gray}{4.72}    &   \textcolor{gray}{0.51}    &   \textcolor{gray}{1.07}    &   \textcolor{blue}{\textbf{0.66}}    &   \textcolor{blue}{\textbf{0.96}}    &   \textcolor{orange}{\textbf{-0.04}}   &   \textcolor{orange}{\textbf{0.7}}    &   \textcolor{gray}{0.33}  &   \textcolor{gray}{2.61}    &  \textcolor{gray}{0.66}    &   \textcolor{gray}{1.1} &   \textcolor{gray}{0.39}    &   \textcolor{gray}{1.96}    \\
    % HopperVelocity  &  0.65    &   6.39    &   \textbf{0.36}    &   \textbf{0.67}    &   \textcolor{blue}{\textbf{0.63}}    &   \textcolor{blue}{\textbf{0.61}} &   \textbf{-0.02}   &   \textbf{0.0} &   \textcolor{orange}{\textbf{0.64}}    &   \textcolor{orange}{\textbf{0.64}}    &   0.17    &   1.27    &   0.42    &   5.89 \\
    % HalfCheetahVelocity &   0.97    &   13.1    &   \textbf{0.88}    &   \textbf{0.54}    &   \textcolor{blue}{\textbf{1.0}} &   \textcolor{blue}{\textbf{0.01}}    &   \textbf{0.05}   &   \textbf{0.0} &   \textbf{0.92}    &   \textbf{0.54}    &    \textcolor{orange}{\textbf{0.93}}    &   \textcolor{orange}{\textbf{0.37}}    &   0.71    &   4.11    \\
    Walker2dVelocity    &   \textcolor{gray}{0.79}    &   \textcolor{gray}{3.88}    &   \textcolor{blue}{\textbf{0.79}}    &   \textcolor{blue}{\textbf{0.04}}    &   \textbf{0.78}    &   \textbf{0.06}    &   \textbf{-0.01}   &   \textbf{0.0} &   \textbf{0.78}    &   \textbf{0.01}    &   \textcolor{gray}{0.11}    &   \textcolor{gray}{1.42}    &   \textcolor{orange}{\textbf{0.79}}    &   \textcolor{orange}{\textbf{0.0}}   \\
    % AntVelocity &   0.98    &   3.72    &   \textcolor{blue}{\textbf{0.98}}    &   \textcolor{blue}{\textbf{0.29}}    &   \textbf{0.98}    &   \textbf{0.39}    &   \textbf{-0.06}   &   \textbf{0.0}    &   \textcolor{orange}{\textbf{0.96}}  &   \textcolor{orange}{\textbf{0.3}}  & \textbf{0.93}    &   \textbf{0.23}    &   \textcolor{orange}{\textbf{0.96}}    &   \textcolor{orange}{\textbf{0.27}}  \\
    \midrule
    \textbf{SafetyGym Average}   &   \textcolor{gray}{0.34}    &   \textcolor{gray}{1.9}    &   \textcolor{blue}{\textbf{0.27}}    &   \textcolor{blue}{\textbf{0.63}}    &   \textcolor{gray}{0.52}    &   \textcolor{gray}{1.24}    &   \textbf{0.1}   &   \textbf{0.82}    &   \textcolor{gray}{0.28}   &   \textcolor{gray}{1.2}    &  \textcolor{gray}{0.24}    &   \textcolor{gray}{1.5} &   \textcolor{orange}{\textbf{0.22}}    &   \textcolor{orange}{\textbf{0.78}}  \\
    % \midrule
    % \textbf{\# of Safe Agents (out of 21)}   &   \multicolumn{2}{c}{6}    &   \multicolumn{2}{c}{17}    &   \multicolumn{2}{c}{9}    &   \multicolumn{2}{c}{14}    &   \multicolumn{2}{c}{11}    &   &   &   \multicolumn{2}{c}{14}  \\
    \midrule
    BallRun &   \textcolor{gray}{0.6} &   \textcolor{gray}{5.08}    &   \textcolor{gray}{0.27}    &   \textcolor{gray}{1.46}    &   \textcolor{gray}{0.39}    &   \textcolor{gray}{1.16}    &   \textcolor{gray}{0.31}   &   \textcolor{gray}{4.79}   &  \textcolor{gray}{0.37}    &   \textcolor{gray}{1.13}    &   \textcolor{gray}{0.35}    &   \textcolor{gray}{1.65}  &   \textcolor{orange}{\textbf{0.19}}    &   \textcolor{orange}{\textbf{0.09}}   \\
    % CarRun  &   \textbf{0.97}    &   \textbf{0.33}    &   \textbf{0.94}    &   \textbf{0.22}    &   \textcolor{blue}{\textbf{0.99}}    &   \textcolor{blue}{\textbf{0.65}}    &   \textbf{0.94}   &    \textbf{0.0}   &   \textcolor{orange}{\textbf{0.97}}    &   \textcolor{orange}{\textbf{0.95}}    &   0.87    &   1.16  &   \textbf{0.96}    &   \textbf{0.0}   \\
    DroneRun    &   \textcolor{gray}{0.24}    &   \textcolor{gray}{2.13}    &   \textbf{0.28}    &   \textbf{0.74}    &   \textcolor{blue}{\textbf{0.63}}    &   \textcolor{blue}{\textbf{0.79}}    &   \textbf{0.11}   &   \textbf{0.17}   &   \textcolor{gray}{0.17}    &   \textcolor{gray}{5.97}    &   \textcolor{gray}{0.47}    &   \textcolor{gray}{3.12}    &   \textcolor{orange}{\textbf{0.16}}    &   \textcolor{orange}{\textbf{0.33}}    \\
    % AntRun  &   0.72    &   2.93    &   0.65    &   1.09    &   \textcolor{blue}{\textbf{0.72}}    &   \textcolor{blue}{\textbf{0.91}}    &   \textbf{0.09}   &    \textbf{0.01}   &   0.66    &   1.38    &   0.72 &   1.04    &   \textcolor{orange}{\textbf{0.61}}    &   \textcolor{orange}{\textbf{0.63}}  \\
    BallCircle  &   \textcolor{gray}{0.74}    &   \textcolor{gray}{4.71}    &   \textcolor{blue}{\textbf{0.52}}    &   \textcolor{blue}{\textbf{0.65}}    &   \textcolor{gray}{0.77}    &   \textcolor{gray}{1.07}    &   \textbf{0.06}   &    \textbf{0.24}   &  \textcolor{orange}{\textbf{0.39}}    &   \textcolor{orange}{\textbf{0.68}}    &   \textcolor{gray}{0.68}    &   \textcolor{gray}{1.2}    &   \textbf{0.22} &   \textbf{0.03}    \\
    % CarCircle   &   0.58    &   3.74    &   \textbf{0.5} &   \textbf{0.84}    &   \textcolor{blue}{\textbf{0.75}}    &   \textcolor{blue}{\textbf{0.95}}    &   \textbf{0.06}   &   \textbf{0.35}   &  0.56    &   1.76    &   \textcolor{orange}{\textbf{0.57}}    &   \textcolor{orange}{\textbf{0.84}}    &   \textbf{0.08}    &   \textbf{0.91} \\
    DroneCircle &   \textcolor{gray}{0.72}    &   \textcolor{gray}{3.03}    &   \textbf{0.56}    &   \textbf{0.57}    &   \textcolor{blue}{\textbf{0.63}}    &   \textcolor{blue}{\textbf{0.98}}    &   \textcolor{gray}{-0.23}   &   \textcolor{gray}{1.59}   &   \textcolor{gray}{0.61}    &   \textcolor{gray}{1.9}  &   \textcolor{orange}{\textbf{0.61}}    &   \textcolor{orange}{\textbf{0.87}}  &   \textbf{0.54}    &   \textbf{0.72}    \\
    % AntCircle   &   0.58    &   4.9 &   \textcolor{blue}{\textbf{0.4}} &   \textcolor{blue}{\textbf{0.96}}    &   0.54    &   1.78    &   0.48   &    3.03   &    0.54    &   3.15    &   0.45    &   2.04    &   0.55    &   3.78 \\
    \midrule
    \textbf{BulletGym Average}   &   \textcolor{gray}{0.58}    &   \textcolor{gray}{3.74}    &   \textcolor{blue}{\textbf{0.41}}    &   \textcolor{blue}{\textbf{0.86}}    &   \textcolor{gray}{0.61}    &   \textcolor{gray}{1.0}    &   \textcolor{gray}{0.06}   &   \textcolor{gray}{1.70}   &    \textcolor{gray}{0.39}    &   \textcolor{gray}{2.42}    &   \textcolor{gray}{0.53}    &   \textcolor{gray}{1.71}   &   \textcolor{orange}{\textbf{0.28}}    &   \textcolor{orange}{\textbf{0.29}}   \\
    % \midrule
    % \textbf{\# of Safe Agents (out of 8)}   &   \multicolumn{2}{c}{1}    &   \multicolumn{2}{c}{6}    &   \multicolumn{2}{c}{5}    &   \multicolumn{2}{c}{5}    &   \multicolumn{2}{c}{2}    &  \multicolumn{2}{c}{2}   &   \multicolumn{2}{c}{7}  \\
    \midrule
    % easysparse  &   \textcolor{gray}{0.17}    &   \textcolor{gray}{1.54}    &   \textbf{0.11}    &   \textbf{0.21}    &   \textcolor{blue}{\textbf{0.17}}    &   \textcolor{blue}{\textbf{0.23}}    &   \textbf{-0.04}   &   \textbf{0.06}    &   \textbf{0.01}    &   \textbf{0.1} &   \textcolor{gray}{0.42}    &   \textcolor{gray}{1.5} &   \textcolor{orange}{\textbf{0.31}}    &   \textcolor{orange}{\textbf{0.67}}    \\
    % easymean    &   \textcolor{gray}{0.43}    &   \textcolor{gray}{2.82}    &   \textbf{0.04}    &   \textbf{0.29}    &   \textcolor{blue}{\textbf{0.45}}    &   \textcolor{blue}{\textbf{0.54}}    &   \textbf{-0.04}   &   \textbf{0.06}    &   \textbf{0.13}    &   \textbf{0.14}    &   \textcolor{gray}{0.33}    &   \textcolor{gray}{1.37}    &   \textcolor{orange}{\textbf{0.37}}    &   \textcolor{orange}{\textbf{0.6}} \\
    easydense   &   \textcolor{gray}{0.27}    &   \textcolor{gray}{1.94}    &   \textbf{0.11}    &   \textbf{0.14}    &   \textcolor{blue}{\textbf{0.32}}    &   \textcolor{blue}{\textbf{0.62}}    &   \textbf{-0.01}   &   \textbf{0.11}    &   \textbf{0.1} &   \textbf{0.14}    &   \textcolor{gray}{0.43}    &   \textcolor{gray}{1.68}    &   \textcolor{orange}{\textbf{0.21}}    &   \textcolor{orange}{\textbf{0.37}}    \\
    % mediumsparse    &   \textcolor{gray}{0.83}    &   \textcolor{gray}{3.34}    &   \textcolor{blue}{\textbf{0.33}}    &   \textcolor{blue}{\textbf{0.3}} &   \textcolor{gray}{0.87}    &   \textcolor{gray}{1.1} & \textbf{0.02}    &   \textbf{0.11}    &   \textbf{0.19}    &   \textbf{0.16}    &   \textcolor{gray}{0.45}    &   \textcolor{gray}{1.12}    &   \textcolor{orange}{\textbf{0.52}}    &   \textcolor{orange}{\textbf{0.54}}    \\
    mediummean  &   \textcolor{gray}{0.77}    &   \textcolor{gray}{2.53}    &   \textbf{0.31}    &   \textbf{0.21}    &   \textcolor{blue}{\textbf{0.45}}    &   \textcolor{blue}{\textbf{0.75}}    &   \textbf{0.03}    &   \textbf{0.09}    &   \textbf{0.09}    &   \textbf{0.17}    &   \textcolor{gray}{0.46}    &   \textcolor{gray}{1.0} &   \textcolor{orange}{\textbf{0.35}}    &   \textcolor{orange}{\textbf{0.43}}    \\
    % mediumdense &   \textcolor{gray}{0.45}    &   \textcolor{gray}{1.47}    &   \textcolor{blue}{\textbf{0.24}}    &   \textcolor{blue}{\textbf{0.17}}    &   \textcolor{gray}{0.88}    &   $\textcolor{gray}{2.41}$    & \textbf{-0.0}    &   \textbf{0.07}    &   \textbf{0.16}    &   \textbf{0.18}    &   \textbf{0.45}    &   \textbf{0.94}    &   \textcolor{orange}{\textbf{0.56}}    &   \textcolor{orange}{\textbf{0.52}}    \\
    hardsparse  &   \textcolor{gray}{0.42}    &   \textcolor{gray}{1.8} &   \textcolor{gray}{0.17}    &   \textcolor{gray}{3.25}    &   \textcolor{blue}{\textbf{0.25}}    &   \textcolor{blue}{\textbf{0.41}}    & \textbf{-0.02}   &   \textbf{0.06}    &   \textbf{0.09}    &   \textbf{0.18}    &   \textcolor{gray}{0.27}    &   \textcolor{gray}{1.16}    &   \textcolor{orange}{\textbf{0.24}}    &   \textcolor{orange}{\textbf{0.51}}    \\
    % hardmean    &   \textcolor{gray}{0.2} &   \textcolor{gray}{1.77}    &   \textbf{0.13}    &   \textbf{0.4} &   \textcolor{blue}{\textbf{0.33}}    &   \textcolor{blue}{\textbf{0.97}}    &   \textbf{-0.02}   &   \textbf{0.07}    &   \textbf{0.01}    &   \textbf{0.29}    &   \textcolor{gray}{0.29}    &   \textcolor{gray}{1.23}    &   \textcolor{orange}{\textbf{0.19}}    &   \textcolor{orange}{\textbf{0.58}}    \\
    % harddense   &   \textcolor{gray}{0.2} &   \textcolor{gray}{1.33}    &   \textcolor{blue}{\textbf{0.15}}    &   \textcolor{blue}{\textbf{0.22}}    &   \textbf{0.08}    &   \textbf{0.21}    &   \textbf{0.05}    &   \textbf{0.16}    &   \textbf{0.08}    &   \textbf{0.19}    &   \textcolor{gray}{0.21}    &   \textcolor{gray}{1.25}    &   \textcolor{orange}{\textbf{0.12}}    &   \textcolor{orange}{\textbf{0.35}}    \\
    \midrule
    \textbf{MetaDrive Average}  &  \textcolor{gray}{0.49}    &   \textcolor{gray}{2.09}    &   \textcolor{gray}{0.2}    &   \textcolor{gray}{1.2}    &   \textcolor{blue}{\textbf{0.34}}    &   \textcolor{blue}{\textbf{0.6}} &    \textbf{-0.0}    &   \textbf{0.22}    &   \textbf{0.08} &   \textbf{0.16}    &   \textcolor{gray}{0.39}    &   \textcolor{gray}{1.28}    &   \textcolor{orange}{\textbf{0.27}}    &   \textcolor{orange}{\textbf{0.44}}    \\
    \bottomrule
  \end{tabular}
  }
\end{table*}

\subsection{Experimental Setting For Synthetic Feedback}

We conduct experiments using the DSRL benchmark~\cite{liu2023datasets}, designed for Offline Safe RL, to generate synthetic feedback for quantitative evaluation. The benchmark provides offline data across 38 tasks in SafetyGym~\cite{ray2019benchmarking,ji2023omnisafe}, BulletGym~\cite{gronauer2022bullet}, and MetaDrive~\cite{li2022metadrive}, where agents aim for high rewards while avoiding obstacles or maintaining safe velocities. For fair comparison, we use the constraint variation evaluation method from DSRL~\cite{liu2023datasets}, testing each method on every task with three cost thresholds and three random seeds. Evaluation metrics include normalized reward and cost, where cost below $1$ indicates safety~\cite{liu2023datasets,fu2020d4rl}.
\textsc{PreSa} follows a training pipeline similar to CPL~\cite{hejnacontrastive}, first pretraining with behavior cloning (BC) on the offline dataset, then optimizing with pairwise preferences and binary safety labels.

\paragraph{Synthetic Feedback.} 
For evaluation, we generate synthetic feedback from the DSRL offline dataset. We randomly select 10,000 pairs of trajectory segments with varying lengths, enabling further testing with fewer feedbacks via random sampling. Pairwise preferences are based on cumulative rewards from the offline dataset, as in dense reward settings, the optimal advantage is a reshaped version of the reward, leading to the same policy~\cite{hejnacontrastive}. Trajectory segments are labeled for safety using ground truth cost values and predefined safety thresholds, which are adjusted proportionally based on segment length relative to the maximum trajectory length in each domain. This reshaped threshold ensures stricter constraints on each segment, guaranteeing that policies satisfying segment thresholds also satisfy the original trajectory threshold.

\paragraph{Baselines.} We compare our approach against several baselines to demonstrate its effectiveness, which are across Offline Safe POHF setting and Offline Safe RL setting with ground truth rewards and costs. In Offline Safe POHF setting, we consider three baseline methods: 1) \textit{Binary Alignment}: Inspired by~\cite{ethayarajh2024kto}, unified binary labels are generated based on comparative preferences and binary safety labels. Segments that are both safe and preferred are assigned a label of $+1$, while all others $-1$. The policy is then learned solely using the safety alignment module. 2) \textit{BC-Safe-Seg}: A behavior cloning (BC) approach trained only on safe trajectory segments. 3) \textit{Safe-RLHF (CDT)}~\cite{daisafe}: A variant of Safe-RLHF adapted to our setting where a state-of-the-art offline safe RL approach, CDT~\cite{liu2023constrained} is applied for policy optimization.
Additionally, we select three baselines from Offline Safe RL setting. They assume access to ground truth reward and cost: 1) \textit{BC-All}: BC trained on the entire dataset. 2) \textit{BC-Safe}: BC trained exclusively on safe trajectories. 3) \textit{CDT}: A sequence modeling approach that incorporates safety constraints into Decision Transformers.

% \subsection{Results}
\subsubsection{How Does \textsc{PreSa} Perform?}
% Extensive experiments were conducted to evaluate \textsc{PreSa}'s performance against baselines across all 29 tasks within SafetyGym and BulletGym.
For safety alignment, Figure~\ref{fig:safe_ratio} shows the proportion of safe policies learned by each approach in both domains. \textsc{PreSa} surpasses all Offline Safe POHF baselines with more \textit{safe} policies learned. Compared to Offline Safe RL baselines, \textsc{PreSa} performs comparably in SafetyGym, outperforming two baselines by a significant margin, except for BC-Safe, while in BulletGym, \textsc{PreSa} outperforms all baselines with most safe policies learned.
\begin{figure}
    \centering
    \includegraphics[width=\columnwidth]{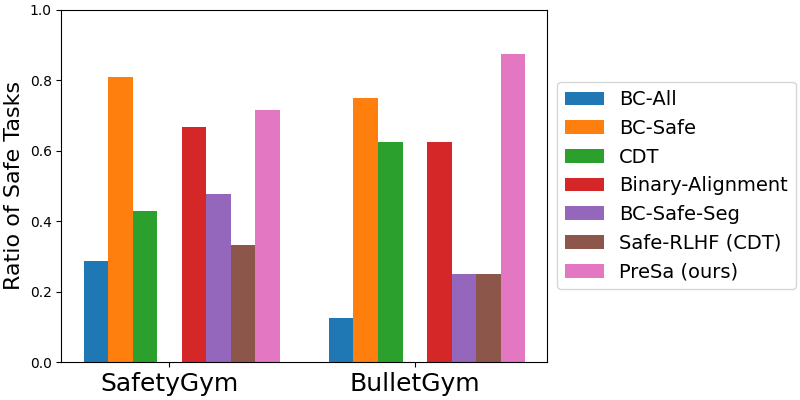}
    \caption{Ratio of safe agents learned by different approaches.}
    \label{fig:safe_ratio}
\end{figure}
% \begin{wrapfigure}{r}{0.5\textwidth}
%     \centering
%     \includegraphics[width=\linewidth]{figures/safety_ratio.png}
%     \caption{Ratio of safe agents learned by different approaches.}
%     \label{fig:safe_ratio}
% \end{wrapfigure}

The results of representative tasks are presented in Table~\ref{tab:representative_result}. Among the Offline Safe POHF baselines, Binary Alignment achieves low costs but also results in very low rewards across all tasks. This occurs because $+1$ is assigned to safe but often low-reward segments, as they are only marginally better than their counterparts. BC-Safe-Seg achieves relatively high average rewards, but this is largely due to unsafe agents. Safe-RLHF (CDT) manages to learn higher rewards but struggles with safe policy learning. \textsc{PreSa} outperforms all Offline Safe POHF baselines by successfully adhering to implicit safety constraints while also achieving high rewards. Similarly, although the Offline Safe RL baselines achieve high rewards, they struggle to learn safe behaviors while \textsc{PreSa} shows significantly better performance in 
% learning safe policies.
safety alignment.
Notably, the Offline Safe RL baselines have access to ground truth data, while \textsc{PreSa} relies solely on human feedback which has much less information. Despite this, \textsc{PreSa} matches or exceeds the performance of these baselines for safety alignment, highlighting its effectiveness.

\begin{figure}[t]
    \centering
    \begin{subfigure}[t]{0.49\textwidth}
        \centering
        \includegraphics[width=0.48\textwidth]{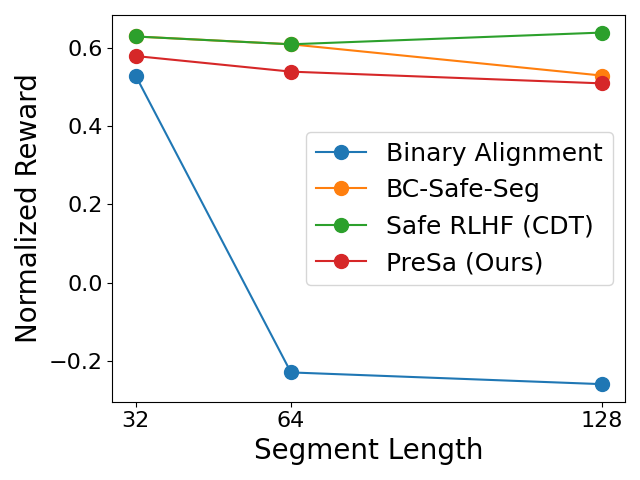}
        \includegraphics[width=0.48\textwidth]{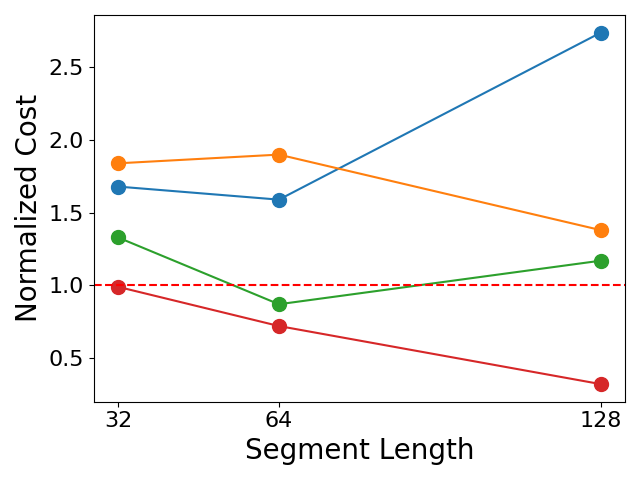}
        \caption{Different segment lengths.}
    \end{subfigure}
    ~ 
    
    \begin{subfigure}[t]{0.49\textwidth}
        \centering
        \includegraphics[width=0.48\textwidth]{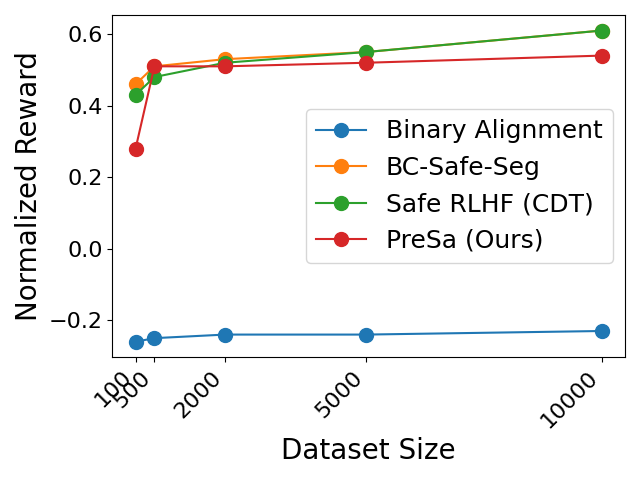}
        \includegraphics[width=0.48\textwidth]{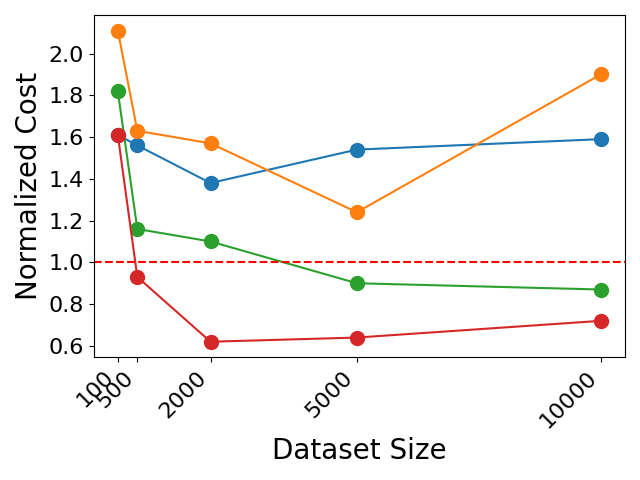}
        \caption{Different sizes of the offline dataset.}
    \end{subfigure}
    \caption{Performance of Offline Safe POHF methods across varying trajectory segment lengths, dataset sizes.}
    \label{fig:varying}
\end{figure}

\begin{table*}
\caption{Ablation study of varying values of $\alpha$ and $\beta$.}
  \centering
  \resizebox{\textwidth}{!}{
  \begin{tabular}{ccccccccc|cccccccc}
    \toprule
    \multirow{2}[2]{*}{Task} & \multicolumn{2}{c}{$\alpha=0.2$} & \multicolumn{2}{c}{$\alpha=0.4$} & \multicolumn{2}{c}{$\alpha=0.6$} & \multicolumn{2}{c}{$\alpha=0.8$} & \multicolumn{2}{c}{$\beta=0.25$} & \multicolumn{2}{c}{$\beta=0.5$} & \multicolumn{2}{c}{$\beta=0.75$} & \multicolumn{2}{c}{$\beta=1.0$} \\
    \cmidrule(lr){2-3}
    \cmidrule(lr){4-5}
    \cmidrule(lr){6-7}
    \cmidrule(lr){8-9}
    \cmidrule(lr){10-11}
    \cmidrule(lr){12-13}
    \cmidrule(lr){14-15}
    \cmidrule(lr){16-17}
      &   reward$\uparrow$    &   cost$\downarrow$ &   reward$\uparrow$    &   cost$\downarrow$  &   reward$\uparrow$    &   cost$\downarrow$ &   reward$\uparrow$    &   cost$\downarrow$  &   reward$\uparrow$    &   cost$\downarrow$ &   reward$\uparrow$    &   cost$\downarrow$  &   reward$\uparrow$    &   cost$\downarrow$ &   reward$\uparrow$    &   cost$\downarrow$     \\
    \midrule
    BallRun   &   0.19    &   0.09    &   0.19    &   0.1    &   0.38    &   2.8    &   0.19    &   0.12    &   0.18    &   0.11    &   0.18    &   0.09    &   0.34    &   2.67    &   0.19    &   0.09    \\
    DroneCircle   &   0.54    &   0.72    &   0.5    &   0.92    &   0.27    &   1.9    &   0.54    &   1.09    &   0.52    &   1.04    &   0.49    &   0.98    &   0.32    &   1.45    &   0.54    &   0.72    \\
    \bottomrule
  \end{tabular}
  }
  \label{tab:ablation}
\end{table*}

\begin{table*}[t]
  \centering
  \caption{Performance of \textsc{PreSa} with varying $\delta$.}
  % \resizebox{\textwidth}{!}{
  \begin{tabular}{ccccccccccc}
    \toprule
    \multirow{2}[2]{*}{Task} & \multicolumn{2}{c}{$\delta=0.55$} & \multicolumn{2}{c}{$\delta=0.65$} & \multicolumn{2}{c}{$\delta=0.75$} & \multicolumn{2}{c}{$\delta=0.85$} & \multicolumn{2}{c}{$\delta=0.95$} \\ %(unnormalized)} \\
    \cmidrule(lr){2-3}
    \cmidrule(lr){4-5}
    \cmidrule(lr){6-7}
    \cmidrule(lr){8-9}
    \cmidrule(lr){10-11}
      &   reward$\uparrow$    &   cost$\downarrow$ &   reward$\uparrow$    &   cost$\downarrow$ &   reward$\uparrow$    &   cost$\downarrow$ &   reward$\uparrow$    &   cost$\downarrow$  &   reward$\uparrow$    &   cost$\downarrow$    \\
    \midrule
    BallRun &   0.35    &   1.94    &   \textbf{0.22}    &   \textbf{0.43}    &   \textbf{0.2}    &   \textbf{0.2}    &   \textbf{0.2} &   \textbf{0.17}    &   \textbf{0.19}    &   \textbf{0.09}   \\
    % CarRun  &   \textbf{0.95}    &   \textbf{0.0}    &   \textbf{0.95}    &   \textbf{0.0}    &   \textbf{0.95}    &   \textbf{0.0}    &   \textbf{0.95}    &   \textbf{0.0}    &   \textbf{0.96}    &   \textbf{0.0}  \\
    % DroneRun    &   0.11    &   3.07    &   0.13 &   2.72 &   0.2    &   2.86    &   0.21    &   2.63    &   \textbf{0.16}    &   \textbf{0.33}    \\
    % AntRun  &   0.73    &   3.87    &   0.72    &   3.63    &   0.72    &   3.3    &   0.7    &   3.15    &   \textbf{0.61}    &   \textbf{0.63}  \\
    % BallCircle  &   \textbf{0.33}    &   \textbf{0.91}    &   \textbf{0.28}    &   \textbf{0.65}    &   \textbf{0.24}    &   \textbf{0.51}    &   \textbf{0.19}    &   \textbf{0.28}    &   \textbf{0.22}    &   \textbf{0.03}  \\
    % CarCircle  &   0.37    &   2.7    &   0.26    &   1.18    &   0.24    &   1.01    &   \textbf{0.23}    &   \textbf{0.87}    &   \textbf{0.08}    &   \textbf{0.91}  \\
    DroneCircle &   0.6    &   2.21    &   0.59    &   1.83    &   0.56    &   1.43    &   0.54    &   1.13    &   \textbf{0.54}    &   \textbf{0.72}  \\
    % AntCircle  &   0.65    &   5.62    &   0.63    &   5.56    &   0.64    &   5.71    &   0.63    &   5.16    &   0.55    &   3.78  \\
    % \midrule
    % \textbf{BulletGym Average}   &  0.51    &   2.54    &   0.47    &   2.0    &  0.47    &   1.88    &   0.46    &   1.68    &   \textbf{0.41}    &   \textbf{0.81}    \\
    \bottomrule
  \end{tabular}
  % }
  \label{tab:delta}
\end{table*}

\subsubsection{Performance of Preference and Safety Alignment Modules}
We show the results for preference alignment and safety alignment modules separately in the BulletGym domain in Figure~\ref{fig:module_performance}. The preference alignment module is designed to learn a policy that aligns with human pairwise preferences, implicitly capturing reward information without accounting for safety. Consequently, the policy learned by this module tends to achieve relatively high rewards. In contrast, the safety alignment module leverages binary safety labels, which encode implicit safety constraints, and successfully learns safe policies for most tasks (most dots to the left of the cost threshold of $1$). The effectiveness of each alignment module forms a solid foundation for learning high-reward, safe policies. When both modules are integrated, \textsc{PreSa} refines the preference-aligned policy to prioritize safer behaviors simultaneously.

\begin{figure*}[t]
    \centering
    \begin{subfigure}[t]{0.3\textwidth}
        \centering
        \includegraphics[width=\textwidth]{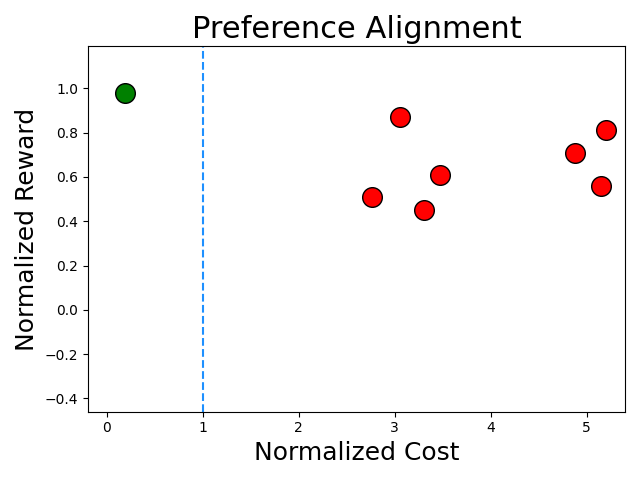}
        % \caption{SafetyBallCircle-v0}
    \end{subfigure}%
    ~ 
    \begin{subfigure}[t]{0.3\textwidth}
        \centering
        \includegraphics[width=\textwidth]{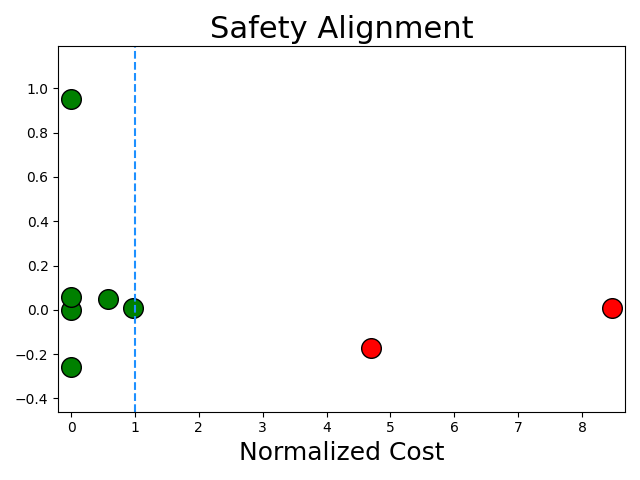}
        % \caption{SafetyGymnasiumPointButton1-v0}
    \end{subfigure}
    ~ 
    \begin{subfigure}[t]{0.3\textwidth}
        \centering
        \includegraphics[width=\textwidth]{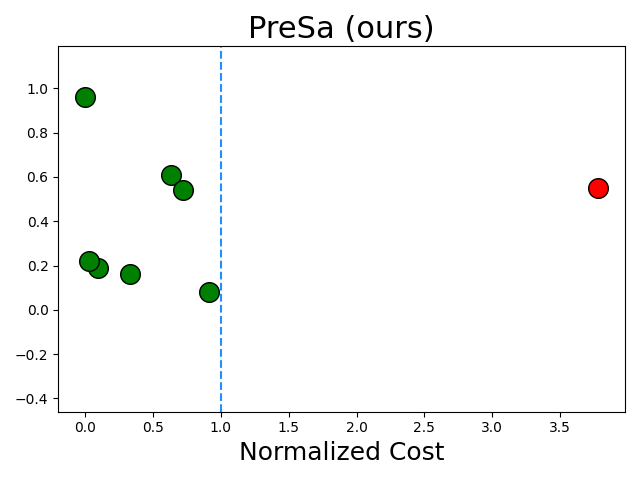}
        % \caption{SafeMetaDrive-hardsparse}
    \end{subfigure}
    \caption{Visualization of normalized reward and cost for each task within BulletGym domain. The dotted blue vertical lines mark the cost threshold of $1$. Each round dot represents a task, where green dots indicate tasks meeting safety constraints, and red dots indicate tasks with constraint violations.}
    \label{fig:module_performance}
\end{figure*}

\subsubsection{Varying Segment Lengths, Dataset Sizes.}
To investigate how \textsc{PreSa} performs in different settings, we evaluate its effectiveness compared to baselines across varying segment lengths and offline dataset sizes. The results of DroneCircle task, shown in Figure~\ref{fig:varying}, indicate that as the length of the segments increases, \textsc{PreSa} consistently learns policies with lower costs while maintaining stable reward performance, compared to other baselines. Moreover, \textsc{PreSa} demonstrates a stable and better performance as more offline data becomes available. These results show that \textsc{PreSa} outperforms other baselines across different segment lengths and dataset sizes.

% Additional experimental results, including an evaluation of \textsc{PreSa}'s effectiveness under noisy feedback and extensive ablation studies, are provided in the Appendix.

\subsubsection{What Contributes To \textsc{PreSa}'s Performance?}
Ablation studies are conducted to assess the impact of varying the hyperparameters $\alpha$ and $\beta$, which are used by the preference alignment and safety alignment modules, respectively, to regularize the learning policy with a pretrained reference policy. The results for BallRun and DroneCircle tasks are shown in Table~\ref{tab:ablation}. We find that \textsc{PreSa} remains robust across different choices of $\alpha$ and $\beta$, although higher performance could likely be achieved with further hyperparameter tuning. 
We also evaluate the effect of varying the safety threshold $\delta$, with results shown in Table~\ref{tab:delta}. Higher values of $\delta$ impose a stricter safety criterion on the learned policy, leading to lower incurred costs and safer behaviors. This demonstrates that $\delta$ effectively controls the desired safety level of the policy.

\subsection{Experiments With Real Human Feedback}

% \begin{wraptable}{r}{0.5\textwidth}
%   \caption{ Experimental Results With Real Human Feedback, evaluated using 10 trajectories per approach across 3 seeds. \textit{Task Comp}: Normalized relative preference for policy behaviors compared to \textsc{PreSa} in task completion. Values $>1.0$ indicate better performance, $<1.0$ worse performance.
% \textit{Safety}: Percentage of safe behaviors generated by each method.}
%   \centering
%   \resizebox{\linewidth}{!}{
%   \begin{tabular}{ccc|cc|cc}
%     \toprule
%     \multirow{2}[2]{*}{Task} & \multicolumn{2}{c}{CPL \emph{(w/ general pref)}} & \multicolumn{2}{c}{\textsc{PreSa}} & \multicolumn{2}{c}{Safe RLHF (CDT)} \\
%     \cmidrule(lr){2-3}
%     \cmidrule(lr){4-5}
%     \cmidrule(lr){6-7}
%       &   task comp    &   safety &   task comp    &   safety &   task comp    &   safety    \\
%     \midrule
%     Intersection   &   0.95    &   0.8    &   \textbf{1.0}    &   \textbf{0.87}    &   0.92    &   0.7    \\
%     U-Turn   &   0.87    &   0.73    &   \textbf{1.0}    &   \textbf{0.83}    &   0.82    &   0.55    \\
%     \midrule
%     Average  &  0.91    &   0.77    &   \textbf{1.0}    &   \textbf{0.85}    &   0.87    &   0.63      \\
%     \bottomrule
%   \end{tabular}
%   }
%   \label{tab:real_human}
% \end{wraptable}

% To further assess the effectiveness of our approach and highlight the significance of Offline Safe POHF problem setting, 
In this section, we conduct experiments using real human feedback. The offline data and human feedback are pre-collected for autonomous driving tasks in the Highway environment~\cite{highway-env} by following the offline data generation pipeline proposed in DSRL~\cite{liu2023datasets} and feedback collection process commonly used in prior offline PbRL studies~\cite{kimpreference, hejna2024inverse, hejnacontrastive}.
% First, we generate 100 pairs of trajectory segments by leveraging policies which are pretrained with varying cost functions, following the offline dataset generation pipeline proposed in DSRL~\cite{daisafe}. Next, we adhere strictly to the feedback collection process commonly used in prior offline PbRL studies~\cite{kimpreference, hejna2024inverse, hejnacontrastive}, gathering feedback from real human. 
For each trajectory pair, we collect not only pairwise preferences regarding task completion and binary safety labels, but also ``general preferences'' which assess the overall performance of the agent, considering both reward and cost. \footnote{To mitigate the immediate impact of collecting one type of feedback when gathering another, we collect one type of feedback per run and shuffle the order of the data to be labeled.} These general preferences are used solely to train a policy, allowing us to investigate whether separating reward and safety feedback, as introduced in Offline Safe POHF, is beneficial, since ``general preferences'' implicitly encode information about both reward and cost as well. In real human experiments, where ground truth reward or cost functions are unavailable, human evaluators are also responsible for assessing the learned policies in terms of two aspects: task completion quality (i.e., rewards) and safety (i.e., cost). The sampled trajectories from each method are shuffled and randomly presented to the human evaluators, ensuring that they are unaware of which method the current trajectory is sampled from, to allow for a fair evaluation.

% \begin{table}
% \caption{\small Experimental Results With Real Human Feedback, evaluated using 10 trajectories per approach across 3 seeds. \textit{Task Comp}: Normalized relative preference for policy behaviors compared to \textsc{PreSa} in task completion. Values $>1.0$ indicate better performance, $<1.0$ worse performance.
% \textit{Safety}: Percentage of safe behaviors generated by each method.}
%   \centering
%   % \small
%   \resizebox{\columnwidth}{!}{
%   \begin{tabular}{ccc|cc|cc}
%     \toprule
%     \multirow{2}[2]{*}{Task} & \multicolumn{2}{c}{CPL \emph{(w/ general pref)}} & \multicolumn{2}{c}{\textsc{PreSa}} & \multicolumn{2}{c}{Safe RLHF (CDT)} \\
%     \cmidrule(lr){2-3}
%     \cmidrule(lr){4-5}
%     \cmidrule(lr){6-7}
%       &   task comp    &   safety &   task comp    &   safety &   task comp    &   safety    \\
%     \midrule
%     Intersection   &   0.95    &   0.8    &   \textbf{1.0}    &   \textbf{0.87}    &   0.92    &   0.7    \\
%     U-Turn   &   0.87    &   0.73    &   \textbf{1.0}    &   \textbf{0.83}    &   0.82    &   0.55    \\
%     \midrule
%     Average  &  0.91    &   0.77    &   \textbf{1.0}    &   \textbf{0.85}    &   0.87    &   0.63      \\
%     \bottomrule
%   \end{tabular}
%   }
%   \label{tab:real_human}
% \end{table}

\begin{table}
  \centering
  \caption{ Experimental Results With Real Human Feedback, evaluated using 10 trajectories per approach across 3 seeds. \textit{Task Comp}: Normalized relative preference for policy behaviors compared to \textsc{PreSa} in task completion. Values $>1.0$ indicate better performance, $<1.0$ worse performance.
\textit{Safety}: Percentage of safe behaviors generated by each method.}
  \resizebox{\columnwidth}{!}{
  \begin{tabular}{ccc|cc|cc}
    \toprule
    \multirow{2}[2]{*}{Task} & \multicolumn{2}{c}{CPL \emph{(w/ general pref)}} & \multicolumn{2}{c}{\textsc{PreSa} (Ours)} & \multicolumn{2}{c}{Safe RLHF (CDT)} \\
    \cmidrule(lr){2-3}
    \cmidrule(lr){4-5}
    \cmidrule(lr){6-7}
      &   task comp    &   safety &   task comp    &   safety &   task comp    &   safety    \\
    \midrule
    Intersection   &   0.95    &   0.8    &   \textbf{1.0}    &   \textbf{0.87}    &   0.92    &   0.7    \\
    U-Turn   &   0.87    &   0.73    &   \textbf{1.0}    &   \textbf{0.83}    &   0.82    &   0.55    \\
    \midrule
    Average  &  0.91    &   0.77    &   \textbf{1.0}    &   \textbf{0.85}    &   0.87    &   0.63      \\
    \bottomrule
  \end{tabular}
  }
  \label{tab:real_human}
\end{table}

% \subsubsection{Results}

\subsubsection{Performance of \textsc{PreSa} With Real Human Feedback.} 
Our experiments investigate two typical autonomous driving tasks: \emph{Intersection} and \emph{U-Turn}. In both tasks, the ego-vehicle must manage lane changes and longitudinal control to reach the goal as quickly as possible while avoiding collisions. We compare the performance of \textsc{PreSa} with Safe RLHF (CDT).
% as well as CPL with ``general preferences''.
As shown in Table~\ref{tab:real_human}, \textsc{PreSa} outperforms it in both task completion and maintaining safe behavior, highlighting that \textsc{PreSa} performs effectively with real human feedback as well.
\subsubsection{Significance of Offline Safe POHF} 
With ``general preferences'', we train a policy using CPL, assuming it implicitly encodes both reward and cost. In contrast, \textsc{PreSa} applies Offline Safe POHF, where reward and cost are leveraged separately from two types of feedback. As shown in Table~\ref{tab:real_human}, \textsc{PreSa} outperforms CPL with ``general preferences'' in both metrics, providing empirical evidence of the significance of separate feedback in Offline Safe POHF. Specifically, the added safety feedback improves safe policy learning.

%%%%%%%%%%%%%%%%%%%%%%%%%%%%%%%%%%%%%%%%%%%%%%%%%%%%%%%%%%%%%%%%%%%%%%%%

\section{Conclusion}

In this paper, we present Offline Safe POHF, a framework for learning policies using human pairwise preferences and binary safety labels for each trajectory segment, without access to ground truth rewards or costs. We first analyze the problem through two modules: preference alignment and safety alignment, which can be applied separately for learning based on preferences or safety labels. We then introduce \textsc{PreSa}, which integrates both modules into a unified constrained optimization objective, as the safety alignment module defines a feasible policy set. \textsc{PreSa} learns a policy directly from offline human feedback, without the needs for reward or cost models or constrained RL. Empirical results with synthetic and real human feedback show that \textsc{PreSa} outperforms Offline Safe POHF baselines as well as matches or surpasses offline safe RL methods with ground truth rewards and costs.

% \paragraph{Limitations}

% Although \textsc{PreSa} demonstrates promising performance, there are several limitations that deserve further discussion and investigation. First, our framework models human feedback using a regret-based approach grounded in the Bradley-Terry assumption. While this modeling choice is common and effective, it may not always be the most appropriate, especially when handling noisy human feedback in real-world scenarios. Relaxing this assumption is an important avenue for future research.
% Second, although \textsc{PreSa} enables direct policy learning and bypasses explicit reward and cost modeling, it still involves solving a constrained optimization problem, which can be complex and nontrivial. Exploring methods to simplify this optimization process without sacrificing performance represents a promising direction for future work.
% Finally, the effectiveness of POHF is highly dependent on the quality and quantity of available data. Model performance can degrade significantly when trajectory data or human feedback are sparse, noisy, or biased. Developing techniques to ensure robust learning under such conditions remains an open and important challenge.

%%%%%%%%%%%%%%%%%%%%%%%%%%%%%%%%%%%%%%%%%%%%%%%%%%%%%%%%%%%%%%%%%%%%%%%%

%%% The acknowledgments section is defined using the "acks" environment
%%% (rather than an unnumbered section). The use of this environment 
%%% ensures the proper identification of the section in the article 
%%% metadata as well as the consistent spelling of the heading.

\begin{acks}
This research/project is supported by the National Research Foundation Singapore and DSO National Laboratories under the AI Singapore Programme (Award Number: AISG2-RP-2020-016).
\end{acks}

%%%%%%%%%%%%%%%%%%%%%%%%%%%%%%%%%%%%%%%%%%%%%%%%%%%%%%%%%%%%%%%%%%%%%%%%

%%% The next two lines define, first, the bibliography style to be 
%%% applied, and, second, the bibliography file to be used.

\balance

\bibliographystyle{ACM-Reference-Format} 
\bibliography{sample}

%%%%%%%%%%%%%%%%%%%%%%%%%%%%%%%%%%%%%%%%%%%%%%%%%%%%%%%%%%%%%%%%%%%%%%%%
\newpage
\appendix

\section{Theory}

\subsection{Score Function Justification}
\label{app:score}

\begin{lemma} 
For a trajectory segment $\sigma$ and a given policy $\pi$, the utility function $u(\sigma;\pi)$ is directly proportional to the difference between the log probability of $\sigma$ under $\pi$ and the reference point $z_{\text{ref}}$, which is given by:
$u(\sigma;\pi)=\log p(\sigma; \pi) - z_{\text{ref}}(\pi)+ \text{constants}$.
\end{lemma}
% \paragraph{Lemma 5.1.}
% For a trajectory segment $\sigma$ and a given policy $\pi$, the utility function $u(\sigma;\pi)$ is directly proportional to the difference between the log probability of $\sigma$ under $\pi$ and the reference point $z_{\text{ref}}$, which is given by:
% $u(\sigma;\pi)=\log p(\sigma; \pi) - z_{\text{ref}}(\pi)+ \text{constants}$.

% The loss function can also be stated as
% \begin{align}
%     L_\text{safety}&(\pi_\theta, D) = \mathbb{E}_{\sigma\sim D_\text{safe}} [\texttt{sigmoid}\left( -u(\sigma;\pi_\theta) \right) ]
%     +  \mathbb{E}_{\sigma\sim D_\text{unsafe}} \left[ \texttt{sigmoid}\left( u(\sigma;\pi_\theta) \right) \right]
%     %\label{eq:safe_obj}
% \end{align}

% We first make an assumption that discount factor $\gamma$ is close to 1, which is often the case in practice. 

% \begin{lemma}
% Minimizing the loss function in Equation~\ref{eq:safe_obj} increases the log probability of safe trajectories in the safe dataset $D_\text{safe}$ and decreases the log probability of trajectories in the unsafe dataset $D_\text{unsafe}$ under policy $\pi$.
% \end{lemma}

\begin{proof}
For simplicity, we ignore the $\beta$ term from Equation
% ~\ref{eq:score}
6. We have:
\begin{align}
u(\sigma;\pi) \approx \sum_{t=0}^T \log \pi(a_t|s_t) - \sum_{t=0}^T \log \pi_{\text{ref}}(a_t|s_t) - z_{\text{ref}}(\pi) \label{eq:scorepolicy}
\end{align}
The approximation sign is due to ignoring the discount factor (being close to 1); $z_{\text{ref}}(\pi)$ is defined as the average over all trajectories (both safe, unsafe): $\mathbb{E}_\sigma \left[ \sum_{t=0}^T \gamma^t \beta \log \frac{\pi(a_t | s_t)}{\pi_\text{ref}(a_t | s_t)} \right]$.

The log progability of a trajectory segment $\sigma$ as per policy $\pi$ is given as:
\begin{equation}
    \log p(\sigma; \pi) = \sum_{t=0}^T \log \pi(a_t|s_t) + \text{constants}
\end{equation}
where constant terms refer to the log of transition function, which is independent of $\pi$. Using Equation~\ref{eq:scorepolicy}, we have:
\begin{align}
&u(\sigma;\pi) \nonumber \\
&\approx \log p(\sigma; \pi) - \sum_{t=0}^T \log \pi_{\text{ref}}(a_t|s_t) - z_{\text{ref}}(\pi) + \text{constants} \nonumber  \\ 
&= \log p(\sigma; \pi) - z_{\text{ref}}(\pi) + \text{constants} \label{eq:logprob}
\end{align}
Notice that $z_{\text{ref}}(\pi)$ only depends on policy $\pi$ and is the same for all the trajectories in the safe and unsafe datasets. Minimizing the loss function in Equation
% ~\ref{eq:safe_obj}
7 would optimize the policy $\pi_\theta$ such that for safe policies, their utility score increases. This is, the score $u(\sigma;\pi_\theta)=\log p(\sigma; \pi_\theta) - z_{\text{ref}}(\pi_\theta)$ increases, implying that the log-probability of safe trajectory segments under policy $\pi_\theta$ should move higher than reference point $z_{\text{ref}}(\pi_\theta)$, which is as desired in the HALO formulation. Analogous reasoning applies when $\sigma$ is unsafe, its utility should decrease.
% higher score $u$ is assigned to safe trajectories $\sigma^+$ and lower scores to $\sigma^-$. Thus, as per Equation~\ref{eq:logprob}, log probabilities of safe trajectories would tend to increase and unsafe trajectories would tend to decrease.
\end{proof}

\subsection{Theoretical Guarantees for \textsc{PreSa}}
\label{subsec:theory-presa}

\begin{lemma}[Feasibility generalization bound]
\label{lemma:feasibility-bound-corrected}
Let $\mathcal{D}=\{(\sigma_i,y_i)\}_{i=1}^N$ be drawn i.i.d.\ from an unknown distribution $\mathcal P$ over trajectory segments and binary safety labels. Define the function class
\begin{equation}
    \mathcal G = \left\{ g_\theta:(\sigma,y)\mapsto \texttt{sigmoid} \left(y\left(\psi_{\pi_\theta}(\sigma)-z_{\text{ref}} \right)\right)\;:\;\theta\in\Theta\right\},
\end{equation}
and the empirical and population safety scores
\begin{equation*}
    \hat F_{\mathcal D}(\pi_\theta)=\frac{1}{N}\sum_{i=1}^N g_\theta(\sigma_i,y_i),\qquad
    F(\pi_\theta)=\mathbb{E}_{(\sigma,y)\sim\mathcal P} \left[g_\theta(\sigma,y) \right].
\end{equation*}
Assume each $g_\theta$ takes values in $[0,1]$ (true for \texttt{sigmoid}). Then for any confidence level $\tau\in(0,1)$, with probability at least $1-\tau$ the following holds simultaneously for all $\pi_\theta\in\Pi_\Theta$:
\begin{equation}\label{eq:two-sided-bound}
    \big|F(\pi_\theta)-\hat F_{\mathcal D}(\pi_\theta)\big|
    \le 2\,\mathfrak R_N(\mathcal G) + \sqrt{\frac{\ln(2/\tau)}{2N}},
\end{equation}
where $\mathfrak R_N(\mathcal G)$ denotes the Rademacher complexity of $\mathcal G$.
\end{lemma}

\begin{proof}
Let $\Phi(\mathcal D)=\sup_{g\in\mathcal G}\left(F(g)-\hat F_{\mathcal D}(g)\right)$ denote the one-sided uniform deviation between population and empirical averages over the class $\mathcal G$.\footnote{For brevity, we write $g \in \mathcal G$ to denote functions of the form $g_\theta(\sigma, y) = \texttt{sigmoid}(y(\psi_{\pi_\theta}(\sigma) - z_{\mathrm{ref}}))$, and suppress the parameter $\theta$ in the derivation.} Writing $F(g)=\mathbb{E}_{\mathcal D'}[\hat F_{\mathcal D'}(g)]$ for an independent sample dataset $\mathcal D'$, and applying Jensen's inequality to the supremum, we obtain the standard symmetrization identity
\begin{equation}
    \mathbb{E}_{\mathcal D}[\Phi(\mathcal D)]
    \le \mathbb{E}_{\mathcal D,\mathcal D'}\left[\sup_{g\in\mathcal G}\frac{1}{N}\sum_{i=1}^N\left(g\left(\sigma'_i,y'_i\right)-g(\sigma_i,y_i)\right)\right].
\end{equation}
Introduce independent Rademacher random variables \(\epsilon_1,\dots,\epsilon_N\), i.e. independent variables taking values in \(\{-1,+1\}\) with \(\Pr(\epsilon_i=1)=\Pr(\epsilon_i=-1)=1/2\). The role of these signs is to ``symmetrize'' the two-sample difference that appears after introducing $\mathcal D'=\{(\sigma'_i,y'_i)\}_{i=1}^N$. Concretely, 
\begin{align*}
    &\mathbb{E}_{\mathcal{D}, \mathcal{D}', \boldsymbol{\epsilon}}\left[ \sup_{g\in\mathcal G} \frac{1}{N} \sum_{i=1}^N \epsilon_i \left( g(\sigma'_i, y'_i) - g(\sigma_i, y_i) \right) \right] \\
    \leq\, &\mathbb{E}_{\mathcal{D}', \boldsymbol{\epsilon}}\left[ \sup_{g\in\mathcal G} \frac{1}{N} \sum_{i=1}^N \epsilon_i g(\sigma'_i, y'_i) \right] + \mathbb{E}_{\mathcal{D}, \boldsymbol{\epsilon}}\left[ \sup_{g\in\mathcal G} \frac{1}{N} \sum_{i=1}^N (-\epsilon_i) g(\sigma_i, y_i) \right] \\
    =\, & 2 \mathbb{E}_{\mathcal{D}, \boldsymbol{\epsilon}}\left[ \sup_{g\in\mathcal G} \frac{1}{N} \sum_{i=1}^N \epsilon_i g(\sigma_i, y_i) \right] = 2 \mathfrak{R}_N(\mathcal{G}).
\end{align*}
Thus the expectation of the supremum deviation is at most $2\mathfrak R_N(\mathcal G)$. 

Because each $g\in\mathcal G$ takes values in $[0,1]$, replacing a single example in the sample can change any empirical average $\hat F_{\mathcal D}(g)=\tfrac{1}{N}\sum_{i=1}^N g(\sigma_i,y_i)$ by at most $1/N$. More precisely, if $\mathcal D$ and $\mathcal D^{(j)}$ differ only in the $j$-th example, then for any $g\in\mathcal G$,
\begin{equation}
    \big|\hat F_{\mathcal D}(g)-\hat F_{\mathcal D^{(j)}}(g)\big|
    = \frac{1}{N}\big|g(\sigma_j,y_j)-g(\tilde\sigma_j,\tilde y_j)\big|
    \le \frac{1}{N}.
\end{equation}
Since the population term \(F(g)\) does not depend on the sample, this implies:
\begin{align*}
    &\big|\Phi(\mathcal D)-\Phi(\mathcal D^{(j)})\big| \\
    = &\left|\sup_{g}(F(g)-\hat F_{\mathcal D}(g))-\sup_{g}(F(g)-\hat F_{\mathcal D^{(j)}}(g))\right| \\
    \le &\sup_{g}\left|\hat F_{\mathcal D}(g)-\hat F_{\mathcal D^{(j)}}(g)\right|
    \le \frac{1}{N},
\end{align*}
so $\Phi$ satisfies the bounded-differences property with constants $c_i=1/N$. Applying McDiarmid's inequality with $\sum_{i=1}^N c_i^2 = N\cdot(1/N)^2 = 1/N$ yields, for any $\epsilon>0$,
\begin{equation}
    \Pr\big(\Phi(\mathcal D)-\mathbb{E}[\Phi(\mathcal D)]\ge \epsilon\big)
    \le \exp\!\big(-2N\epsilon^2\big).
\end{equation}
To obtain a bound that holds with confidence at least $1-\tau/2$, we choose $\epsilon$ such that the right-hand side equals $\tau/2$, i.e.,
$\exp(-2N\epsilon^2)=\tau/2$. Solving for $\epsilon$ gives
$\epsilon=\sqrt{\tfrac{\ln(2/\tau)}{2N}}$. Substituting this into the inequality and using $\mathbb{E}[\Phi(\mathcal D)]\le 2\mathfrak R_N(\mathcal G)$ yields that, with probability at least $1-\tau/2$,
\begin{equation}
    \sup_{g\in\mathcal G} \left(F(g)-\hat F_{\mathcal D}(g)\right)
    \le 2\,\mathfrak R_N(\mathcal G)
    + \sqrt{\frac{\ln(2/\tau)}{2N}}.
\end{equation}
Applying the same argument to the negated function class $-\mathcal G = \{-g : g \in \mathcal G\}$, which has the same Rademacher complexity as $\mathcal G$, provides a symmetric bound for the reverse deviation $\sup_{g\in\mathcal G}\big(\hat F_{\mathcal D}(g)-F(g)\big)$ with the same confidence level. By the union bound, both one-sided inequalities hold simultaneously with probability at least $1-\tau$. Combining these two results yields the desired two-sided uniform generalization bound:
\begin{equation}
    \sup_{g\in\mathcal G}\!\left|F(g)-\hat F_{\mathcal D}(g)\right|
    \le 2\,\mathfrak R_N(\mathcal G)
    + \sqrt{\frac{\ln(2/\tau)}{2N}}.
    \label{eq:bound}
\end{equation}
Note that the above expression is uniform over $\mathcal G$; hence, the pointwise statement in Lemma~\ref{lemma:feasibility-bound-corrected} follows immediately since it holds simultaneously for all $\pi_\theta\in\Pi_\Theta$.
\end{proof}

\section{Experimental Details}

This section provides the experimental details required to reproduce the experiments and results presented in our paper.

\subsection{Task Description}

We conducted our experiments and generated synthetic feedback using the well-established DSRL benchmark~\cite{liu2023datasets}, which offers datasets specifically designed for offline safe RL research. This benchmark includes 38 datasets with various safe RL environments and difficulty levels in SafetyGym~\cite{ray2019benchmarking,ji2023omnisafe}, BulletGym~\cite{gronauer2022bullet}, and MetaDrive~\cite{li2022metadrive}.

\begin{itemize}
    \item \textbf{SafetyGym} is a suite of environments built on the Mujoco physics simulator, with a diverse set of tasks. It features two types of agents, \texttt{Car} and \texttt{Point}, each tasked with four different activities: \texttt{Button}, \texttt{Circle}, \texttt{Goal}, and \texttt{Push}. The difficulty of these tasks is further distinguished by levels, denoted by \texttt{1} and \texttt{2}. In each task, the agents must reach a goal while avoiding hazards and obstacles. Moreover, SafetyGym includes five velocity-constrained tasks for different agents: \texttt{Ant}, \texttt{HalfCheetah}, \texttt{Hopper}, \texttt{Walker2d}, and \texttt{Swimmer}. Figure~\ref{fig:safetygym} illustrates the agents and tasks within SafetyGym.
    \item \textbf{BulletGym} is a collection of environments built using the PyBullet physics simulator. Similar to SafetyGym, it focuses on safety-critical tasks but has shorter time horizons and a wider variety of agents. The suite includes four types of agents: \texttt{Ball}, \texttt{Car}, \texttt{Drone}, and \texttt{Ant}, each with two tasks: \texttt{Circle} and \texttt{Run}. The agents and tasks within BulletGym are shown in Figure~\ref{fig:bulletgym}.
    \item \textbf{MetaDrive} is built on the Panda3D game engine \cite{goslin2004panda3d}, providing complex road conditions and dynamic scenarios that closely emulate real-world driving, making it well-suited for evaluating safe RL algorithms in high-stakes, realistic environments. The environments feature three road types—\texttt{easy}, \texttt{medium}, and \texttt{hard}—each with varying levels of surrounding traffic: \texttt{sparse}, \texttt{mean}, and \texttt{dense}. Figure \ref{fig:metadrive} illustrates examples of the tasks in MetaDrive.
\end{itemize}

\begin{figure*}[h]
    \centering
    \subfloat[SafetyGym]{
        \includegraphics[height=5cm]{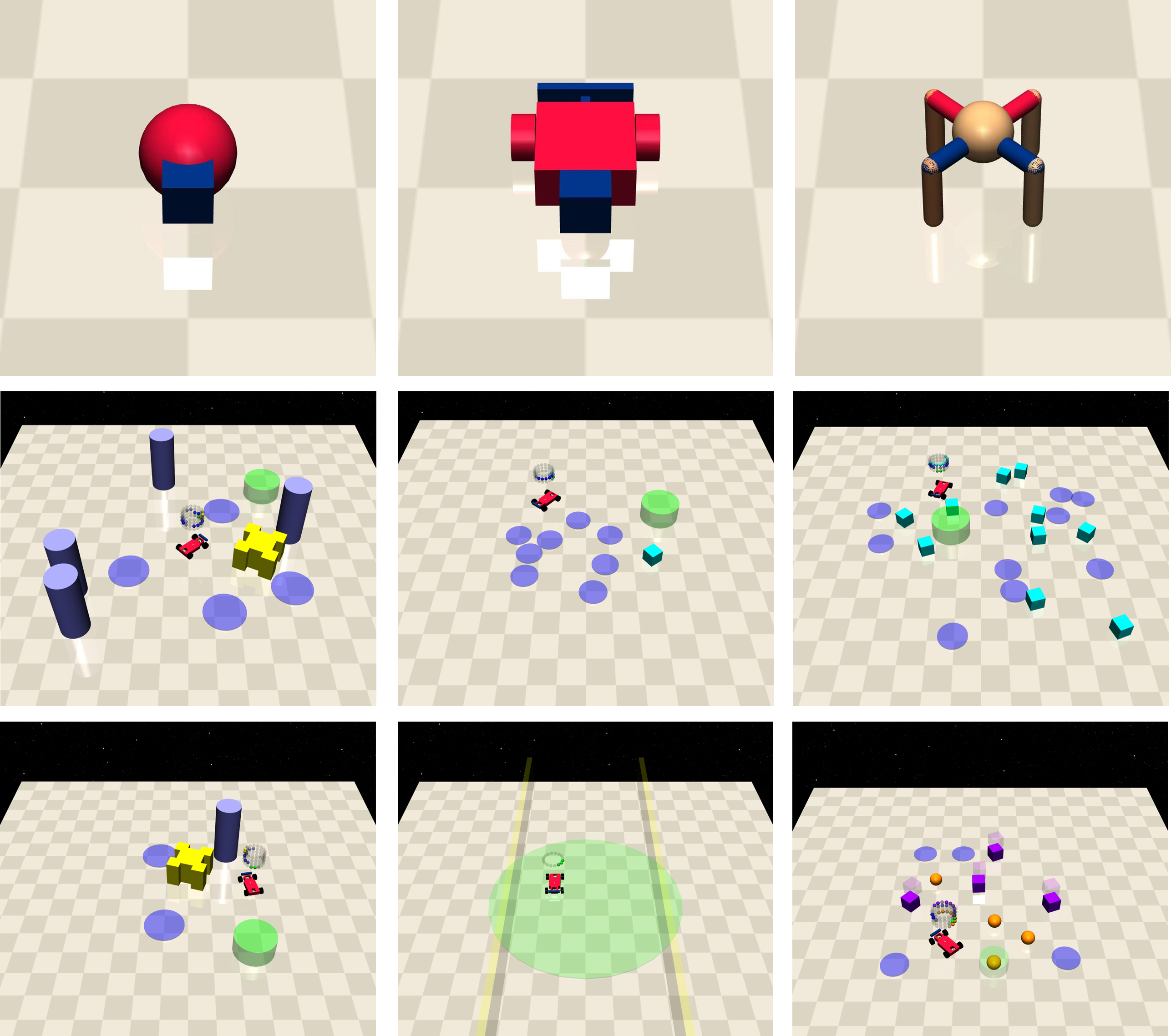}%
        \label{fig:safetygym}
    }
    \qquad
    \subfloat[BulletGym]{
        \includegraphics[height=5cm]{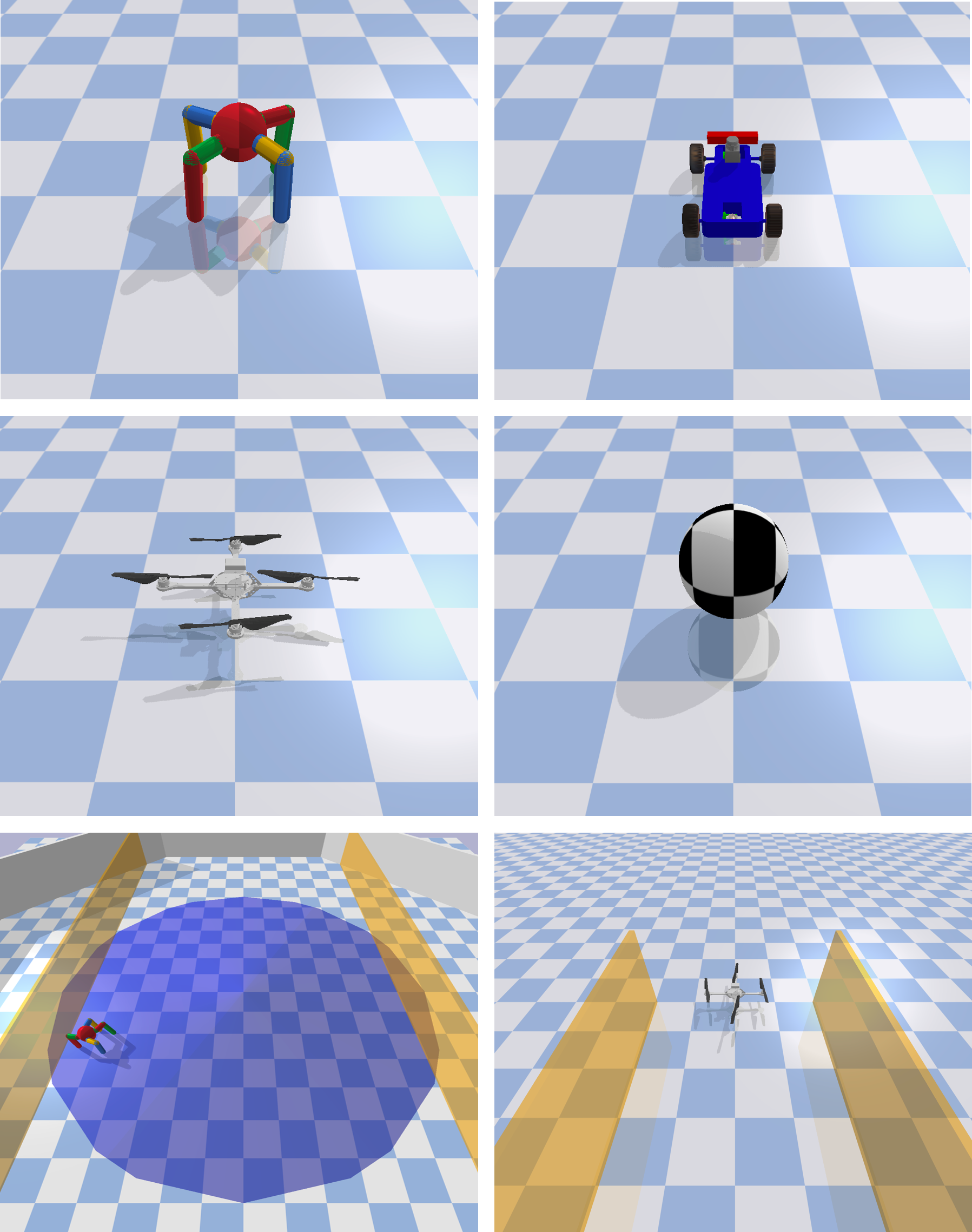}
        \label{fig:bulletgym}
    }
    \qquad
    \subfloat[MetaDrive]{
        \includegraphics[height=5cm]{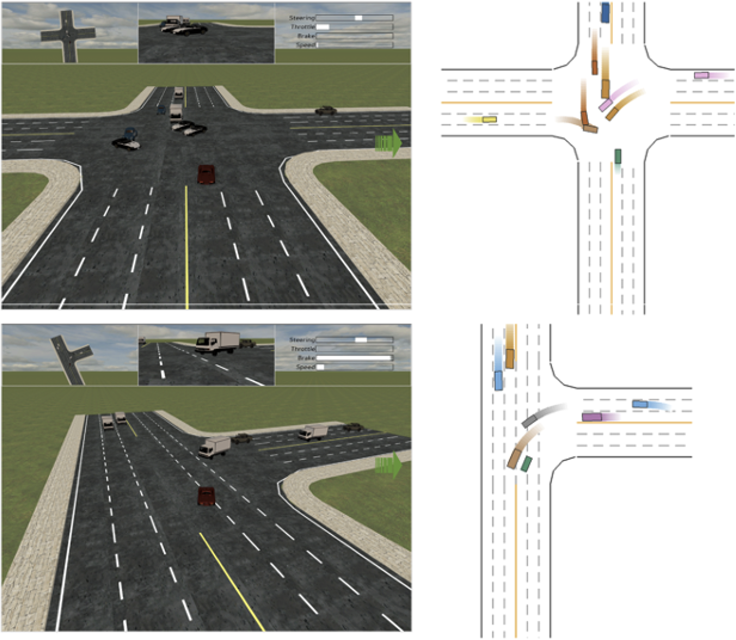}
        \label{fig:metadrive}
    }
    \caption{Visualization of agents and tasks in SafetyGym, BulletGym, and MetaDrive.}
\end{figure*} 

\subsection{Synthetic Feedback}

To evaluate our approach in the Offline Safe POHF setting, we generate synthetic feedback consisting of pairwise preferences regarding agent behavior and safety labels indicating whether the agent's actions are safe or not. For pairwise preferences, we follow the method used in CPL~\cite{hejnacontrastive}, assuming that the human model follows a regret-based preference framework. Given that negated regret is equivalent to the discounted sum of optimal advantages, we provide preference feedback based on the cumulative advantages of the compared trajectory segments. However, since obtaining optimal advantages requires running RL for each task, we simplify the process by using cumulative rewards instead. This is because optimal advantages are essentially a reshaped version of the reward, especially in the dense reward setting, leading to the same optimal policy~\cite{hejnacontrastive}.

For the binary safety labels, we use the ground truth cost data from the original offline dataset to assess the safety of each trajectory segment. Safety labels are then assigned based on predefined cost thresholds, which are generally defined for the entire task or trajectory. To label individual segments (which are parts of full trajectories), we proportionally adjust the cost threshold according to the segment’s length relative to the maximum trajectory length in each domain.

\subsection{Evaluation Metrics}

To assess the algorithm's performance, we adopt the evaluation methodology from DSRL~\cite{liu2023datasets}, using normalized reward and normalized cost as metrics. The normalized reward is defined as follows:
\begin{equation}
    R_\text{normalized} = \frac{R_\pi - r_\text{min}}{r_\text{max} - r_\text{min}} \nonumber
\end{equation}
where $R_\pi$ is the cumulative reward under policy $\pi$, and $r_\text{max}$ and $r_\text{min}$ denote the empirical maximum and minimum reward returns. The normalized cost is represented as:
\begin{equation}
    C_\text{normalized} = \frac{C_\pi + \epsilon}{\kappa + \epsilon} \nonumber
\end{equation}
where $C_\pi$ is the cumulative cost under policy $\pi$. The cost threshold is given by $\kappa$, and $\epsilon$ is a small positive constant added to ensure numerical stability when $\kappa=0$. According to the DSRL benchmark, a task is considered safe if the normalized cost does not exceed 1.

\subsection{Training Details and Hyperparameters}

Our approach follows a two-step training process. In the first step, we pretrain the policy using behavior cloning (BC) on all trajectory segments, establishing a reference policy, denoted as $\pi_\text{ref}$, which will later regulate the learning policy $\pi$. In the second step, we optimize the policy by applying \textsc{PreSa} with offline human feedback.

The hyperparameters used in the experiments are summarized in Table~\ref{tab:hyperparameters}. We assume all policies to be Gaussian with a fixed variance. Actions are predicted using a standard multi-layer perceptron (MLP), and the log probability $\log \pi(a|s)$ is calculated as $-| \pi(s) - a |^2_2$, following the implementation design from CPL. The policy networks are structured with two hidden layers, each containing 256 hidden units and employing ReLU activation functions, with dropout applied. Different values for the hyperparameters in \textsc{PreSa} are used depending on the specific domain.

\begin{table*}
  \centering
  \caption{Hyperparameters of \textsc{PreSa} for tasks in two domains.}
  % \resizebox{\columnwidth}{!}{
  \begin{tabular}{cccc}
    \toprule
    Hyperparameters &   BulletGym   & SafetGym     \\
    \midrule
    Training Steps   &   100k    &   200k    \\
    Pretraining Steps   &   30k    &   60k    \\
    Batch Size  &  32    &   96      \\
    Policy network architecture   &   [256, 256] MLP    &   [256, 256] MLP  \\
    Policy network dropout   &   0.1    &   0.25    \\
    Optimizer   &   Adam    &   Adam    \\
    Policy learning Rate   &   0.0001    &   0.0001    \\
    Temperature $\alpha$  &  0.2    &   0.2      \\
    Temperature $\beta$  &  1.0    &   0.2      \\
    Discount factor $\gamma$  &  1.0    &   1.0      \\
    Balancing factor $\eta$  &  0.1    &   2.0      \\
    Constraint lower bound $\delta$   &   0.95    &   0.9    \\
    Lagrangian multiplier learning rate   &   0.005    &   0.0005    \\
    \bottomrule
  \end{tabular}
  % }
  \label{tab:hyperparameters}
\end{table*}

\begin{figure*}[ht]
    \centering
    \subfloat[Intersection]{
        \includegraphics[height=4cm]{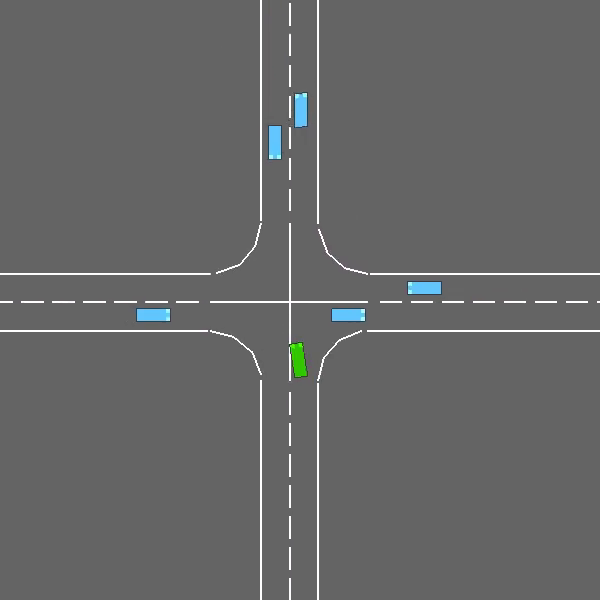}%
        \label{fig:highway-intersection}
    }
    \qquad
    \subfloat[U-Turn]{
        \includegraphics[height=3cm]{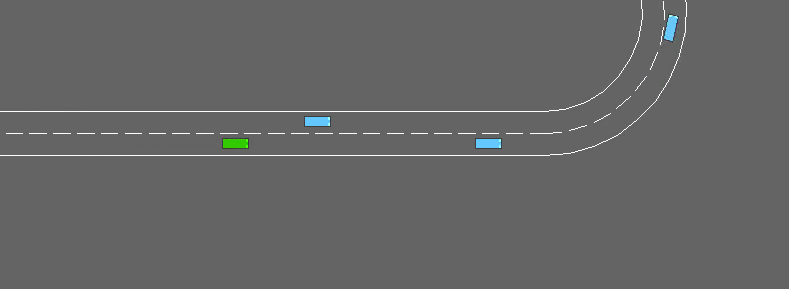}
        \label{fig:highway-uturn}
    }
    \caption{Visualization of autonomous driving tasks in Highway environment.}
\end{figure*} 

\subsection{Real Human Experimental Details}

We perform real human experiments on two autonomous driving tasks in Highway environment~\cite{highway-env}, specifically intersection and U-turn tasks.
\begin{itemize}
    \item \textbf{Intersection}: In this task, the ego-vehicle must navigate through a typical intersection, driving as quickly as possible to reach the destination while avoiding collisions. As illustrated in Figure~\ref{fig:highway-intersection}, The ego-vehicle is depicted in green, while other vehicles are shown in blue.
    \item \textbf{U-Turn}: As visualized in Figure~\ref{fig:highway-uturn}, the ego-vehicle approaches a U-turn, requiring lane changes and the ability to drive as fast as possible to reach the destination while avoiding collisions.
\end{itemize}

To gather offline human feedback, we collect data and feedback in two steps, following the data and feedback collection pipelines introduced in~\cite{liu2023datasets} and~\cite{kimpreference, hejna2024inverse, hejnacontrastive}, respectively.:
\begin{itemize}
    \item \textbf{Step 1 (Data Collection)}: To create a diverse set of datasets, we train a group of policies using different cost functions and save their intermediate checkpoint models. This process generates a collection of raw data that encompasses various task scenarios. From this pool, we then uniformly select pairs of trajectory segments to build the pairwise comparison dataset.
    \item \textbf{Step 2 (Feedback Collection)}: Each pair of trajectory segments is visualized and presented to a human labeler, who provides three types of feedback: pairwise preferences based on task completion, binary safety labels, and "general preferences" that consider both task performance (reward) and safety (cost).
\end{itemize}

To minimize the cognitive impact of feedback from one type influencing the labeler’s response to another, we collect each type of feedback separately in individual runs. During feedback collection, the unlabeled data is shuffled and presented in a random order for each feedback type. Furthermore, to ensure a fair evaluation, we randomly select trajectories from different testing approaches during the evaluation phase. This ensures that human evaluators are unaware of the specific approach from which the trajectories were sampled, helping to reduce potential biases that could arise if the evaluators knew the source of the trajectories.

\section{Extended Results}

In this section, we provide complete evaluation results as reported in Table~\ref{tab:complete_result} and supplementary results, including individual performance of preference alignment and safety alignment modules, \textsc{PreSa}'s performance across various segment lengths and dataset sizes, its performance with noisy safety feedback, the proportion of safe/unsafe trajectories under different threshold and the ablation study on segment weight $w(y_\sigma)$, comprehensive ablation studies on several main hyperparameters (such as $\alpha$, $\beta$, $\eta$, and $\delta$), and the learning curves of \textsc{PreSa} for all 29 tasks.

\begin{table*}[t]
  \caption{All evaluation results of normalized reward and cost. The $\uparrow$ symbol indicates that the higher reward, the better, while the $\downarrow$ symbol signifies that the lower cost (up to a threshold of $1$), the better. 
  % Each value is averaged over 3 distinct cost thresholds, 20 evaluation episodes, and 3 random seeds. 
  \textbf{Bold}: Safe agents whose normalized cost is below 1. \textcolor{blue}{\textbf{Blue}}: Safe agent with the highest reward among offline safe RL baselines. \textcolor{orange}{\textbf{Orange}}: Safe agent with the highest reward among approaches learning from offline synthetic feedback. \textcolor{gray}{\textbf{Gray}}: Unsafe agent.}
  % \centering
  \resizebox{\textwidth}{!}{
  \begin{tabular}{ccccccc|cccccccccc}
    \toprule
    \multirow{2}[2]{*}{Task} & \multicolumn{6}{c}{Offline Safe RL (w/ ground truth rewards and costs)}  &  \multicolumn{8}{c}{Offline Safe POHF} \\
    \cmidrule(lr){2-7}
    \cmidrule(lr){8-15} 
    & \multicolumn{2}{c}{BC-All} & \multicolumn{2}{c}{BC-Safe} & \multicolumn{2}{c}{CDT} & \multicolumn{2}{c}{Binary Alignment} & \multicolumn{2}{c}{BC-Safe-Seg} & \multicolumn{2}{c}{Safe-RLHF (CDT)} & \multicolumn{2}{c}{\textsc{PreSa} (Ours)}  \\ %(unnormalized)} \\
    \cmidrule(lr){2-3}
    \cmidrule(lr){4-5}
    \cmidrule(lr){6-7}
    \cmidrule(lr){8-9}
    \cmidrule(lr){10-11}
    \cmidrule(lr){12-13}
    \cmidrule(lr){14-15}
      &   reward$\uparrow$    &   cost$\downarrow$ &   reward$\uparrow$    &   cost$\downarrow$  &   reward$\uparrow$    &   cost$\downarrow$ &   reward$\uparrow$    &   cost$\downarrow$  &   reward$\uparrow$    &   cost$\downarrow$    &   reward$\uparrow$    &   cost$\downarrow$    &   reward$\uparrow$    &   cost$\downarrow$    \\
    \midrule
    PointButton1    &   $\textcolor{gray}{0.1}$   &   $\textcolor{gray}{1.05}$   &   $\bm{0.06}$   &   $\bm{0.52}$   &   $\textcolor{gray}{0.53}$   &   $\textcolor{gray}{1.68}$   &   \textbf{0.02}   &   \textbf{0.55}    &    \textbf{0.06}    &   \textbf{0.81}    & \textbf{0.06}    &   \textbf{0.78}    &   $\textcolor{orange}{\bm{0.09}}$    &   $\textcolor{orange}{\bm{0.84}}$    \\
    PointButton2    &   $\textcolor{gray}{0.27}$    &   $\textcolor{gray}{2.02}$    &   $\textcolor{gray}{0.16}$    &   $\textcolor{gray}{1.1}$ &   $\textcolor{gray}{0.46}$    &   $\textcolor{gray}{1.57}$    &   \textcolor{orange}{\textbf{-0.03}}   &   \textcolor{orange}{\textbf{0.5}}    &    \textcolor{gray}{0.22}    &   \textcolor{gray}{1.57}    & \textcolor{gray}{0.18}    &   \textcolor{gray}{1.33}    &   $\bm{-0.1}$   &   $\bm{0.74}$    \\
    PointCircle1    &   $\textcolor{gray}{0.79}$    &   $\textcolor{gray}{3.98}$    &   $\bm{0.41}$    &   $\bm{0.16}$    &   $\textcolor{blue}{\bm{0.59}}$    &   $\textcolor{blue}{\bm{0.69}}$    &   \textcolor{gray}{-0.23}   &   \textcolor{gray}{1.21}    &   \textcolor{gray}{0.32}    &   \textcolor{gray}{1.09}    &   \textcolor{gray}{0.37}    &   \textcolor{gray}{2.97}    &   \textcolor{orange}{\textbf{0.4}} &   \textcolor{orange}{\textbf{0.21}} \\
    PointCircle2    &   $\textcolor{gray}{0.66}$    &   $\textcolor{gray}{4.17}$    &   $\bm{0.48}$    &   $\bm{0.99}$    &   $\textcolor{gray}{0.64}$    &   $\textcolor{gray}{1.05}$    &   \textcolor{gray}{-0.24}   &   \textcolor{gray}{8.3}    &   \textcolor{gray}{0.44}    &   \textcolor{gray}{1.89}    &  \textcolor{gray}{0.66}    &   \textcolor{gray}{4.87}    &   \textcolor{orange}{\textbf{0.16}} &   \textcolor{orange}{\textbf{0.96}}    \\
    PointGoal1  &   $\bm{0.65}$    &   $\bm{0.95}$    &   $\bm{0.43}$    &   $\bm{0.54}$    &   $\textcolor{gray}{0.69}$    &   $\textcolor{gray}{1.12}$    &   \textbf{0.31}    &   \textbf{0.42}    &   \textcolor{gray}{0.48}    &   \textcolor{gray}{1.17}    & \textbf{0.34}    &   \textbf{0.52}    &   \textcolor{orange}{\textbf{0.37}}    &   \textcolor{orange}{\textbf{0.73}}  \\
    PointGoal2  &   $\textcolor{gray}{0.54}$    &   $\textcolor{gray}{1.97}$    &   $\bm{0.29}$    &   $\bm{0.78}$    &   $\textcolor{gray}{0.59}$    &   $\textcolor{gray}{1.34}$    &   \textcolor{gray}{0.39}   &   \textcolor{gray}{1.15}    &   \textcolor{gray}{0.52}  &   \textcolor{gray}{2.08}    &    \textcolor{gray}{0.35}    &   \textcolor{gray}{2.5} &   \textcolor{orange}{\textbf{0.16}}    &   \textcolor{orange}{\textbf{0.96}}    \\
    PointPush1  &   $\bm{0.19}$    &   $\bm{0.61}$    &   $\bm{0.13}$    &   $\bm{0.43}$    &   $\bm{0.24}$    &   $\bm{0.48}$    &   \textbf{0.14}   &   \textbf{0.51}    &   \textcolor{orange}{\textbf{0.19}} &   \textcolor{orange}{\textbf{0.6}} &   \textbf{0.1} &   \textbf{0.36}    &   \textbf{0.14}    &   \textbf{0.4}  \\
    PointPush2  &   $\bm{0.18}$    &   $\bm{0.91}$    &   $\bm{0.11}$    &   $\bm{0.8}$ &   $\bm{0.21}$    &   $\bm{0.65}$    &   \textcolor{gray}{0.17}   &   \textcolor{gray}{1.69}    &   \textcolor{orange}{\textbf{0.18}}    &   \textcolor{orange}{\textbf{0.8}}    &   \textbf{0.08}    &   \textbf{0.22}    &   \textbf{0.12}    &   \textbf{0.9}  \\
    CarButton1  &   $\textcolor{gray}{0.03}$    &   $\textcolor{gray}{1.38}$    &   $\textcolor{blue}{\bm{0.07}}$    &   $\textcolor{blue}{\bm{0.85}}$    &   $\textcolor{gray}{0.21}$    &   $\textcolor{gray}{1.6}$ &   \textcolor{gray}{-0.01}   &   \textcolor{gray}{2.52}    &   \textcolor{gray}{0.02}    &   \textcolor{gray}{1.42}    &   \textcolor{gray}{0.05}    &   \textcolor{gray}{3.96}    &   \textcolor{gray}{0.12}    &   \textcolor{gray}{1.87}   \\
    CarButton2  &   $\textcolor{gray}{-0.13}$   &   $\textcolor{gray}{1.24}$    &   $\textcolor{blue}{\bm{-0.01}}$   &   $\textcolor{blue}{\bm{0.63}}$    &   $\textcolor{gray}{0.13}$    &   $\textcolor{gray}{1.58}$    &   \textcolor{gray}{-0.06}   &   \textcolor{gray}{1.36}    &   \textcolor{gray}{-0.03}   &   \textcolor{gray}{1.01}    & \textcolor{gray}{0.02}  & \textcolor{gray}{1.77}    &   \textcolor{gray}{-0.04}   &   \textcolor{gray}{1.27}    \\
    CarCircle1  &   $\textcolor{gray}{0.72}$    &   $\textcolor{gray}{4.39}$    &   $\textcolor{gray}{0.37}$    &   $\textcolor{gray}{1.38}$    &   $\textcolor{gray}{0.6}$ &   $\textcolor{gray}{1.73}$    &   \textcolor{gray}{-0.32}   &   \textcolor{gray}{4.71}    &   \textcolor{gray}{0.61}    &   \textcolor{gray}{4.53}    &   \textcolor{gray}{0.27}    &   \textcolor{gray}{3.53}    &   \textcolor{gray}{-0.26}    &   \textcolor{gray}{2.86}\\
    CarCircle2  &   $\textcolor{gray}{0.76}$    &   $\textcolor{gray}{6.44}$    &   $\textcolor{gray}{0.54}$    &   $\textcolor{gray}{3.38}$    &   $\textcolor{gray}{0.66}$    &   $\textcolor{gray}{2.53}$    &   \textbf{-0.23}   &   \textbf{0.0}   &   \textcolor{gray}{0.63}    &   \textcolor{gray}{4.23}    &   \textcolor{gray}{0.5} &   \textcolor{gray}{3.91}    &   \textcolor{orange}{\textbf{0.23}}    &   \textcolor{orange}{\textbf{0.22}}  \\
    CarGoal1    &   $\bm{0.39}$    &   $\bm{0.33}$    &   $\bm{0.24}$    &   $\bm{0.28}$    &   $\textcolor{gray}{0.66}$    &   $\textcolor{gray}{1.21}$    &   \textbf{0.29}    &   \textbf{0.38}    &   \textbf{0.25}    &   \textbf{0.3} &   \textcolor{orange}{\textbf{0.4}}    &   \textcolor{orange}{\textbf{0.61}}    &   \textbf{0.26}    &   \textbf{0.14}    \\
    CarGoal2    &   $\textcolor{gray}{0.23}$    &   $\textcolor{gray}{1.05}$    &   $\bm{0.14}$    &   $\bm{0.51}$    &   $\textcolor{gray}{0.48}$    &   $\textcolor{gray}{1.25}$    &   \textcolor{orange}{\textbf{0.18}}   &   \textcolor{orange}{\textbf{0.64}}    &   \textcolor{gray}{0.17} &   \textcolor{gray}{1.03}    &   \textcolor{gray}{0.18}    &   \textcolor{gray}{1.01}    &   \textbf{0.14}   &   \textbf{0.35}  \\
    CarPush1    &   $\bm{0.22}$    &   $\bm{0.36}$    &   $\bm{0.14}$    &   $\bm{0.33}$    &   $\textcolor{blue}{\bm{0.31}}$    &   $\textcolor{blue}{\bm{0.4}}$ &   \textbf{0.16}   &   \textbf{0.34} &   \textcolor{orange}{\textbf{0.21}}    &   \textcolor{orange}{\textbf{0.51}}    & \textbf{0.17}    &   \textbf{0.96}    &   \textbf{0.15}    &   \textbf{0.56}  \\
    CarPush2    &   $\textcolor{blue}{\bm{0.14}}$    &   $\textcolor{blue}{\bm{0.9}}$ &   $\bm{0.05}$    &   $\bm{0.45}$    &   $\textcolor{gray}{0.19}$    &   $\textcolor{gray}{1.3}$ &   \textbf{0.07}   &   \textbf{0.69} &   \textbf{0.07}    &   \textbf{0.91}    &   \textcolor{gray}{0.1} &   \textcolor{gray}{1.81}    &   \textcolor{orange}{\textbf{0.1}}    &   \textcolor{orange}{\textbf{0.52}}    \\
    SwimmerVelocity &   $\textcolor{gray}{0.49}$    &   $\textcolor{gray}{4.72}$    &   $\textcolor{gray}{0.51}$    &   $\textcolor{gray}{1.07}$    &   $\textcolor{blue}{\bm{0.66}}$    &   $\textcolor{blue}{\bm{0.96}}$    &   \textcolor{orange}{\textbf{-0.04}}   &   \textcolor{orange}{\textbf{0.7}}    &   \textcolor{gray}{0.33}  &   \textcolor{gray}{2.61}    &  \textcolor{gray}{0.66}    &   \textcolor{gray}{1.1} &   \textcolor{gray}{0.39}    &   \textcolor{gray}{1.96}    \\
    HopperVelocity  &  $\textcolor{gray}{0.65}$    &   $\textcolor{gray}{6.39}$    &   $\bm{0.36}$    &   $\bm{0.67}$    &   $\textcolor{blue}{\bm{0.63}}$    &   $\textcolor{blue}{\bm{0.61}}$    &   \textbf{-0.02}   &   \textbf{0.0} &   \textcolor{orange}{\textbf{0.64}}    &   \textcolor{orange}{\textbf{0.64}}    &   \textcolor{gray}{0.17}    &   \textcolor{gray}{1.27}    &   \textcolor{gray}{0.42}    &   \textcolor{gray}{5.89} \\
    HalfCheetahVelocity &   $\textcolor{gray}{0.97}$    &   $\textcolor{gray}{13.1}$    &   $\bm{0.88}$    &   $\bm{0.54}$    &   $\textcolor{blue}{\bm{1.0}}$ &   $\textcolor{blue}{\bm{0.01}}$    &   \textbf{0.05}   &   \textbf{0.0} &   \textbf{0.92}    &   \textbf{0.54}    &    \textcolor{orange}{\textbf{0.93}}    &   \textcolor{orange}{\textbf{0.37}}    &   \textcolor{gray}{0.71}    &   \textcolor{gray}{4.11}    \\
    Walker2dVelocity    &   $\textcolor{gray}{0.79}$    &   $\textcolor{gray}{3.88}$    &   $\textcolor{blue}{\bm{0.79}}$    &   $\textcolor{blue}{\bm{0.04}}$    &   $\bm{0.78}$    &   $\bm{0.06}$    &   \textbf{-0.01}   &   \textbf{0.0} &   \textbf{0.78}    &   \textbf{0.01}    &   \textcolor{gray}{0.11}    &   \textcolor{gray}{1.42}    &   \textcolor{orange}{\textbf{0.79}}    &   \textcolor{orange}{\textbf{0.0}}   \\
    AntVelocity &   $\textcolor{gray}{0.98}$    &   $\textcolor{gray}{3.72}$    &   $\textcolor{blue}{\bm{0.98}}$    &   $\textcolor{blue}{\bm{0.29}}$    &   $\bm{0.98}$    &   $\bm{0.39}$    &   \textbf{-0.06}   &   \textbf{0.0}    &   \textcolor{orange}{\textbf{0.96}}  &   \textcolor{orange}{\textbf{0.3}}  & \textbf{0.93}    &   \textbf{0.23}    &   \textcolor{orange}{\textbf{0.96}}    &   \textcolor{orange}{\textbf{0.27}}  \\
    \midrule
    \textbf{SafetyGym Average}   &   $\textcolor{gray}{0.46}$    &   $\textcolor{gray}{3.03}$    &   $\bm{0.34}$    &   $\bm{0.75}$    &   $\textcolor{gray}{0.54}$    &   $\textcolor{gray}{1.06}$    &   \textcolor{gray}{0.03}   &   \textcolor{gray}{1.22}    &   \textcolor{gray}{0.38}   &   \textcolor{gray}{1.34}    &  \textcolor{gray}{0.32}    &   \textcolor{gray}{1.7} &   \textcolor{gray}{0.25}    &   \textcolor{gray}{1.23}  \\
    % \midrule
    % \textbf{\# of Safe Agents (out of 21)}   &   \multicolumn{2}{c}{6}    &   \multicolumn{2}{c}{17}    &   \multicolumn{2}{c}{9}    &   \multicolumn{2}{c}{14}    &   \multicolumn{2}{c}{11}    &   &   &   \multicolumn{2}{c}{14}  \\
    \midrule
    BallRun &   $\textcolor{gray}{0.6}$ &   $\textcolor{gray}{5.08}$    &   $\textcolor{gray}{0.27}$    &   $\textcolor{gray}{1.46}$    &   $\textcolor{gray}{0.39}$    &   $\textcolor{gray}{1.16}$    &   \textcolor{gray}{0.31}   &   $\textcolor{gray}{4.79}$   &  $\textcolor{gray}{0.37}$    &   $\textcolor{gray}{1.13}$    &   $\textcolor{gray}{0.35}$    &   $\textcolor{gray}{1.65}$  &   $\textcolor{orange}{\bm{0.19}}$    &   $\textcolor{orange}{\bm{0.09}}$   \\
    CarRun  &   $\bm{0.97}$    &   $\bm{0.33}$    &   $\bm{0.94}$    &   $\bm{0.22}$    &   $\textcolor{blue}{\bm{0.99}}$    &   $\textcolor{blue}{\bm{0.65}}$    &   \textbf{0.94}   &    \textbf{0.0}   &   \textcolor{orange}{\textbf{0.97}}    &   \textcolor{orange}{\textbf{0.95}}    &   \textcolor{gray}{0.87}    &   \textcolor{gray}{1.16}  &   $\bm{0.96}$    &   $\bm{0.0}$   \\
    DroneRun    &   $\textcolor{gray}{0.24}$    &   $\textcolor{gray}{2.13}$    &   $\bm{0.28}$    &   $\bm{0.74}$    &   $\textcolor{blue}{\bm{0.63}}$    &   $\textcolor{blue}{\bm{0.79}}$    &   $\bm{0.11}$   &   $\bm{0.17}$   &   $\textcolor{gray}{0.17}$    &   $\textcolor{gray}{5.97}$    &   $\textcolor{gray}{0.47}$    &   $\textcolor{gray}{3.12}$    &   $\textcolor{orange}{\bm{0.16}}$    &   $\textcolor{orange}{\bm{0.33}}$    \\
    AntRun  &   $\textcolor{gray}{0.7}$   &   $\textcolor{gray}{2.93}$    &   $\textcolor{gray}{0.65}$    &   $\textcolor{gray}{1.09}$    &   $\textcolor{blue}{\bm{0.72}}$    &   $\textcolor{blue}{\bm{0.91}}$    &   \textbf{0.09}   &    \textbf{0.01}   &   \textcolor{gray}{0.66}    &   \textcolor{gray}{1.38}    &   \textcolor{gray}{0.72} &   \textcolor{gray}{1.04}    &   \textcolor{orange}{\textbf{0.61}}    &   \textcolor{orange}{\textbf{0.63}}  \\
    BallCircle  &   $\textcolor{gray}{0.74}$    &   $\textcolor{gray}{4.71}$    &   $\bm{0.52}$    &   $\bm{0.65}$    &   $\textcolor{gray}{0.77}$    &   $\textcolor{gray}{1.07}$    &   $\bm{0.06}$   &    $\bm{0.24}$   &  $\textcolor{orange}{\bm{0.39}}$    &   $\textcolor{orange}{\bm{0.68}}$    &   $\textcolor{gray}{0.68}$    &   $\textcolor{gray}{1.2}$    &   $\bm{0.22}$ &   $\bm{0.03}$    \\
    CarCircle   &   $\textcolor{gray}{0.58}$    &   $\textcolor{gray}{3.74}$    &   $\bm{0.5}$ &   $\bm{0.84}$    &   $\textcolor{blue}{\bm{0.75}}$    &   $\textcolor{blue}{\bm{0.95}}$    &   \textbf{0.06}   &   \textbf{0.35}   &  \textcolor{gray}{0.56}    &   \textcolor{gray}{1.76}    &   \textcolor{orange}{\textbf{0.57}}    &   \textcolor{orange}{\textbf{0.84}}    &   \textbf{0.08}    &   \textbf{0.91} \\
    DroneCircle &   $\textcolor{gray}{0.72}$    &   $\textcolor{gray}{3.03}$    &   $\bm{0.56}$    &   $\bm{0.57}$    &   $\textcolor{blue}{\bm{0.63}}$    &   $\textcolor{blue}{\bm{0.98}}$    &   $\textcolor{gray}{-0.23}$   &   $\textcolor{gray}{1.59}$   &   $\textcolor{gray}{0.61}$    &   $\textcolor{gray}{1.9}$  &   $\textcolor{orange}{\bm{0.61}}$    &   $\textcolor{orange}{\bm{0.87}}$  &   $\bm{0.54}$    &   $\bm{0.72}$    \\
    AntCircle   &   $\textcolor{gray}{0.58}$    &   $\textcolor{gray}{4.9}$ &   $\bm{0.4}$ &   $\bm{0.96}$    &   $\textcolor{gray}{0.54}$    &   $\textcolor{gray}{1.78}$    &   \textcolor{gray}{0.48}   &    \textcolor{gray}{3.03}   &    \textcolor{gray}{0.54}    &   \textcolor{gray}{3.15}    &   \textcolor{gray}{0.45}    &   \textcolor{gray}{2.04}    &   \textcolor{gray}{0.55}    &   \textcolor{gray}{3.78} \\
    \midrule
    \textbf{BulletGym Average}   &   $\textcolor{gray}{0.64}$    &   $\textcolor{gray}{3.36}$    &   $\bm{0.52}$    &   $\bm{0.82}$    &   $\textcolor{gray}{0.68}$    &   $\textcolor{gray}{1.04}$    &   \textcolor{gray}{0.23}   &   \textcolor{gray}{1.27}   &    \textcolor{gray}{0.53}    &   \textcolor{gray}{2.11}    &   \textcolor{gray}{0.59}    &   \textcolor{gray}{1.49}   &   \textcolor{orange}{\textbf{0.41}}    &   \textcolor{orange}{\textbf{0.81}}   \\
    % \midrule
    % \textbf{\# of Safe Agents (out of 8)}   &   \multicolumn{2}{c}{1}    &   \multicolumn{2}{c}{6}    &   \multicolumn{2}{c}{5}    &   \multicolumn{2}{c}{5}    &   \multicolumn{2}{c}{2}    &  \multicolumn{2}{c}{2}   &   \multicolumn{2}{c}{7}  \\
    \midrule
    easysparse  &   \textcolor{gray}{0.17}    &   \textcolor{gray}{1.54}    &   \textbf{0.11}    &   \textbf{0.21}    &   \textcolor{blue}{\textbf{0.17}}    &   \textcolor{blue}{\textbf{0.23}}    &   \textbf{-0.04}   &   \textbf{0.06}    &   \textbf{0.01}    &   \textbf{0.1} &   \textcolor{gray}{0.42}    &   \textcolor{gray}{1.5} &   \textcolor{orange}{\textbf{0.31}}    &   \textcolor{orange}{\textbf{0.67}}    \\
    easymean    &   \textcolor{gray}{0.43}    &   \textcolor{gray}{2.82}    &   \textbf{0.04}    &   \textbf{0.29}    &   \textcolor{blue}{\textbf{0.45}}    &   \textcolor{blue}{\textbf{0.54}}    &   \textbf{-0.04}   &   \textbf{0.06}    &   \textbf{0.13}    &   \textbf{0.14}    &   \textcolor{gray}{0.33}    &   \textcolor{gray}{1.37}    &   \textcolor{orange}{\textbf{0.37}}    &   \textcolor{orange}{\textbf{0.6}} \\
    easydense   &   \textcolor{gray}{0.27}    &   \textcolor{gray}{1.94}    &   \textbf{0.11}    &   \textbf{0.14}    &   \textcolor{blue}{\textbf{0.32}}    &   \textcolor{blue}{\textbf{0.62}}    &   \textbf{-0.01}   &   \textbf{0.11}    &   \textbf{0.1} &   \textbf{0.14}    &   \textcolor{gray}{0.43}    &   \textcolor{gray}{1.68}    &   \textcolor{orange}{\textbf{0.21}}    &   \textcolor{orange}{\textbf{0.37}}    \\
    mediumsparse    &   \textcolor{gray}{0.83}    &   \textcolor{gray}{3.34}    &   \textcolor{blue}{\textbf{0.33}}    &   \textcolor{blue}{\textbf{0.3}} &   \textcolor{gray}{0.87}    &   \textcolor{gray}{1.1} & \textbf{0.02}    &   \textbf{0.11}    &   \textbf{0.19}    &   \textbf{0.16}    &   \textcolor{gray}{0.45}    &   \textcolor{gray}{1.12}    &   \textcolor{orange}{\textbf{0.52}}    &   \textcolor{orange}{\textbf{0.54}}    \\
    mediummean  &   \textcolor{gray}{0.77}    &   \textcolor{gray}{2.53}    &   \textbf{0.31}    &   \textbf{0.21}    &   \textcolor{blue}{\textbf{0.45}}    &   \textcolor{blue}{\textbf{0.75}}    &   \textbf{0.03}    &   \textbf{0.09}    &   \textbf{0.09}    &   \textbf{0.17}    &   \textcolor{gray}{0.46}    &   \textcolor{gray}{1.0} &   \textcolor{orange}{\textbf{0.35}}    &   \textcolor{orange}{\textbf{0.43}}    \\
    mediumdense &   \textcolor{gray}{0.45}    &   \textcolor{gray}{1.47}    &   \textcolor{blue}{\textbf{0.24}}    &   \textcolor{blue}{\textbf{0.17}}    &   \textcolor{gray}{0.88}    &   $\textcolor{gray}{2.41}$    & \textbf{-0.0}    &   \textbf{0.07}    &   \textbf{0.16}    &   \textbf{0.18}    &   \textbf{0.45}    &   \textbf{0.94}    &   \textcolor{orange}{\textbf{0.56}}    &   \textcolor{orange}{\textbf{0.52}}    \\
    hardsparse  &   \textcolor{gray}{0.42}    &   \textcolor{gray}{1.8} &   \textcolor{gray}{0.17}    &   \textcolor{gray}{3.25}    &   \textcolor{blue}{\textbf{0.25}}    &   \textcolor{blue}{\textbf{0.41}}    & \textbf{-0.02}   &   \textbf{0.06}    &   \textbf{0.09}    &   \textbf{0.18}    &   \textcolor{gray}{0.27}    &   \textcolor{gray}{1.16}    &   \textcolor{orange}{\textbf{0.24}}    &   \textcolor{orange}{\textbf{0.51}}    \\
    hardmean    &   \textcolor{gray}{0.2} &   \textcolor{gray}{1.77}    &   \textbf{0.13}    &   \textbf{0.4} &   \textcolor{blue}{\textbf{0.33}}    &   \textcolor{blue}{\textbf{0.97}}    &   \textbf{-0.02}   &   \textbf{0.07}    &   \textbf{0.01}    &   \textbf{0.29}    &   \textcolor{gray}{0.29}    &   \textcolor{gray}{1.23}    &   \textcolor{orange}{\textbf{0.19}}    &   \textcolor{orange}{\textbf{0.58}}    \\
    harddense   &   \textcolor{gray}{0.2} &   \textcolor{gray}{1.33}    &   \textcolor{blue}{\textbf{0.15}}    &   \textcolor{blue}{\textbf{0.22}}    &   \textbf{0.08}    &   \textbf{0.21}    &   \textbf{0.05}    &   \textbf{0.16}    &   \textbf{0.08}    &   \textbf{0.19}    &   \textcolor{gray}{0.21}    &   \textcolor{gray}{1.25}    &   \textcolor{orange}{\textbf{0.12}}    &   \textcolor{orange}{\textbf{0.35}}    \\
    \midrule
    \textbf{MetaDrive Average}  &  \textcolor{gray}{0.42}    &   \textcolor{gray}{2.06}    &   \textbf{0.18}    &   \textbf{0.58}    &   \textcolor{blue}{\textbf{0.42}}    &   \textcolor{blue}{\textbf{0.8}} &    \textbf{-0.0}    &   \textbf{0.09}    &   \textbf{0.1} &   \textbf{0.17}    &   \textcolor{gray}{0.37}    &   \textcolor{gray}{1.25}    &   \textcolor{orange}{\textbf{0.32}}    &   \textcolor{orange}{\textbf{0.51}}    \\
    \bottomrule
  \end{tabular}
  }
  \label{tab:complete_result}
\end{table*}

\subsection{PreSa With Varying Segment Lengths and Dataset Sizes}

The comprehensive results of \textsc{PreSa} with varying segment lengths are presented in Table~\ref{tab:complete_varying_seglen}. Overall, \textsc{PreSa} demonstrates strong performance across different segment lengths. However, for certain specific tasks, such as DroneRun and BallCircle, the policies learned under different segment lengths exhibit varying performance.

Table~\ref{tab:complete_varying_datasize} presents the results of \textsc{PreSa} trained with different sizes of offline datasets. Unsurprisingly, when trained on only 100 pairs of trajectory segments, \textsc{PreSa} struggles to learn safe policies. However, as the size of the offline datasets increases, we observe that \textsc{PreSa} trained on a few hundred to thousands segment pairs can achieve performance comparable to that of models trained on 10,000 pairs. This suggests the effectiveness of \textsc{PreSa}.

\begin{table*}[t]
  \centering
  \caption{Performance of \textsc{PreSa} with varying segment lengths}
  % \resizebox{\textwidth}{!}{
  \begin{tabular}{ccccccccc}
    \toprule
    \multirow{2}[2]{*}{Segment Length} & \multicolumn{2}{c}{32} & \multicolumn{2}{c}{64} & \multicolumn{2}{c}{128} \\ %(unnormalized)} \\
    \cmidrule(lr){2-3}
    \cmidrule(lr){4-5}
    \cmidrule(lr){6-7}
      &   reward$\uparrow$    &   cost$\downarrow$ &   reward$\uparrow$    &   cost$\downarrow$  &   reward$\uparrow$    &   cost$\downarrow$    \\
    \midrule
    BallRun &   \textbf{0.26} &   \textbf{1.12}    &   \textbf{0.19}    &   \textbf{0.09}    &   -    &   -   \\
    CarRun  &   \textbf{0.96}    &   \textbf{0.0}    &   \textbf{0.96}    &   \textbf{0.0}    &   \textbf{0.96}    &   \textbf{0.0}   \\
    DroneRun    &   \textbf{0.17}    &   \textbf{0.18}    &   \textbf{0.16}    &   \textbf{0.33}    &   \textbf{-0.01}    &   \textbf{0.0} \\
    AntRun  &   \textbf{0.64}    &   \textbf{0.79}    &   \textbf{0.61}    &   \textbf{0.63}    &   0.69    &   1.26   \\
    BallCircle  &   \textbf{0.36}    &   \textbf{0.25}    &   \textbf{0.22}    &   \textbf{0.03}    &   \textbf{0.22}    &   \textbf{0.05}  \\
    CarCircle   &   0.11    &   1.29    &   \textbf{0.08} &   \textbf{0.91}    &   \textbf{0.07}    &   \textbf{0.79}   \\
    DroneCircle &   \textbf{0.58}    &   \textbf{0.99}    &   \textbf{0.54}    &   \textbf{0.72}    &   \textbf{0.51}    &   \textbf{0.32}  \\
    AntCircle   &   0.56    &   3.88    &   0.55 &   3.78    &   0.57    &   3.88    \\
    \midrule
    \textbf{BulletGym Average}   &   0.45    &   1.06    &   \textbf{0.41}    &   \textbf{0.81}    &   \textbf{0.43}    &   \textbf{0.9}    \\
    \bottomrule
  \end{tabular}
  % }
  \label{tab:complete_varying_seglen}
\end{table*}

\begin{table*}[t]
  \centering
  \caption{Performance of \textsc{PreSa} with varying offline dataset sizes}
  % \resizebox{\textwidth}{!}{
  \begin{tabular}{ccccccccccc}
    \toprule
    \multirow{2}[2]{*}{Offline Data} & \multicolumn{2}{c}{100 pairs} & \multicolumn{2}{c}{500 pairs} & \multicolumn{2}{c}{2000 pairs} & \multicolumn{2}{c}{5000 pairs} & \multicolumn{2}{c}{10000 pairs} \\ %(unnormalized)} \\
    \cmidrule(lr){2-3}
    \cmidrule(lr){4-5}
    \cmidrule(lr){6-7}
    \cmidrule(lr){8-9}
    \cmidrule(lr){10-11}
      &   reward$\uparrow$    &   cost$\downarrow$ &   reward$\uparrow$    &   cost$\downarrow$ &   reward$\uparrow$    &   cost$\downarrow$ &   reward$\uparrow$    &   cost$\downarrow$  &   reward$\uparrow$    &   cost$\downarrow$    \\
    \midrule
    BallRun &   0.27    &   1.67    &   \textbf{0.21}    &   \textbf{0.23}    &   \textbf{0.19}    &   \textbf{0.14}    &   \textbf{0.19} &   \textbf{0.14}    &   \textbf{0.19}    &   \textbf{0.09}   \\
    CarRun  &   \textbf{0.91}    &   \textbf{0.0}   &   \textbf{0.96}    &    \textbf{0.0}    &   \textbf{0.96}    &   \textbf{0.0}   &   \textbf{0.96}    &   \textbf{0.0}    &   \textbf{0.96}    &   \textbf{0.0}   \\
    DroneRun    &   0.38    &   5.11    &   0.2 &   2.6 &   \textbf{0.19}    &   \textbf{0.52}    &   0.17    &   1.24    &   \textbf{0.16}    &   \textbf{0.33}    \\
    AntRun  &   0.5 &   1.72    &   \textbf{0.61}    &   \textbf{0.95}    &   \textbf{0.64}    &   \textbf{0.88}    &   \textbf{0.59}    &   \textbf{0.61}    &   \textbf{0.61}    &   \textbf{0.63}   \\
    BallCircle  &   \textbf{0.26}    &   \textbf{0.77}    &   \textbf{0.24}    &   \textbf{0.07}    &   \textbf{0.21}    &   \textbf{0.06}    &   \textbf{0.2}    &   \textbf{0.02}    &   \textbf{0.22}    &   \textbf{0.03}  \\
    CarCircle   &   0.26    &   1.55    &   \textbf{0.17}    &   \textbf{0.85}    &   \textbf{0.1}    &   \textbf{0.56}    &   0.09    &   1.18    &   \textbf{0.08} &   \textbf{0.91}   \\
    DroneCircle &   0.28    &   1.61    &   \textbf{0.51}    &   \textbf{0.93}    &   \textbf{0.51}    &   \textbf{0.62}    &   \textbf{0.52}    &   \textbf{0.64}    &   \textbf{0.54}    &   \textbf{0.72}  \\
    AntCircle   &   0.5 &   3.67    &   0.52    &   3.09    &   0.56    &   3.91    &   0.57    &   3.86    &   0.55    &   3.78    \\
    \midrule
    \textbf{BulletGym Average}   &  0.42    &   2.01    &   0.43    &   1.09    &  \textbf{0.42}    &   \textbf{0.84}    &   \textbf{0.41}    &   \textbf{0.96}    &   \textbf{0.41}    &   \textbf{0.81}    \\
    \bottomrule
  \end{tabular}
  % }
  \label{tab:complete_varying_datasize}
\end{table*}

\subsection{PreSa with Noisy Feedback}
To evaluate the effectiveness of \textsc{PreSa} under noisy safety feedback, we simulate different levels of noisiness by introducing noise into the safety feedback. At each level, a subset of the feedback is randomly selected, and its True/False labels are flipped. The results, shown in Table~\ref{tab:imperfect_safety}, demonstrate that \textsc{PreSa} consistently outperforms the baselines across varying levels of noisy safety feedback, although its performance gradually declines as the noise level increases. 

Additionally, we conduct experiments with synthetic noisy feedback at varying levels for pairwise preferences. The results, presented in Table~\ref{tab:imperfect_pref}, demonstrate that \textsc{PreSa} consistently outperforms the baselines across different levels of noisy feedback, although as noise levels increase, its performance slightly declines for both noisy pairwise preferences and binary safety labels.

\begin{table*}
  \centering
  \caption{Evaluation results with noisy binary safety labels: The noise level represents the percentage of binary safety labels that are randomly selected for label flipping.}
  % \resizebox{\textwidth}{!}{
  \begin{tabular}{ccccccccccc}
    \toprule
    \multirow{2}[2]{*}{Task} & \multirow{2}[2]{*}{Noise Level} & \multicolumn{2}{c}{Binary Alignment} & \multicolumn{2}{c}{BC-Safe-Seg} & \multicolumn{2}{c}{Safe-RLHF (CDT)} & \multicolumn{2}{c}{PreSa (Ours)} \\
    \cmidrule(lr){3-4}
    \cmidrule(lr){5-6}
    \cmidrule(lr){7-8}
    \cmidrule(lr){9-10}
      & &   reward$\uparrow$    &   cost$\downarrow$ &   reward$\uparrow$    &   cost$\downarrow$  &   reward$\uparrow$    &   cost$\downarrow$ &   reward$\uparrow$    &   cost$\downarrow$     \\
    \midrule
    \multirow{4}{*}{BallRun}   & 0\% &   0.31   &   4.79   &  0.37    &   1.13    &   0.35    &   1.65   &   \textcolor{orange}{\textbf{0.19}}    &   \textcolor{orange}{\textbf{0.09}}   \\
       &    10\% &   \textbf{0.2}   &   \textbf{0.04}   &   0.8    &   3.98  &   0.46    &   3.15    &   \textcolor{orange}{\textbf{0.21}}    &   \textcolor{orange}{\textbf{0.37}}    \\
       &    20\% &   \textbf{0.2}   &   \textbf{0.03}   &   0.71    &   3.51  &   0.52    &   2.87    &   \textcolor{orange}{\textbf{0.24}}    &   \textcolor{orange}{\textbf{0.92}}    \\
       &    30\% &   \textcolor{orange}{\textbf{0.2}}   &   \textcolor{orange}{\textbf{0.13}}   &   0.64    &   3.34  &   0.39    &   2.43    &   0.29    &   1.49    \\
    \midrule
    \multirow{4}{*}{BallCircle}   & 0\% &   \textbf{0.06}   &    \textbf{0.24}   &  \textcolor{orange}{\textbf{0.39}}    &   \textcolor{orange}{\textbf{0.68}}    &   0.68    &   1.2    &   \textbf{0.22} &   \textbf{0.03}    \\
       &    10\% &   \textbf{0.14}   &   \textbf{0.03}   &   0.48    &   1.57  &   0.66    &   1.34    &   \textcolor{orange}{\textbf{0.22}}    &   \textcolor{orange}{\textbf{0.07}}    \\
       &    20\% &   \textbf{0.14}   &   \textbf{0.03}   &   0.6    &   2.07  &   0.69    &   1.96    &   \textcolor{orange}{\textbf{0.27}}    &   \textcolor{orange}{\textbf{0.26}}    \\
       &    30\% &   \textbf{0.16}   &   \textbf{0.07}   &   0.53    &   1.52  &   0.68    &   2.21    &   \textcolor{orange}{\textbf{0.32}}    &   \textcolor{orange}{\textbf{0.49}}    \\
    \midrule
    \multirow{4}{*}{DroneRun}   & 0\% &   \textbf{0.11}   &   \textbf{0.17}   &   0.17    &   5.97    &   0.47    &   3.12    &   \textcolor{orange}{\textbf{0.16}}    &   \textcolor{orange}{\textbf{0.33}}    \\
       &    10\% &   \textbf{0.24}   &   \textbf{0.32}   &   \textcolor{orange}{\textbf{0.4}}    &   \textcolor{orange}{\textbf{0.6}}  &   0.6    &   4.0    &   \textbf{0.15}    &   \textbf{0.37}    \\
       &    20\% &   \textcolor{orange}{\textbf{0.25}}   &   \textcolor{orange}{\textbf{0.31}}   &   \textbf{0.22}    &   \textbf{0.82}  &   0.36    &   2.07    &   \textbf{0.14}    &   \textbf{0.44}    \\
       &    30\% &   \textcolor{orange}{\textbf{0.26}}   &   \textcolor{orange}{\textbf{0.32}}   &   0.59    &   1.67  &   \textbf{0.26}    &   \textbf{0.76}    &   \textbf{0.14}    &   \textbf{0.26}    \\
    \midrule
    \multirow{4}{*}{DroneCircle}   & 0\% &   -0.23   &   1.59   &   0.61    &   1.9  &   \textcolor{orange}{\textbf{0.61}}    &   \textcolor{orange}{\textbf{0.87}}    &   \textbf{0.54}    &   \textbf{0.72}    \\
       &    10\% &   0.51   &   1.33   &   0.53    &   1.61  &   0.59    &   1.21    &   \textcolor{orange}{\textbf{0.57}}    &   \textcolor{orange}{\textbf{0.92}}    \\
       &    20\% &   0.53   &   1.37   &   0.71    &   2.79  &   0.58    &   1.44    &   0.59    &   1.2    \\
       &    30\% &   0.54   &   1.49   &   0.72    &   2.78  &   0.59    &   1.4    &   0.61    &   1.41    \\
    \bottomrule
  \end{tabular}
  % }
  \label{tab:imperfect_safety}
\end{table*}

\begin{table*}
  \centering
  \caption{Evaluation results with noisy pairwise preferences: The noise level represents the percentage of preferences that are randomly selected for preference flipping. \textbf{Bold}: Safe agents whose normalized cost is below 1. \textcolor{orange}{\textbf{Orange}}: Safe agent with the highest reward among approaches learning from offline synthetic feedback.}
  % \resizebox{\textwidth}{!}{
  \begin{tabular}{ccccccccccc}
    \toprule
    \multirow{2}[2]{*}{Task} & \multirow{2}[2]{*}{Noise Level} & \multicolumn{2}{c}{Binary Alignment} & \multicolumn{2}{c}{BC-Safe-Seg} & \multicolumn{2}{c}{Safe-RLHF (CDT)} & \multicolumn{2}{c}{PreSa (Ours)} \\
    \cmidrule(lr){3-4}
    \cmidrule(lr){5-6}
    \cmidrule(lr){7-8}
    \cmidrule(lr){9-10}
      & &   reward$\uparrow$    &   cost$\downarrow$ &   reward$\uparrow$    &   cost$\downarrow$  &   reward$\uparrow$    &   cost$\downarrow$ &   reward$\uparrow$    &   cost$\downarrow$     \\
    \midrule
    \multirow{4}{*}{BallRun}   & 0\% &   0.31   &   4.79   &  0.37    &   1.13    &   0.35    &   1.65   &   \textcolor{orange}{\textbf{0.19}}    &   \textcolor{orange}{\textbf{0.09}}   \\
       &    10\% &   0.35   &   4.9   &   0.37    &   1.13  &   0.33    &   1.9    &   \textcolor{orange}{\textbf{0.2}}    &   \textcolor{orange}{\textbf{0.1}}    \\
       &    20\% &   0.62   &   5.11   &   0.37    &   1.13  &   0.4    &   1.72    &   \textcolor{orange}{\textbf{0.2}}    &   \textcolor{orange}{\textbf{0.08}}    \\
       &    30\% &   0.63   &   5.28   &   0.37    &   1.13  &   0.38    &   2.09    &   \textcolor{orange}{\textbf{0.2}}    &   \textcolor{orange}{\textbf{0.06}}    \\
    \midrule
    \multirow{4}{*}{BallCircle}   & 0\% &   \textbf{0.06}   &    \textbf{0.24}   &  \textcolor{orange}{\textbf{0.39}}    &   \textcolor{orange}{\textbf{0.68}}    &   0.68    &   1.2    &   \textbf{0.22} &   \textbf{0.03}    \\
       &    10\% &   0.11   &   3.21   &   \textcolor{orange}{\textbf{0.39}}    &   \textcolor{orange}{\textbf{0.68}}  &   0.69    &   1.36    &   \textbf{0.21}    &   \textbf{0.06}    \\
       &    20\% &   0.11   &   2.22   &   \textcolor{orange}{\textbf{0.39}}    &   \textcolor{orange}{\textbf{0.68}}  &   0.69    &   1.28    &   \textbf{0.17}    &   \textbf{0.12}    \\
       &    30\% &   \textbf{0.1}   &   \textbf{0.73}   &   \textcolor{orange}{\textbf{0.39}}    &   \textcolor{orange}{\textbf{0.68}}  &   0.69    &   1.23    &   \textbf{0.18}    &   \textbf{0.3}    \\
    \midrule
    \multirow{4}{*}{DroneRun}   & 0\% &   \textbf{0.11}   &   \textbf{0.17}   &   0.17    &   5.97    &   0.47    &   3.12    &   \textcolor{orange}{\textbf{0.16}}    &   \textcolor{orange}{\textbf{0.33}}    \\
       &    10\% &   0.08   &   1.52   &   0.17    &   5.97  &   0.42    &   3.29    &   \textcolor{orange}{\textbf{0.13}}    &   \textcolor{orange}{\textbf{0.44}}    \\
       &    20\% &   \textbf{0.05}   &   \textbf{0.98}   &   0.17    &   5.97  &   0.42    &   2.61    &   \textcolor{orange}{\textbf{0.15}}    &   \textcolor{orange}{\textbf{0.42}}    \\
       &    30\% &   0.12   &   1.23   &   0.17    &   5.97  &   0.46    &   3.48    &   \textcolor{orange}{\textbf{0.09}}    &   \textcolor{orange}{\textbf{0.26}}    \\
    \midrule
    \multirow{4}{*}{DroneCircle}   & 0\% &   -0.23   &   1.59   &   0.61    &   1.9  &   \textcolor{orange}{\textbf{0.61}}    &   \textcolor{orange}{\textbf{0.87}}    &   \textbf{0.54}    &   \textbf{0.72}    \\
       &    10\% &   -0.25   &   1.91   &   0.61    &   1.9  &   0.58    &   1.07    &   \textcolor{orange}{\textbf{0.35}}    &   \textcolor{orange}{\textbf{0.68}}    \\
       &    20\% &   -0.26   &   1.06   &   0.61    &   1.9  &   0.61    &   1.19    &   \textcolor{orange}{\textbf{0.23}}    &   \textcolor{orange}{\textbf{0.65}}    \\
       &    30\% &   -0.25   &   1.26   &   0.61    &   1.9  &   0.57    &   1.06    &   \textcolor{orange}{\textbf{0.18}}    &   \textcolor{orange}{\textbf{0.4}}    \\
    \bottomrule
  \end{tabular}
  % }
  \label{tab:imperfect_pref}
\end{table*}

\subsection{Threshold and Proportion of Safe/Unsafe Trajectories}
Following the experimental design of DSRL, we use three cost thresholds to generate data (e.g., [10, 20, 40] for BulletGym and [20, 40, 80] for SafetyGym). Safety labels of trajectory segments are assigned based on ground truth costs, with thresholds adjusted proportionally to segment length, ensuring segment-level compliance aligns with trajectory-level constraints. Notably, different cost thresholds yield varying proportions of safe and unsafe trajectories. We present statistics for four tasks in Table~\ref{tab:prop_safe_unsafe}, showing imbalance in the dataset. In Table~\ref{tab:varying_thres}, we demonstrate that \textsc{PreSa} performs effectively and robustly, with $w(y_\sigma)$ in Equation~\ref{tab:ablation_w} specifically designed to mitigate the impact of this imbalance. Additionally, an ablation study on $w(y_\sigma)$ in Table 4 reveals that its removal significantly degrades performance, making it more susceptible to the imbalanced training data.

\begin{table*}
  \centering
  \caption{Proportion of safe and unsafe trajectory segments in the dataset. The number of safe trajectory segments is shown in \textcolor{teal}{teal}, while the number of unsafe trajectory segments is shown in \textcolor{red}{red}.}
  % \small
  % \resizebox{\columnwidth}{!}{
  \begin{tabular}{cccc}
    \toprule
    Task   &   cost threshold $= 10$    &   cost threshold $= 20$    &   cost threshold $= 40$    \\
    \midrule
    BallRun   &   \textcolor{teal}{4,045}/\textcolor{red}{15,955}    &   \textcolor{teal}{5,194}/\textcolor{red}{14,806}    &   \textcolor{teal}{8,010}/\textcolor{red}{11,990}    \\
    BallCircle   &   \textcolor{teal}{3,374}/\textcolor{red}{16,626}    &   \textcolor{teal}{4,424}/\textcolor{red}{15,576}    &   \textcolor{teal}{7,934}/\textcolor{red}{12,066}    \\
    DroneRun  &  \textcolor{teal}{11,551}/\textcolor{red}{8,449}    &   \textcolor{teal}{12,038}/\textcolor{red}{7,962}    &   \textcolor{teal}{13,054}/\textcolor{red}{6,946}      \\
    DroneCircle   &   \textcolor{teal}{4,625}/\textcolor{red}{15,375}    &   \textcolor{teal}{4,988}/\textcolor{red}{15,012}    &   \textcolor{teal}{5,739}/\textcolor{red}{14,261}    \\
    \bottomrule
  \end{tabular}
  % }
  \label{tab:prop_safe_unsafe}
\end{table*}

\begin{table*}
  \centering
  \caption{The performance of \textsc{PreSa} with varying cost threshold.}
  % \small
  % \resizebox{\columnwidth}{!}{
  \begin{tabular}{ccccccccc}
    \toprule
    \multirow{2}[2]{*}{Task} & \multicolumn{2}{c}{cost threshold $= 10$} & \multicolumn{2}{c}{cost threshold $= 20$} & \multicolumn{2}{c}{cost threshold $= 40$} & \multicolumn{2}{c}{\textbf{Average}} \\
    \cmidrule(lr){2-3}
    \cmidrule(lr){4-5}
    \cmidrule(lr){6-7}
    \cmidrule(lr){8-9}
      &   reward$\uparrow$    &   cost$\downarrow$ &   reward$\uparrow$    &   cost$\downarrow$ &   reward$\uparrow$    &   cost$\downarrow$    &   reward$\uparrow$    &   cost$\downarrow$     \\
    \midrule
    BallRun   &   \textbf{0.18}    &   \textbf{0.04}    &   \textbf{0.18}    &   \textbf{0.0}   &   \textbf{0.21}    &   \textbf{0.22}  &   \textbf{0.19}    &   \textbf{0.09}    \\
    BallCircle   &   \textbf{0.14}    &   \textbf{0.0}    &   \textbf{0.24}    &   \textbf{0.02}    &   \textbf{0.29}    &   \textbf{0.08}    &   \textbf{0.22}    &   \textbf{0.03}    \\
    DroneRun  &  \textbf{0.15}    &   \textbf{0.45}    &   \textbf{0.17}    &   \textbf{0.25}   &   \textbf{0.16}    &   \textbf{0.29}    &   \textbf{0.16}    &   \textbf{0.33}      \\
    DroneCircle   &   0.53    &   1.34    &   \textbf{0.54}    &   \textbf{0.55}  &   \textbf{0.54}    &   \textbf{0.29}    &   \textbf{0.54}    &   \textbf{0.72}    \\
    \midrule
    \textbf{Average}  &  \textbf{0.25}    &   \textbf{0.46}    &   \textbf{0.28}    &   \textbf{0.21}  &   \textbf{0.3}    &   \textbf{0.22}   &   \textbf{0.28}    &   \textbf{0.29}      \\
    \bottomrule
  \end{tabular}
  % }
  \label{tab:varying_thres}
\end{table*}

\begin{table*}
  \centering
  \caption{The performance of \textsc{PreSa} w/o $w(y_\sigma)$ with varying cost threshold.}
  % \small
  % \resizebox{\columnwidth}{!}{
  \begin{tabular}{ccccccccc}
    \toprule
    \multirow{2}[2]{*}{Task} & \multicolumn{2}{c}{cost threshold $= 10$} & \multicolumn{2}{c}{cost threshold $= 20$} & \multicolumn{2}{c}{cost threshold $= 40$} & \multicolumn{2}{c}{\textbf{Average}} \\
    \cmidrule(lr){2-3}
    \cmidrule(lr){4-5}
    \cmidrule(lr){6-7}
    \cmidrule(lr){8-9}
      &   reward$\uparrow$    &   cost$\downarrow$ &   reward$\uparrow$    &   cost$\downarrow$ &   reward$\uparrow$    &   cost$\downarrow$    &   reward$\uparrow$    &   cost$\downarrow$     \\
    \midrule
    BallRun   &   -0.48    &   9.4    &   0.04    &   4.6   &   \textbf{0.12}    &   \textbf{0.0}   &   -0.1    &   4.67    \\
    BallCircle   &   0.0    &   12.8    &   -0.01    &   2.76    &   \textbf{0.06}    &   \textbf{0.9}   &   0.02    &   5.49    \\
    DroneRun  &  0.52    &   6.72    &   0.24    &   3.73   &   0.38    &   2.4  &  0.38    &   4.28      \\
    DroneCircle   &   \textbf{-0.26}    &   \textbf{0.02}    &   \textbf{-0.19}    &   \textbf{0.89}  &   \textbf{-0.26}    &   \textbf{0.38}   &   \textbf{-0.24}    &   \textbf{0.43}    \\
    \midrule
    \textbf{Average}  &  -0.05    &   7.24    &   0.02    &   2.99  &   \textbf{0.07}    &   \textbf{0.92}  &  0.02    &   3.72      \\
    \bottomrule
  \end{tabular}
  % }
  \label{tab:ablation_w}
\end{table*}

\subsection{Additional Baseline}
We conducted experiments on an additional baseline, with the results presented in Table~\ref{tab:cost_learning}. 
\textit{Safe-RLHF (CDT) (Cost: binary label only)}: A variant of Safe-RLHF (CDT) in which the cost model is trained using binary labels only.
The results indicate that \textsc{PreSa} outperforms this additional baseline. This performance gap is attributed to the baselines incorporating less information during learning, resulting in weaker performance.

\begin{table*}
  \centering
  \caption{Evaluation results for additional baselines using cost models trained exclusively on binary labels for Safe-RLHF.}
  % \resizebox{\textwidth}{!}{
  \begin{tabular}{ccccccc}
    \toprule
    \multirow{2}[2]{*}{Task} & \multicolumn{2}{c}{Safe-RLHF (CDT)} & \multicolumn{2}{c}{Safe-RLHF (CDT) \emph{(Cost: binary label only)}} & \multicolumn{2}{c}{PreSa (Ours)} \\
    \cmidrule(lr){2-3}
    \cmidrule(lr){4-5}
    \cmidrule(lr){6-7}
      &   reward$\uparrow$    &   cost$\downarrow$ &   reward$\uparrow$    &   cost$\downarrow$ &   reward$\uparrow$    &   cost$\downarrow$     \\
    \midrule
    BallRun   &   0.35    &   1.65   &   0.31    &   1.52  &   \textcolor{orange}{\textbf{0.19}}    &   \textcolor{orange}{\textbf{0.09}}   \\
    DroneCircle   &   \textcolor{orange}{\textbf{0.61}}    &   \textcolor{orange}{\textbf{0.87}}    &   0.61    &   1.19    &   \textbf{0.54}    &   \textbf{0.72}    \\
    \bottomrule
  \end{tabular}
  % }
  \label{tab:cost_learning}
\end{table*}

\subsection{Evaluation With Filtered Preference Dataset}
To investigate whether rewards matter when safety constraints are violated, we conducted experiments with \textsc{PreSa} using filtered preference data. Specifically, we excluded pairs where both segments were unsafe and pairs where the preferred segment was unsafe. The results, presented in Table~\ref{tab:filtered}, indicate that when using filtered preference data, our approach experiences a significant drop in performance and fails to learn safe behaviors. This demonstrates that reward information is crucial, even for unsafe trajectories.

\begin{table*}
  \centering
  \caption{Evaluation results using a filtered preference dataset, excluding pairs where both segments are unsafe or where the preferred segment is unsafe.}
  % \small
  % \resizebox{\linewidth}{!}{
  \begin{tabular}{ccccc}
    \toprule
    \multirow{2}[2]{*}{Task} & \multicolumn{2}{c}{Preference w/ safety} & \multicolumn{2}{c}{PreSa} \\
    \cmidrule(lr){2-3}
    \cmidrule(lr){4-5}
      &   reward$\uparrow$    &   cost$\downarrow$ &   reward$\uparrow$    &   cost$\downarrow$     \\
    \midrule
    BallRun   &   -0.08    &   5.46    &   \textbf{0.19}    &   \textbf{0.09}    \\
    CarRun   &   \textbf{0.44}    &   \textbf{0.0}    &   \textbf{0.96}    &   \textbf{0.0}    \\
    DroneRun  &  \textbf{0.02}    &   \textbf{0.0}    &   \textbf{0.16}    &   \textbf{0.33}      \\
    AntRun   &   \textbf{0.35}    &   \textbf{0.46}    &   \textbf{0.61}    &   \textbf{0.63}    \\
    BallCircle   &   -0.0    &   5.03    &   \textbf{0.22}    &   \textbf{0.03}    \\
    CarCircle  &  -0.1    &   3.23    &   \textbf{0.08}    &   \textbf{0.91}      \\
    DroneCircle   &   \textbf{-0.26}    &   \textbf{0.03}    &   \textbf{0.54}    &   \textbf{0.72}    \\
    AntCircle  &  \textbf{0.0}    &   \textbf{0.0}    &   0.55    &   3.78      \\
    \midrule
    Average  &  0.05    &   1.77    &   \textbf{0.41}    &   \textbf{0.81}      \\
    \bottomrule
  \end{tabular}
  % }
  \label{tab:filtered}
\end{table*}

\begin{table*}[t]
  \centering
  \caption{Performance of \textsc{PreSa} with varying $\alpha$.}
  % \resizebox{\textwidth}{!}{
  \begin{tabular}{ccccccccccc}
    \toprule
    \multirow{2}[2]{*}{Hyperparameter $\alpha$} & \multicolumn{2}{c}{$\alpha=0.2$} & \multicolumn{2}{c}{$\alpha=0.4$} & \multicolumn{2}{c}{$\alpha=0.6$} & \multicolumn{2}{c}{$\alpha=0.8$} \\ %(unnormalized)} \\
    \cmidrule(lr){2-3}
    \cmidrule(lr){4-5}
    \cmidrule(lr){6-7}
    \cmidrule(lr){8-9}
      &   reward$\uparrow$    &   cost$\downarrow$ &   reward$\uparrow$    &   cost$\downarrow$  &   reward$\uparrow$    &   cost$\downarrow$   &   reward$\uparrow$    &   cost$\downarrow$    \\
    \midrule
    BallRun &   \textbf{0.19}    &   \textbf{0.09}    & \textbf{0.19}    &   \textbf{0.1} &   0.38    &   2.8 &   \textbf{0.19}    &   \textbf{0.12}   \\
    CarRun  &   \textbf{0.96}    &   \textbf{0.0}    &   \textbf{0.95}    &   \textbf{0.0}  &   \textbf{0.95}    &   \textbf{0.0}   &   \textbf{0.95}    &   \textbf{0.0}   \\
    DroneRun    &   \textbf{0.16}    &   \textbf{0.33}    &   0.2    &   2.49   &   0.41    &   2.67    &   0.2 &   2.47    \\
    AntRun  &   \textbf{0.61}    &   \textbf{0.63}    &   0.65    &   1.85  &   0.63    &   2.88    &   0.68    &   1.86   \\
    BallCircle  &   \textbf{0.22}    &   \textbf{0.03}    &   \textbf{0.15}    &   \textbf{0.03}    &   0.3 &   1.06    &   \textbf{0.17}    &   \textbf{0.04}  \\
    CarCircle   &   \textbf{0.08}   &   \textbf{0.91}    &   \textbf{0.2}    &   \textbf{0.87}  &   0.32    &   1.01    &   \textbf{0.23}    &   \textbf{0.99}   \\
    DroneCircle &   \textbf{0.54}    &   \textbf{0.72}    &   \textbf{0.5}    &   \textbf{0.92} &   0.27    &   1.9 &   0.54    &   1.09  \\
    AntCircle   &   0.55 &   3.78    &   0.57    &   3.85   &   0.57    &   4.65    &   0.58    &   4.25    \\
    \midrule
    \textbf{BulletGym Average}   &   \textbf{0.41}    &   \textbf{0.81}    &   0.43    &   1.26    &   0.48    &   2.12 &   0.44    &   1.35    \\
    \bottomrule
  \end{tabular}
  % }
  \label{tab:alpha}
\end{table*}

\begin{table*}[t]
  \centering
  \caption{Performance of \textsc{PreSa} with varying $\beta$.}
  % \resizebox{\textwidth}{!}{
  \begin{tabular}{ccccccccccc}
    \toprule
    \multirow{2}[2]{*}{Hyperparameter $\beta$} & \multicolumn{2}{c}{$\beta=0.25$} & \multicolumn{2}{c}{$\beta=0.5$} & \multicolumn{2}{c}{$\beta=0.75$} & \multicolumn{2}{c}{$\beta=1.0$} \\ %(unnormalized)} \\
    \cmidrule(lr){2-3}
    \cmidrule(lr){4-5}
    \cmidrule(lr){6-7}
    \cmidrule(lr){8-9}
      &   reward$\uparrow$    &   cost$\downarrow$ &   reward$\uparrow$    &   cost$\downarrow$  &   reward$\uparrow$    &   cost$\downarrow$   &   reward$\uparrow$    &   cost$\downarrow$    \\
    \midrule
    BallRun &   \textbf{0.18}    &   \textbf{0.11}    & \textbf{0.18}    &   \textbf{0.09}  &   0.34    &   2.67   &   \textbf{0.19}    &   \textbf{0.09}   \\
    CarRun  &   \textbf{0.95}    &   \textbf{0.0}    &   \textbf{0.95}    &   \textbf{0.0}  &   \textbf{0.95}    &   \textbf{0.0}   &   \textbf{0.96}    &   \textbf{0.0}   \\
    DroneRun    &   0.22    &   2.61    &   0.22    &   2.49    &   0.26    &   2.13    &   \textbf{0.16}   &   \textbf{0.33}    \\
    AntRun  &   \textbf{0.6}    &   \textbf{0.88}    &   0.64    &   1.37  &   0.53    &   2.14    &   \textbf{0.61}    &   \textbf{0.63}   \\
    BallCircle  &   \textbf{0.16}    &   \textbf{0.03}    &   \textbf{0.16}    &   \textbf{0.03}    &   0.3 &   1.09    &   \textbf{0.22}    &   \textbf{0.03}  \\
    CarCircle   &   \textbf{0.22}   &   \textbf{0.83}    &   \textbf{0.2}    &   \textbf{0.79}  &   0.27    &   1.24    &   \textbf{0.08}    &   \textbf{0.91}   \\
    DroneCircle &   0.52    &   1.04    &   \textbf{0.49}    &   \textbf{0.98} &   0.32    &   1.45   &   \textbf{0.54}    &   \textbf{0.72}  \\
    AntCircle   &   0.58 &   4.01    &   0.54    &   4.07   &   0.51    &   4.36    &   0.55    &   3.78    \\
    \midrule
    \textbf{BulletGym Average}   &   0.43    &   1.19    &   0.43    &   1.23    &   0.44    &   1.88 &   \textbf{0.41}    &   \textbf{0.81}    \\
    \bottomrule
  \end{tabular}
  % }
  \label{tab:beta}
\end{table*}

\begin{table*}[t]
  \centering
  \caption{Performance of \textsc{PreSa} with varying $\eta$.}
  % \resizebox{\textwidth}{!}{
  \begin{tabular}{ccccccccccc}
    \toprule
    \multirow{2}[2]{*}{Hyperparameter $\eta$} & \multicolumn{2}{c}{$\eta=0.1$} & \multicolumn{2}{c}{$\eta=0.5$} & \multicolumn{2}{c}{$\eta=1.0$} & \multicolumn{2}{c}{$\eta=5.0$} \\ %(unnormalized)} \\
    \cmidrule(lr){2-3}
    \cmidrule(lr){4-5}
    \cmidrule(lr){6-7}
    \cmidrule(lr){8-9}
      &   reward$\uparrow$    &   cost$\downarrow$ &   reward$\uparrow$    &   cost$\downarrow$  &   reward$\uparrow$    &   cost$\downarrow$   &   reward$\uparrow$    &   cost$\downarrow$    \\
    \midrule
    BallRun &   \textbf{0.19}    &   \textbf{0.09}    & 0.66    &   4.4   &   -0.53    &   4.68   &   -0.17    &   5.46   \\
    CarRun  &   \textbf{0.96}    &   \textbf{0.0}    &   \textbf{0.94}    &   \textbf{0.0}  &   \textbf{0.91}    &   \textbf{0.0}   &   \textbf{0.6}    &   \textbf{0.01}   \\
    DroneRun    &   \textbf{0.16}    &   \textbf{0.33}    &   0.35    &   3.41  &   0.32    &   3.14    &   0.21 &   1.85    \\
    AntRun  &   \textbf{0.61}    &   \textbf{0.63}    &   \textbf{0.6}    &   \textbf{0.7}  &   \textbf{0.57}    &   \textbf{0.71}    &   \textbf{0.5}    &   \textbf{0.65}   \\
    BallCircle  &   \textbf{0.22}    &   \textbf{0.03}    &   \textbf{0.13}    &   \textbf{0.14}    &   0.05    &   3.14    &   0.09    &   5.45  \\
    CarCircle   &   \textbf{0.08}   &   \textbf{0.91}    &   \textbf{0.11}    &   \textbf{0.46}  &   \textbf{0.07}    &   \textbf{0.42}    &   0.01    &   9.2   \\
    DroneCircle &   \textbf{0.54}    &   \textbf{0.72}    &   \textbf{0.38}    &   \textbf{0.6} &   0.25    &   1.12   &   \textbf{-0.26}    &   \textbf{0.04}  \\
    AntCircle   &   0.55 &   3.78    &   0.51    &   3.51   &   0.48    &   6.17    &   \textbf{0.0}    &   \textbf{0.0}    \\
    \midrule
    \textbf{BulletGym Average}   &   \textbf{0.41}    &   \textbf{0.81}    &   0.46    &   1.65    &   0.26    &   2.42 &   0.12    &   2.83    \\
    \bottomrule
  \end{tabular}
  % }
  \label{tab:eta}
\end{table*}

\begin{table*}[t]
  \centering
  \caption{Performance of \textsc{PreSa} with varying $\delta$.}
  % \resizebox{\textwidth}{!}{
  \begin{tabular}{ccccccccccc}
    \toprule
    \multirow{2}[2]{*}{Hyperparameter $\delta$} & \multicolumn{2}{c}{$\delta=0.55$} & \multicolumn{2}{c}{$\delta=0.65$} & \multicolumn{2}{c}{$\delta=0.75$} & \multicolumn{2}{c}{$\delta=0.85$} & \multicolumn{2}{c}{$\delta=0.95$} \\ %(unnormalized)} \\
    \cmidrule(lr){2-3}
    \cmidrule(lr){4-5}
    \cmidrule(lr){6-7}
    \cmidrule(lr){8-9}
    \cmidrule(lr){10-11}
      &   reward$\uparrow$    &   cost$\downarrow$ &   reward$\uparrow$    &   cost$\downarrow$ &   reward$\uparrow$    &   cost$\downarrow$ &   reward$\uparrow$    &   cost$\downarrow$  &   reward$\uparrow$    &   cost$\downarrow$    \\
    \midrule
    BallRun &   0.35    &   1.94    &   \textbf{0.22}    &   \textbf{0.43}    &   \textbf{0.2}    &   \textbf{0.2}    &   \textbf{0.2} &   \textbf{0.17}    &   \textbf{0.19}    &   \textbf{0.09}   \\
    CarRun  &   \textbf{0.95}    &   \textbf{0.0}    &   \textbf{0.95}    &   \textbf{0.0}    &   \textbf{0.95}    &   \textbf{0.0}    &   \textbf{0.95}    &   \textbf{0.0}    &   \textbf{0.96}    &   \textbf{0.0}  \\
    DroneRun    &   0.11    &   3.07    &   0.13 &   2.72 &   0.2    &   2.86    &   0.21    &   2.63    &   \textbf{0.16}    &   \textbf{0.33}    \\
    AntRun  &   0.73    &   3.87    &   0.72    &   3.63    &   0.72    &   3.3    &   0.7    &   3.15    &   \textbf{0.61}    &   \textbf{0.63}  \\
    BallCircle  &   \textbf{0.33}    &   \textbf{0.91}    &   \textbf{0.28}    &   \textbf{0.65}    &   \textbf{0.24}    &   \textbf{0.51}    &   \textbf{0.19}    &   \textbf{0.28}    &   \textbf{0.22}    &   \textbf{0.03}  \\
    CarCircle  &   0.37    &   2.7    &   0.26    &   1.18    &   0.24    &   1.01    &   \textbf{0.23}    &   \textbf{0.87}    &   \textbf{0.08}    &   \textbf{0.91}  \\
    DroneCircle &   0.6    &   2.21    &   0.59    &   1.83    &   0.56    &   1.43    &   0.54    &   1.13    &   \textbf{0.54}    &   \textbf{0.72}  \\
    AntCircle  &   0.65    &   5.62    &   0.63    &   5.56    &   0.64    &   5.71    &   0.63    &   5.16    &   0.55    &   3.78  \\
    \midrule
    \textbf{BulletGym Average}   &  0.51    &   2.54    &   0.47    &   2.0    &  0.47    &   1.88    &   0.46    &   1.68    &   \textbf{0.41}    &   \textbf{0.81}    \\
    \bottomrule
  \end{tabular}
  % }
  \label{tab:delta}
\end{table*}

\subsection{Ablation study}

We explore the influence of several key hyperparameters used in \textsc{PreSa}. The results are shown in Table~\ref{tab:alpha}, Table~\ref{tab:beta}, Table~\ref{tab:eta}, and Table~\ref{tab:delta}. These results indicate that while \textsc{PreSa} generally performs well across various hyperparameter settings, further fine-tuning can lead to improved performance.

\begin{table*}
  \centering
  \caption{Ablation study of $z_\text{ref}$.}
  % \small
  % \resizebox{\columnwidth}{!}{
  \begin{tabular}{ccccc}
    \toprule
    \multirow{2}[2]{*}{Task} & \multicolumn{2}{c}{w/o $z_\text{ref}$} & \multicolumn{2}{c}{PreSa} \\
    \cmidrule(lr){2-3}
    \cmidrule(lr){4-5}
      &   reward$\uparrow$    &   cost$\downarrow$ &   reward$\uparrow$    &   cost$\downarrow$     \\
    \midrule
    BallRun   &   0.29    &   1.59    &   \textbf{0.19}    &   \textbf{0.09}    \\
    CarRun   &   \textbf{0.95}    &   \textbf{0.0}    &   \textbf{0.96}    &   \textbf{0.0}    \\
    DroneRun  &  0.61    &   2.77    &   \textbf{0.16}    &   \textbf{0.33}      \\
    AntRun   &   0.67    &   2.68    &   \textbf{0.61}    &   \textbf{0.63}    \\
    BallCircle   &   \textbf{0.39}    &   \textbf{0.67}    &   \textbf{0.22}    &   \textbf{0.03}    \\
    CarCircle  &  \textbf{0.46}    &   \textbf{0.96}    &   \textbf{0.08}    &   \textbf{0.91}      \\
    DroneCircle   &   0.32    &   1.52    &   \textbf{0.54}    &   \textbf{0.72}    \\
    AntCircle  &  0.59    &   5.67    &   0.55    &   3.78      \\
    \midrule
    Average  &  0.53    &   1.93    &   \textbf{0.41}    &   \textbf{0.81}      \\
    \bottomrule
  \end{tabular}
  % }
  \label{tab:zref}
\end{table*}

We also investigate the effect of the reference point $z_\text{ref}$, which is used to determine the ``gain'' or ``loss'' of a trajectory segment based on the relative difference between the segment's score and the reference point. In our approach, $z_\text{ref}$ is an estimated average score for all segments. To demonstrate its effectiveness, we performed an ablation study on $z_\text{ref}$, as shown in Table~\ref{tab:zref}. The results indicate that without the reference point, \textsc{PreSa}'s performance drops significantly, making it difficult to find safe policies. This occurs because, without the reference point, the utility of a segment is based solely on its absolute score rather than a relative value, which destabilizes the learning process and degrades performance.

\subsection{Learning Curves}

We train \textsc{PreSa} with the parameters in Table~\ref{tab:hyperparameters} for all 29 tasks. The learning curves are shown in Figure~\ref{fig:bulletgym_curve}, Figure~\ref{fig:safetgym_curve1}, Figure~\ref{fig:safetgym_curve2}, and Figure~\ref{fig:safetgym_curve3}. In each figure, the dotted vertical line marks the point where $\pi_\text{ref}$ pretraining stops, while the dotted horizontal line indicates the cost threshold of 1.

\section{Limitations}

Although \textsc{PreSa} demonstrates promising performance, there are several limitations that deserve further discussion and investigation. First, our framework models human feedback using a regret-based approach grounded in the Bradley-Terry assumption. While this modeling choice is common and effective, it may not always be the most appropriate, especially when handling noisy human feedback in real-world scenarios. Relaxing this assumption is an important avenue for future research.

Second, although \textsc{PreSa} enables direct policy learning and bypasses explicit reward and cost modeling, it still involves solving a constrained optimization problem, which can be complex and nontrivial. Exploring methods to simplify this optimization process without sacrificing performance represents a promising direction for future work.

Finally, the effectiveness of Offline Safe POHF methods is highly dependent on the quality and quantity of available data. Model performance can degrade significantly when trajectory data or human feedback are sparse, noisy, or biased. Developing techniques to ensure robust learning under such conditions remains an open and important challenge.

% \section{Impact Statement}

% This paper demonstrates the feasibility of learning high-reward and safe policies from human pairwise preferences and binary safety labels, eliminating the need for explicit reward and cost in offline settings. By leveraging human feedback, our approach enables the deployment of RL agents in safety-critical applications where designing accurate reward and cost functions is challenging.

\begin{figure*}
    \centering
    \begin{subfigure}[t]{0.245\textwidth}
        \centering
        \includegraphics[width=\textwidth]{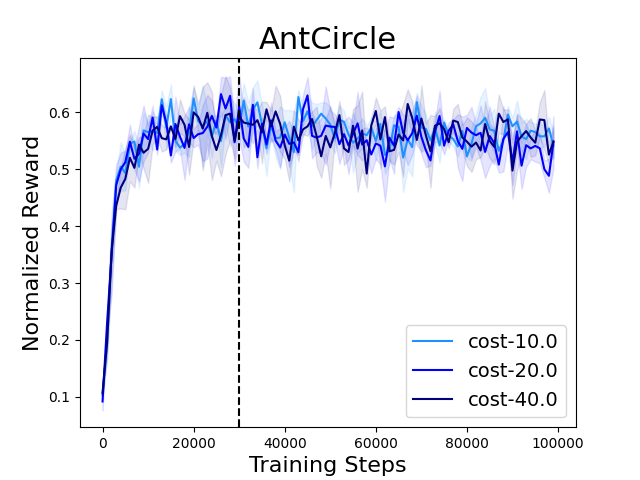}
    \end{subfigure}%
    \begin{subfigure}[t]{0.245\textwidth}
        \centering
        \includegraphics[width=\textwidth]{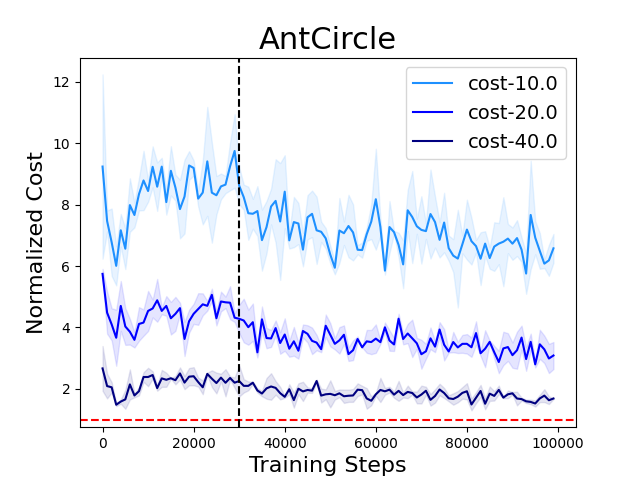}
    \end{subfigure}
    \smallskip
    \begin{subfigure}[t]{0.245\textwidth}
        \centering
        \includegraphics[width=\textwidth]{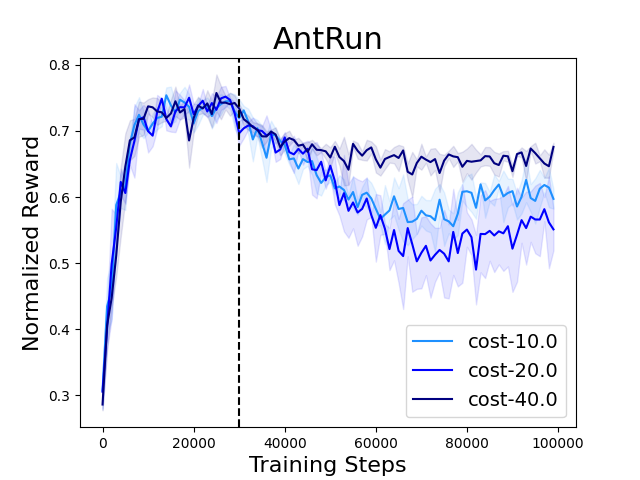}
    \end{subfigure}
    \begin{subfigure}[t]{0.245\textwidth}
        \centering
        \includegraphics[width=\textwidth]{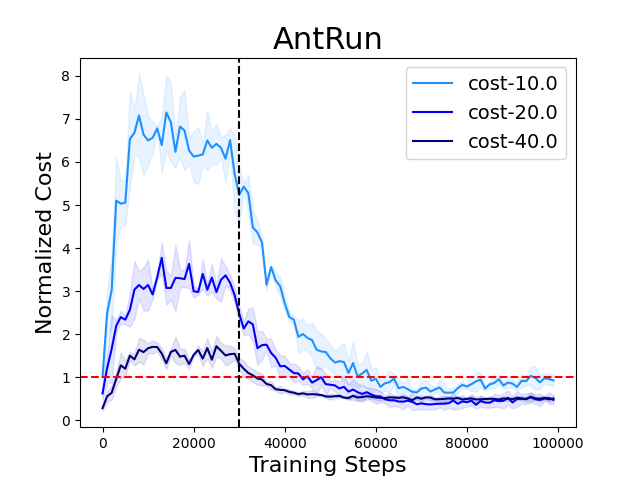}
    \end{subfigure}
    
    \begin{subfigure}[t]{0.245\textwidth}
        \centering
        \includegraphics[width=\textwidth]{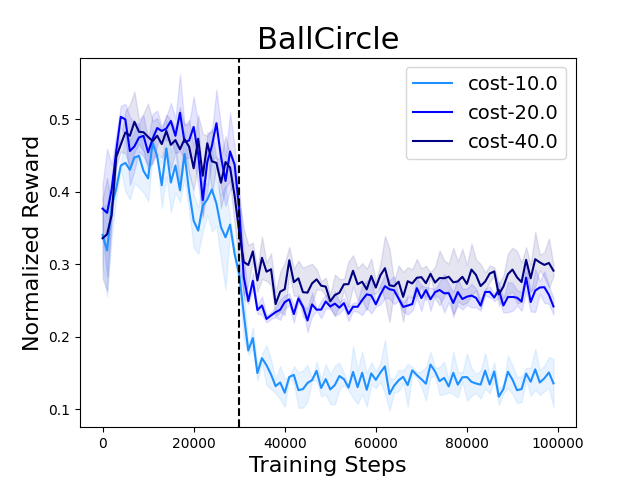}
    \end{subfigure}
    \begin{subfigure}[t]{0.245\textwidth}
        \centering
        \includegraphics[width=\textwidth]{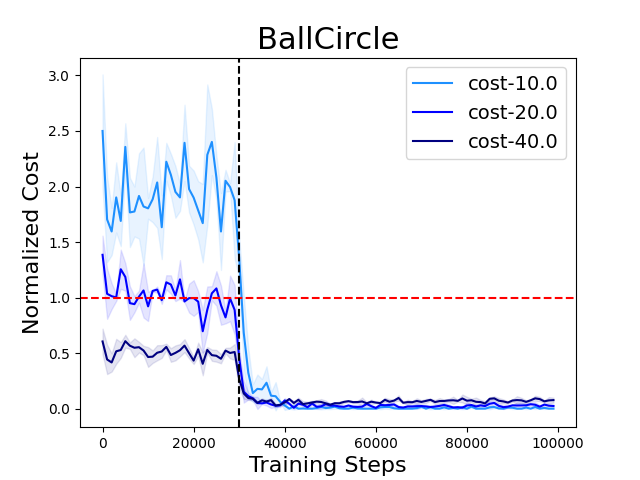}
    \end{subfigure}
    \smallskip
    \begin{subfigure}[t]{0.245\textwidth}
        \centering
        \includegraphics[width=\textwidth]{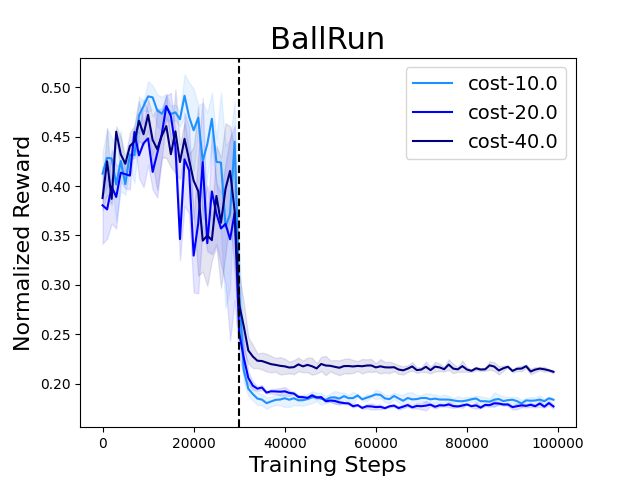}
    \end{subfigure}
    \begin{subfigure}[t]{0.245\textwidth}
        \centering
        \includegraphics[width=\textwidth]{figures/curves/bullet/BallCircle_cost_curve.png}
    \end{subfigure}

    \begin{subfigure}[t]{0.245\textwidth}
        \centering
        \includegraphics[width=\textwidth]{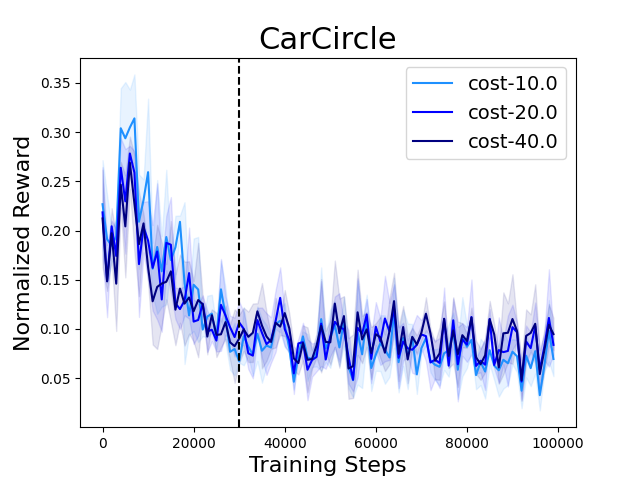}
    \end{subfigure}%
    \begin{subfigure}[t]{0.245\textwidth}
        \centering
        \includegraphics[width=\textwidth]{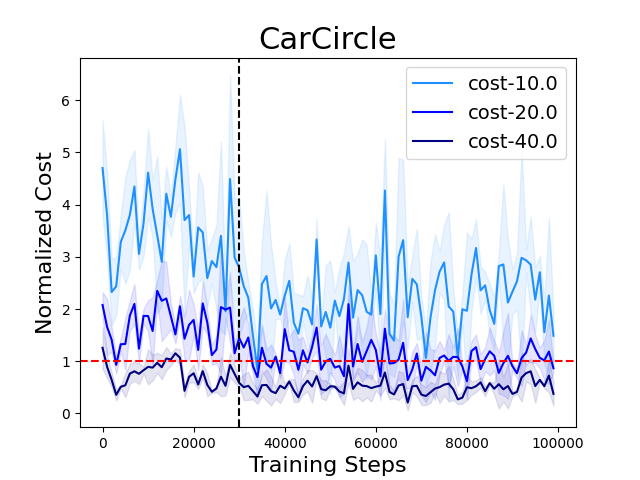}
    \end{subfigure}
    \smallskip
    \begin{subfigure}[t]{0.245\textwidth}
        \centering
        \includegraphics[width=\textwidth]{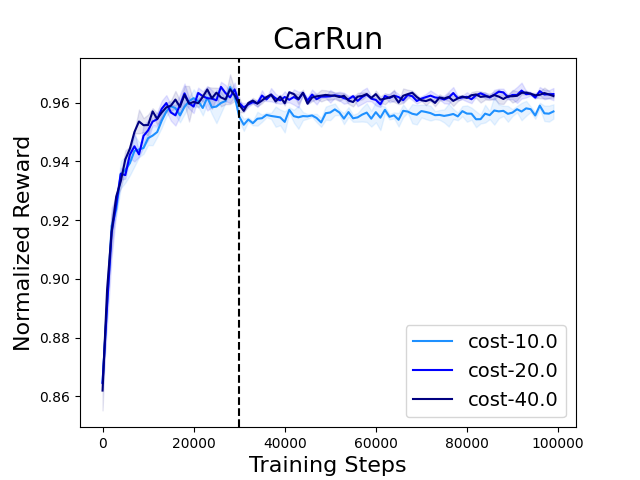}
    \end{subfigure}
    \begin{subfigure}[t]{0.245\textwidth}
        \centering
        \includegraphics[width=\textwidth]{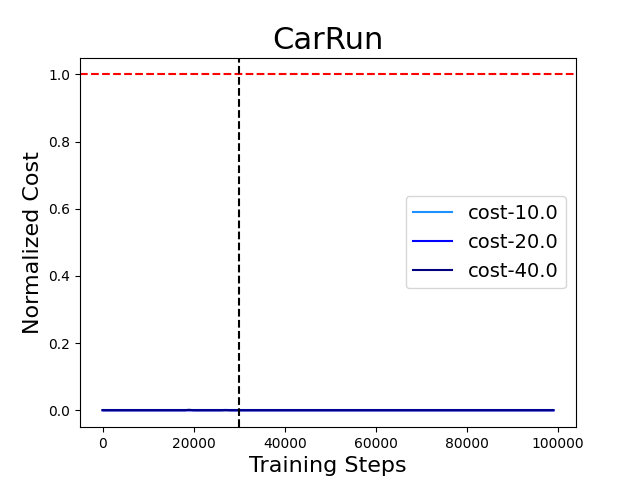}
    \end{subfigure}
    
    \begin{subfigure}[t]{0.245\textwidth}
        \centering
        \includegraphics[width=\textwidth]{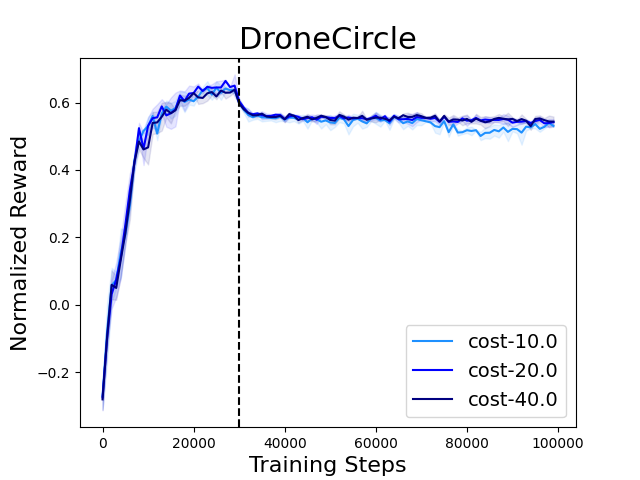}
    \end{subfigure}
    \begin{subfigure}[t]{0.245\textwidth}
        \centering
        \includegraphics[width=\textwidth]{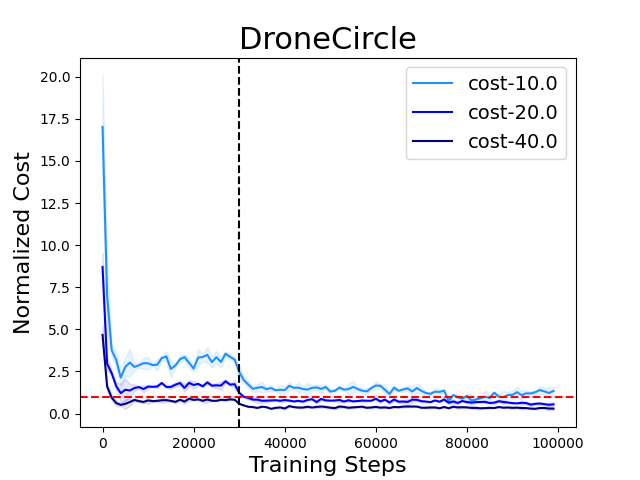}
    \end{subfigure}
    \smallskip
    \begin{subfigure}[t]{0.245\textwidth}
        \centering
        \includegraphics[width=\textwidth]{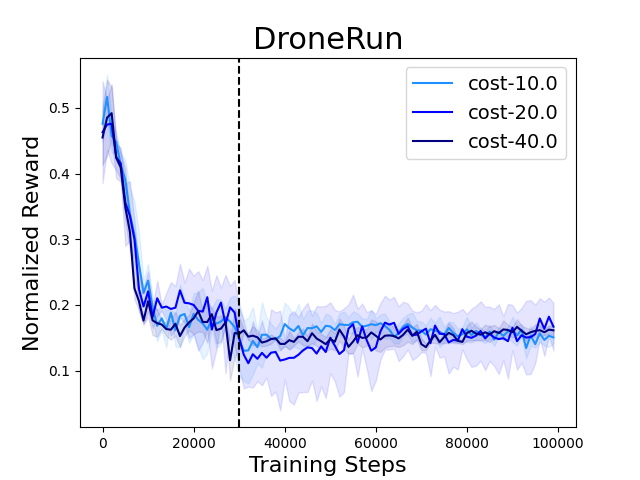}
    \end{subfigure}
    \begin{subfigure}[t]{0.245\textwidth}
        \centering
        \includegraphics[width=\textwidth]{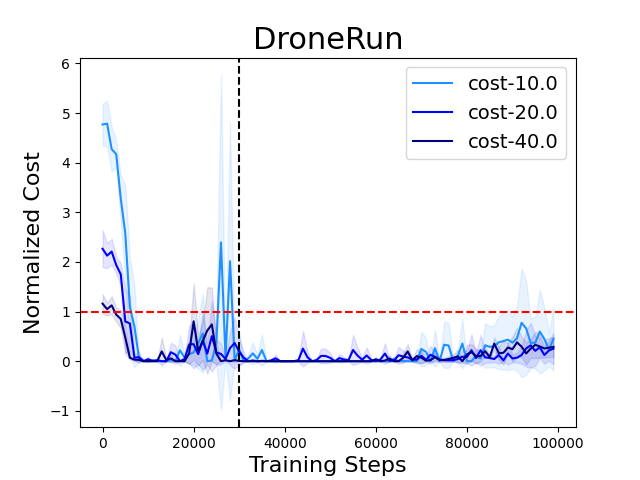}
    \end{subfigure}
    \caption{Training curves for the 8 tasks in BulletGym.}
    \label{fig:bulletgym_curve}
\end{figure*}

\begin{figure*}
    \centering
    \begin{subfigure}[t]{0.245\textwidth}
        \centering
        \includegraphics[width=\textwidth]{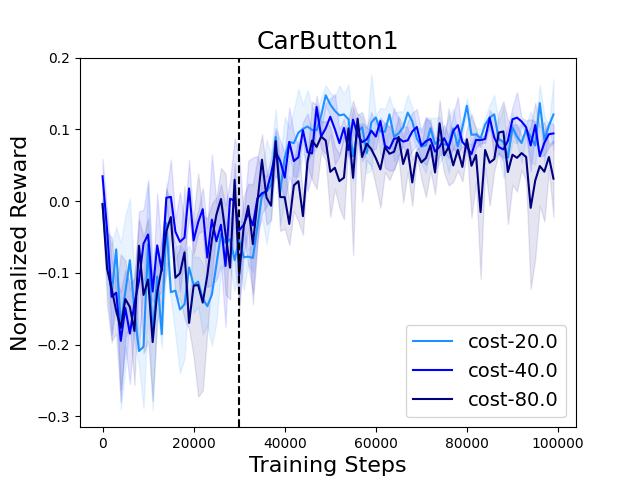}
    \end{subfigure}%
    \begin{subfigure}[t]{0.245\textwidth}
        \centering
        \includegraphics[width=\textwidth]{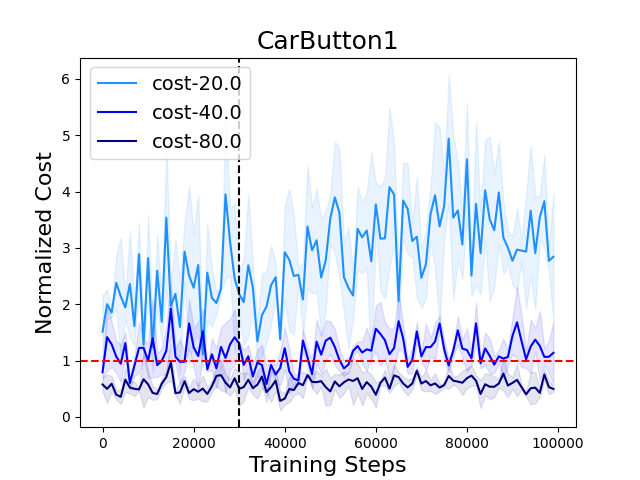}
    \end{subfigure}
    \smallskip
    \begin{subfigure}[t]{0.245\textwidth}
        \centering
        \includegraphics[width=\textwidth]{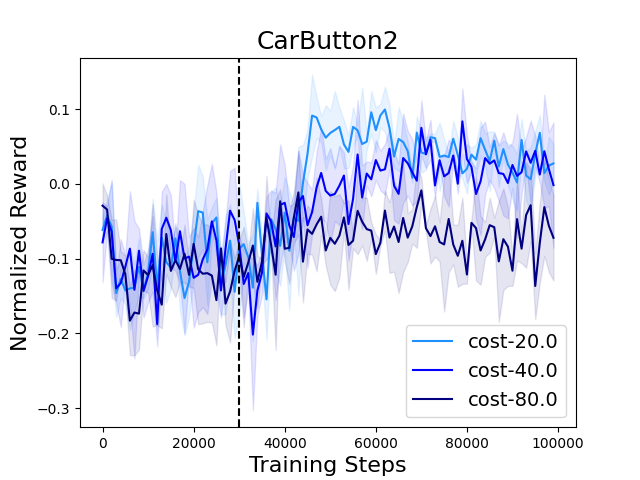}
    \end{subfigure}
    \begin{subfigure}[t]{0.245\textwidth}
        \centering
        \includegraphics[width=\textwidth]{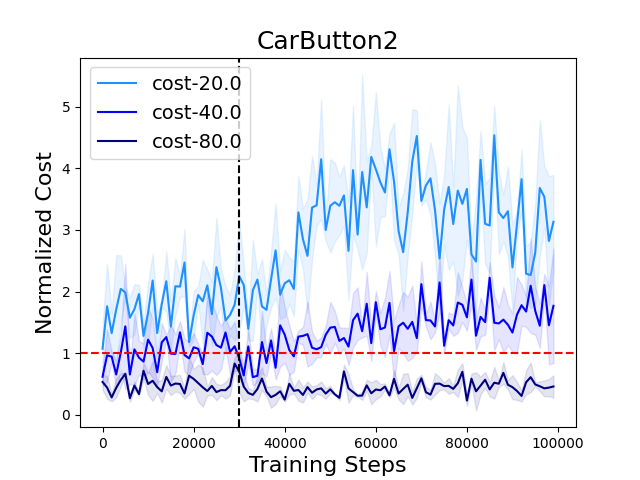}
    \end{subfigure}
    
    \begin{subfigure}[t]{0.245\textwidth}
        \centering
        \includegraphics[width=\textwidth]{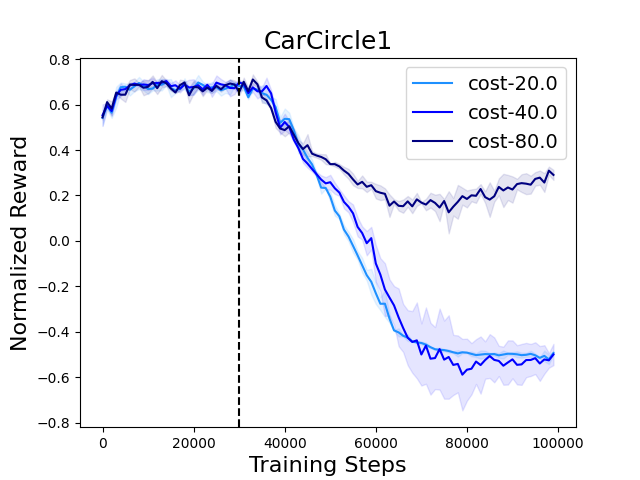}
    \end{subfigure}
    \begin{subfigure}[t]{0.245\textwidth}
        \centering
        \includegraphics[width=\textwidth]{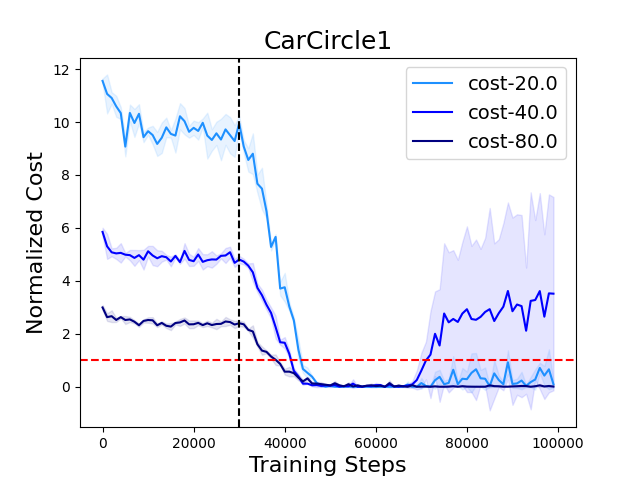}
    \end{subfigure}
    \smallskip
    \begin{subfigure}[t]{0.245\textwidth}
        \centering
        \includegraphics[width=\textwidth]{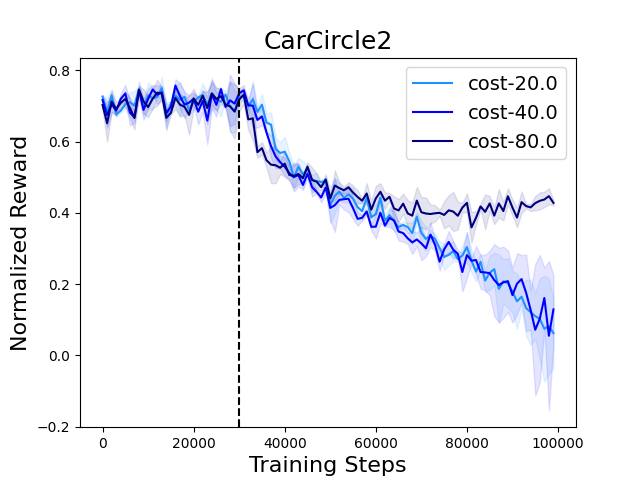}
    \end{subfigure}
    \begin{subfigure}[t]{0.245\textwidth}
        \centering
        \includegraphics[width=\textwidth]{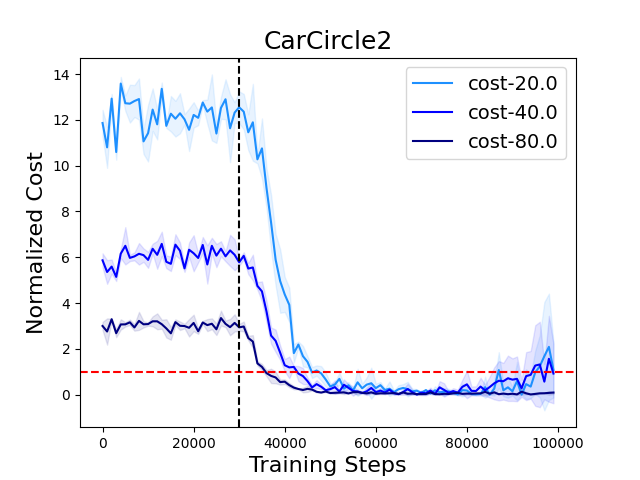}
    \end{subfigure}

    \begin{subfigure}[t]{0.245\textwidth}
        \centering
        \includegraphics[width=\textwidth]{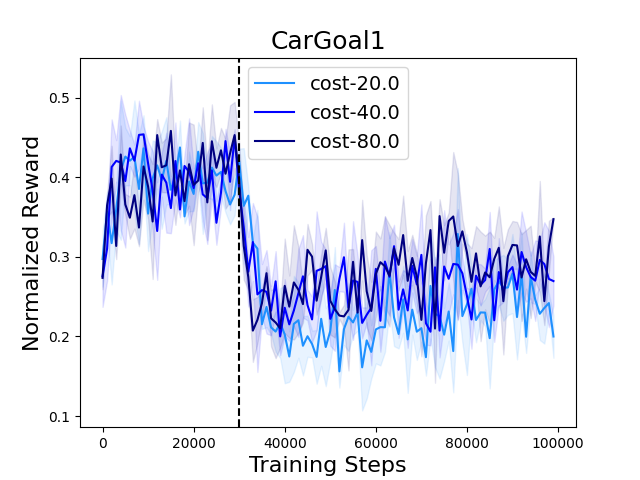}
    \end{subfigure}%
    \begin{subfigure}[t]{0.245\textwidth}
        \centering
        \includegraphics[width=\textwidth]{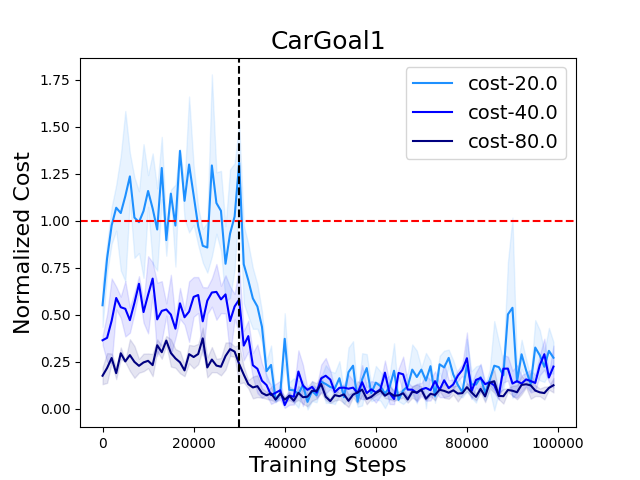}
    \end{subfigure}
    \smallskip
    \begin{subfigure}[t]{0.245\textwidth}
        \centering
        \includegraphics[width=\textwidth]{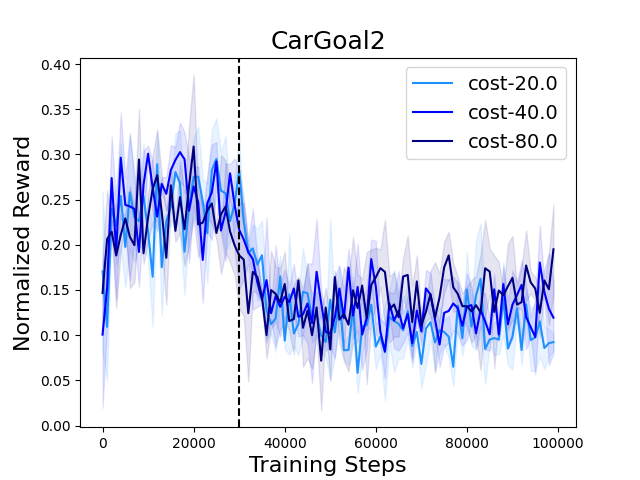}
    \end{subfigure}
    \begin{subfigure}[t]{0.245\textwidth}
        \centering
        \includegraphics[width=\textwidth]{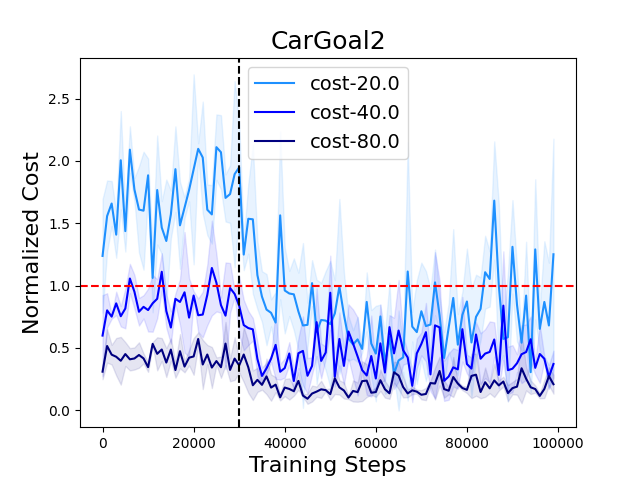}
    \end{subfigure}
    
    \begin{subfigure}[t]{0.245\textwidth}
        \centering
        \includegraphics[width=\textwidth]{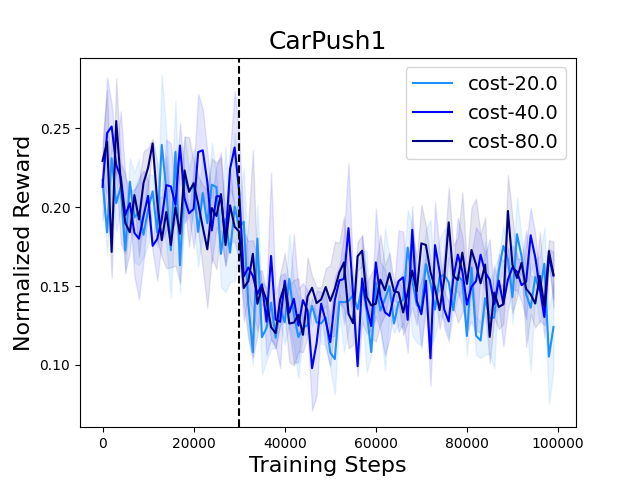}
    \end{subfigure}
    \begin{subfigure}[t]{0.245\textwidth}
        \centering
        \includegraphics[width=\textwidth]{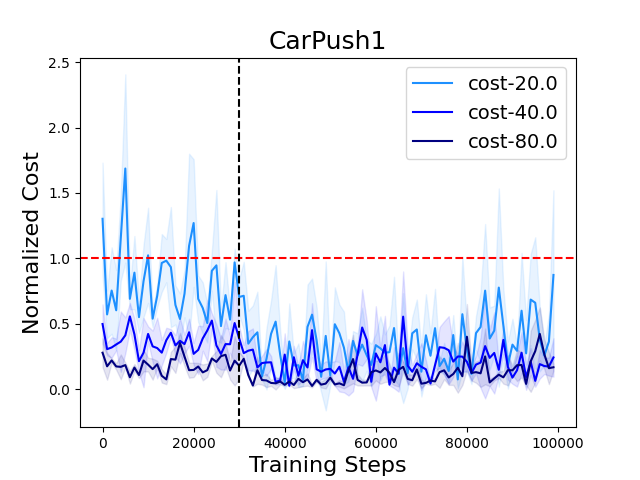}
    \end{subfigure}
    \smallskip
    \begin{subfigure}[t]{0.245\textwidth}
        \centering
        \includegraphics[width=\textwidth]{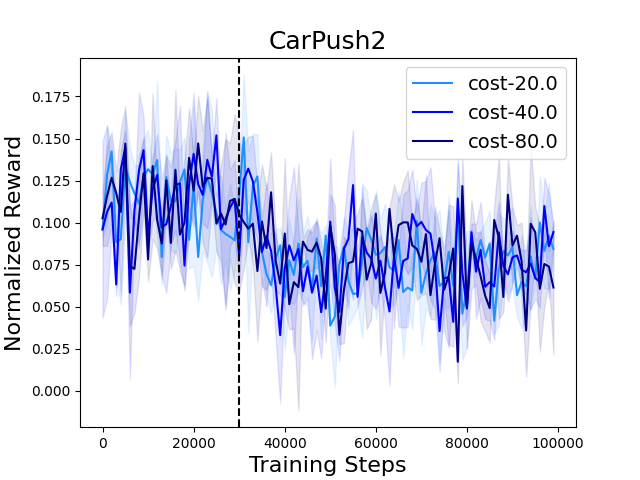}
    \end{subfigure}
    \begin{subfigure}[t]{0.245\textwidth}
        \centering
        \includegraphics[width=\textwidth]{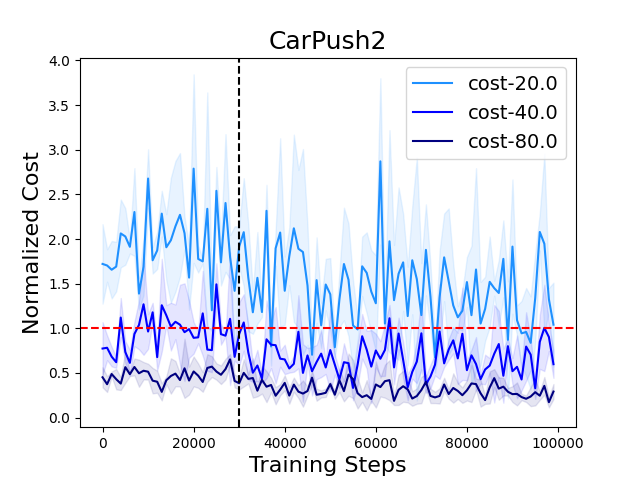}
    \end{subfigure}
    \caption{Learning curves for the 8 Car tasks in SafetyGym.}
    \label{fig:safetgym_curve1}
\end{figure*}

\begin{figure*}
    \centering
    \begin{subfigure}[t]{0.245\textwidth}
        \centering
        \includegraphics[width=\textwidth]{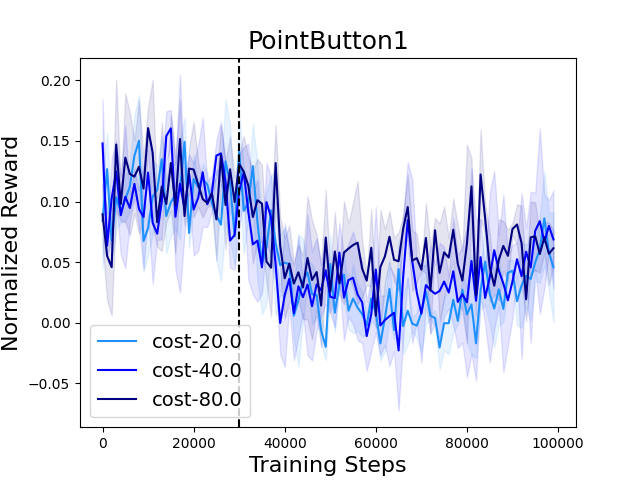}
    \end{subfigure}%
    \begin{subfigure}[t]{0.245\textwidth}
        \centering
        \includegraphics[width=\textwidth]{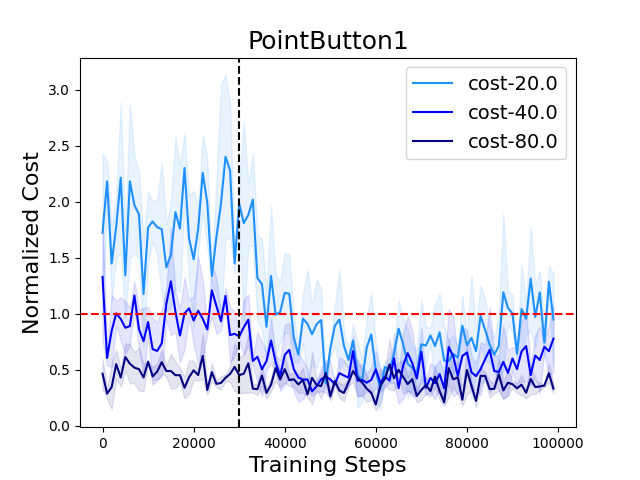}
    \end{subfigure}
    \smallskip
    \begin{subfigure}[t]{0.245\textwidth}
        \centering
        \includegraphics[width=\textwidth]{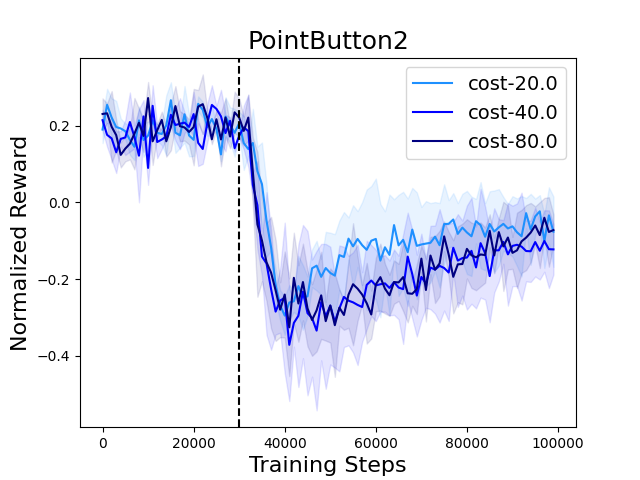}
    \end{subfigure}
    \begin{subfigure}[t]{0.245\textwidth}
        \centering
        \includegraphics[width=\textwidth]{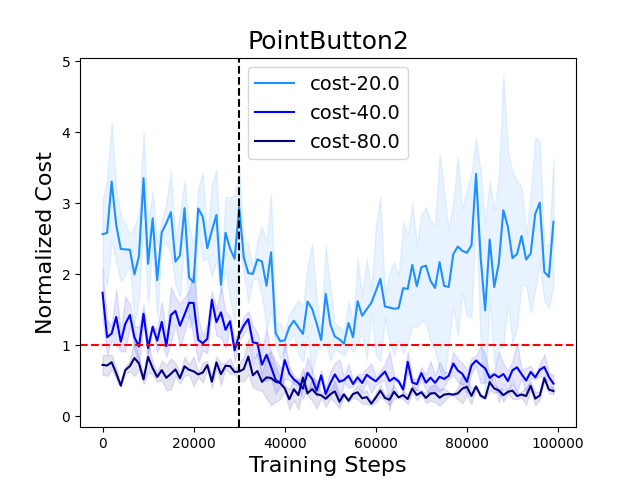}
    \end{subfigure}
    
    \begin{subfigure}[t]{0.245\textwidth}
        \centering
        \includegraphics[width=\textwidth]{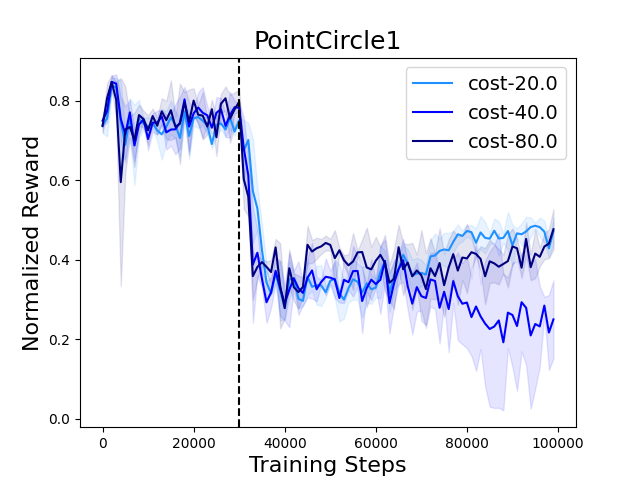}
    \end{subfigure}
    \begin{subfigure}[t]{0.245\textwidth}
        \centering
        \includegraphics[width=\textwidth]{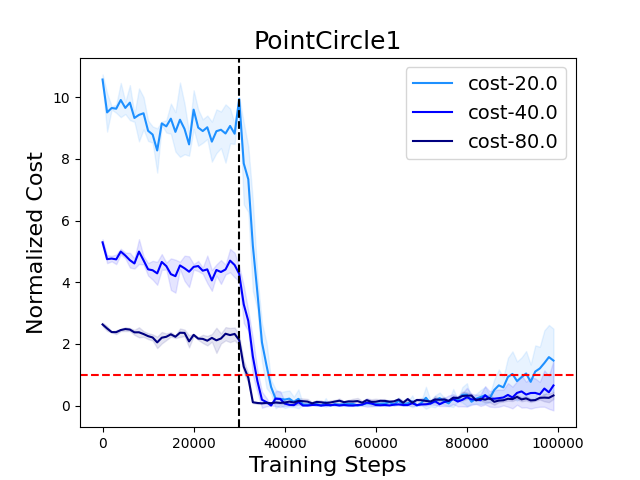}
    \end{subfigure}
    \smallskip
    \begin{subfigure}[t]{0.245\textwidth}
        \centering
        \includegraphics[width=\textwidth]{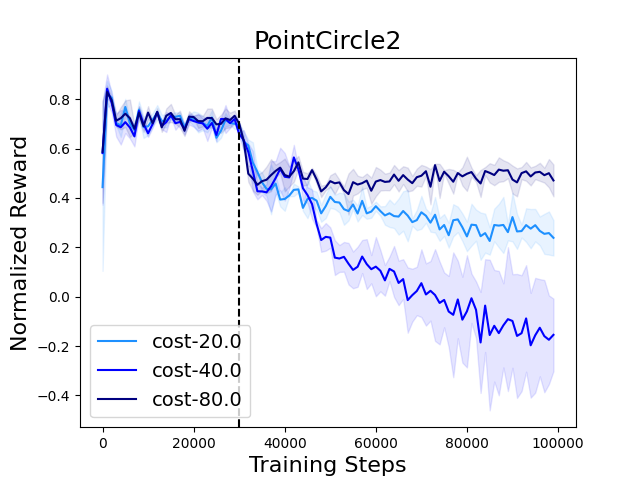}
    \end{subfigure}
    \begin{subfigure}[t]{0.245\textwidth}
        \centering
        \includegraphics[width=\textwidth]{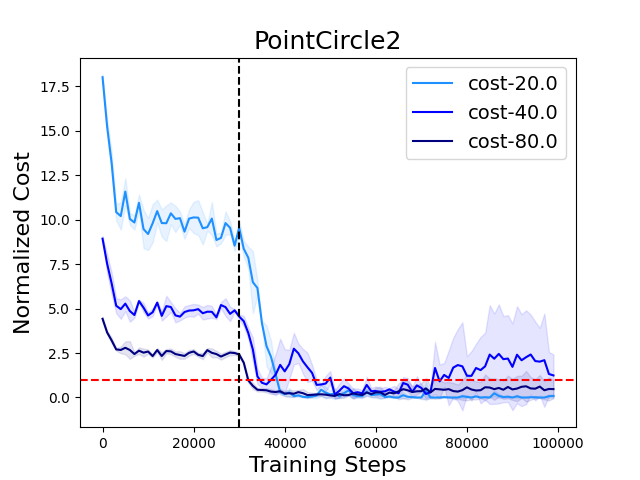}
    \end{subfigure}

    \begin{subfigure}[t]{0.245\textwidth}
        \centering
        \includegraphics[width=\textwidth]{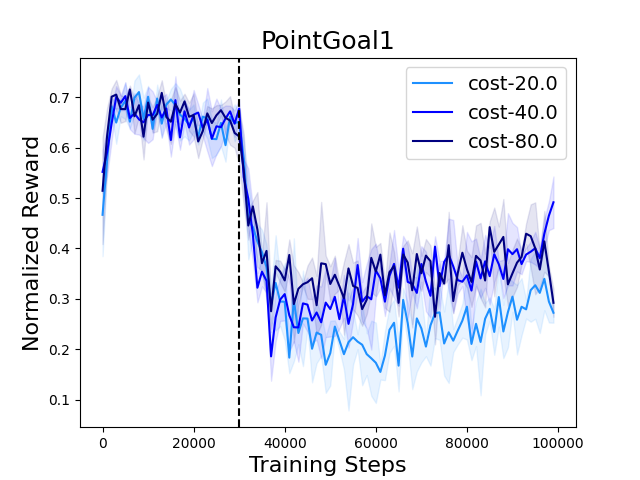}
    \end{subfigure}%
    \begin{subfigure}[t]{0.245\textwidth}
        \centering
        \includegraphics[width=\textwidth]{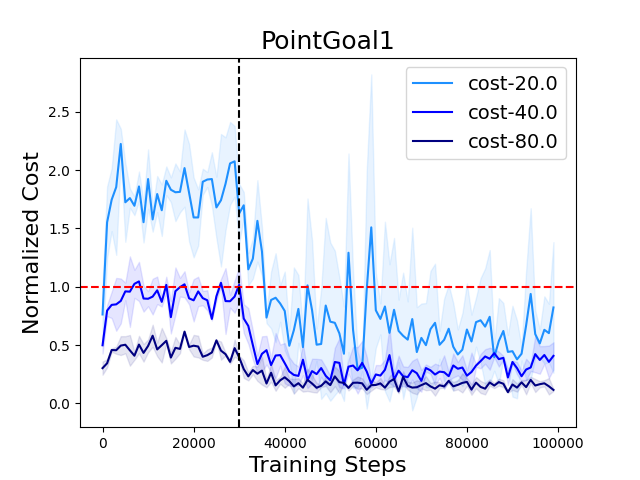}
    \end{subfigure}
    \smallskip
    \begin{subfigure}[t]{0.245\textwidth}
        \centering
        \includegraphics[width=\textwidth]{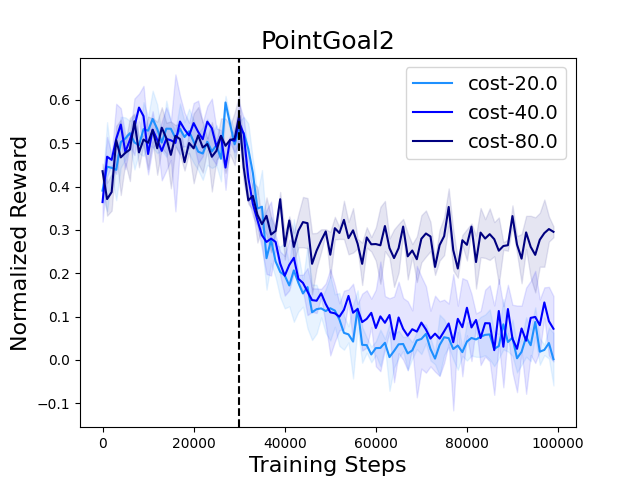}
    \end{subfigure}
    \begin{subfigure}[t]{0.245\textwidth}
        \centering
        \includegraphics[width=\textwidth]{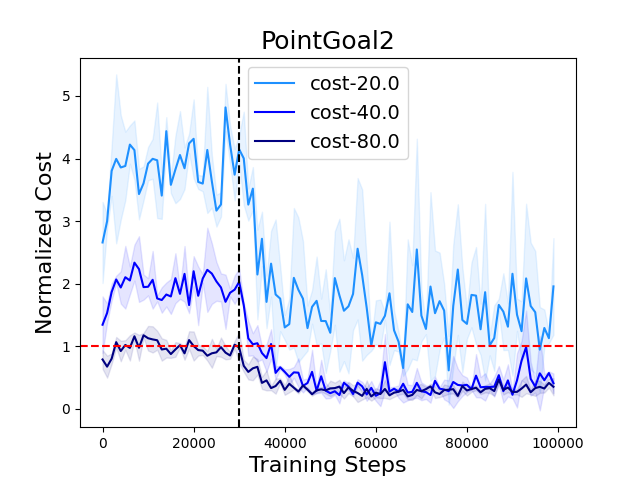}
    \end{subfigure}
    
    \begin{subfigure}[t]{0.245\textwidth}
        \centering
        \includegraphics[width=\textwidth]{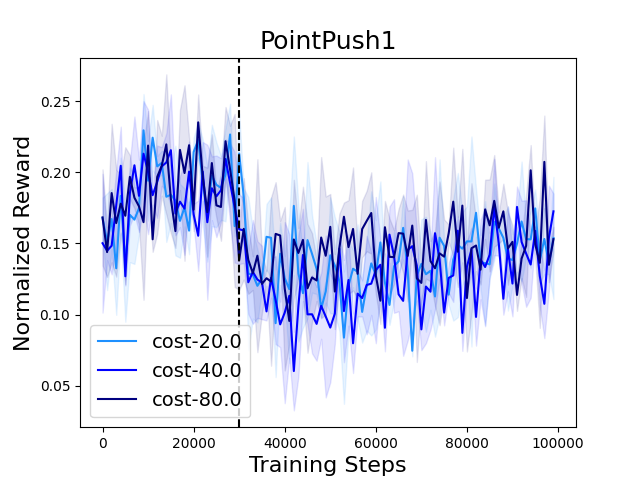}
    \end{subfigure}
    \begin{subfigure}[t]{0.245\textwidth}
        \centering
        \includegraphics[width=\textwidth]{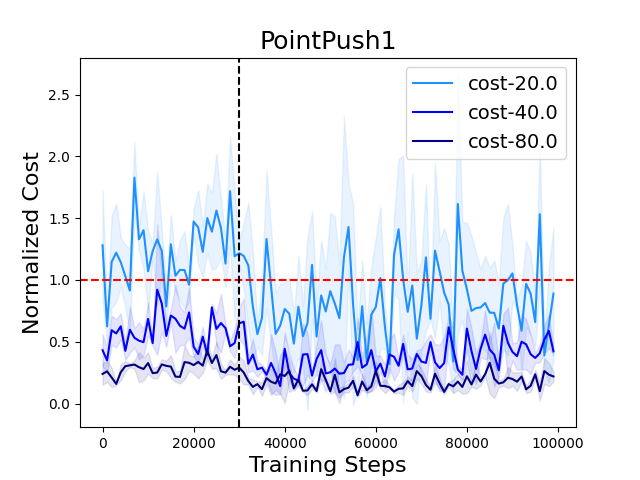}
    \end{subfigure}
    \smallskip
    \begin{subfigure}[t]{0.245\textwidth}
        \centering
        \includegraphics[width=\textwidth]{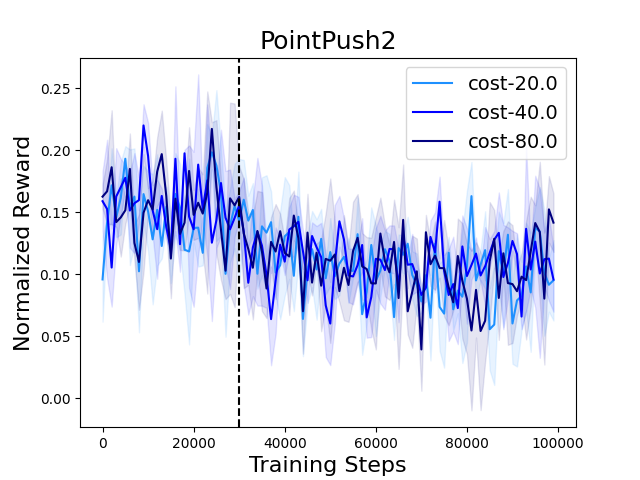}
    \end{subfigure}
    \begin{subfigure}[t]{0.245\textwidth}
        \centering
        \includegraphics[width=\textwidth]{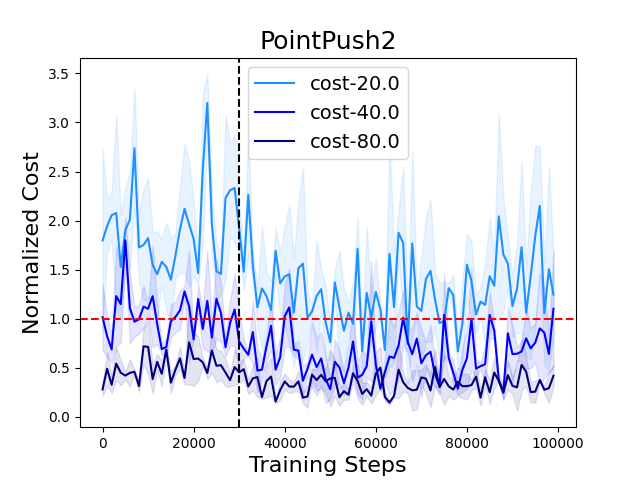}
    \end{subfigure}
    \caption{Learning curves for the 8 Point tasks in SafetyGym.}
    \label{fig:safetgym_curve2}
\end{figure*}

\begin{figure*}
    \centering
    \begin{subfigure}[t]{0.245\textwidth}
        \centering
        \includegraphics[width=\textwidth]{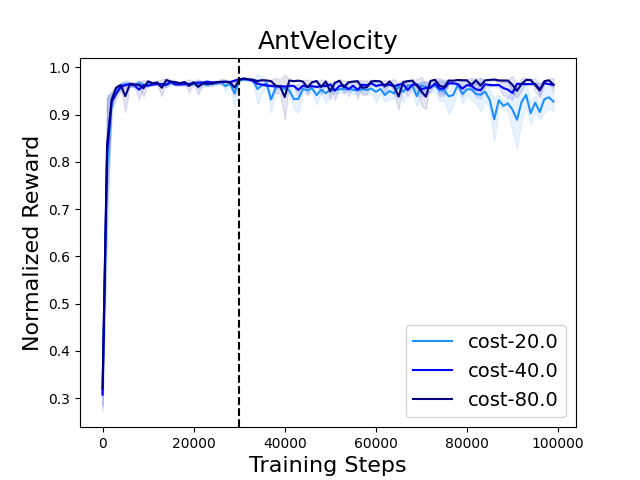}
    \end{subfigure}%
    \begin{subfigure}[t]{0.245\textwidth}
        \centering
        \includegraphics[width=\textwidth]{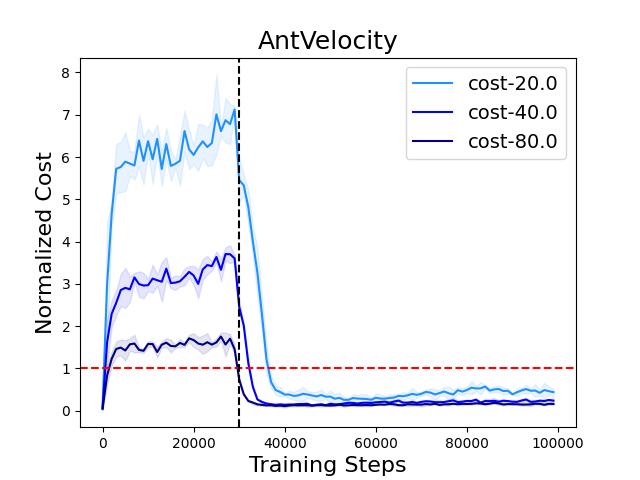}
    \end{subfigure}
    \smallskip
    \begin{subfigure}[t]{0.245\textwidth}
        \centering
        \includegraphics[width=\textwidth]{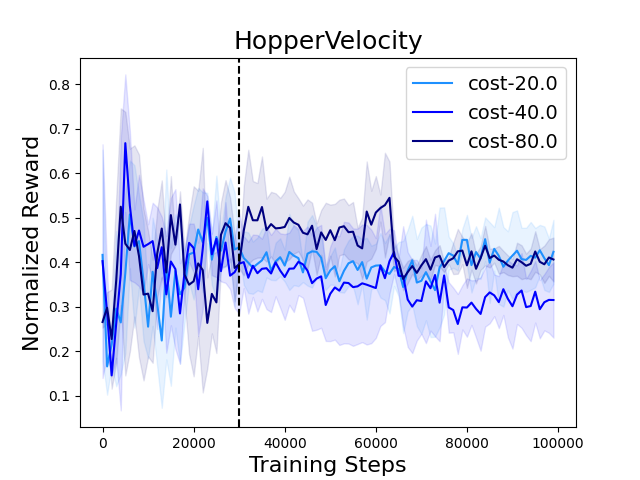}
    \end{subfigure}
    \begin{subfigure}[t]{0.245\textwidth}
        \centering
        \includegraphics[width=\textwidth]{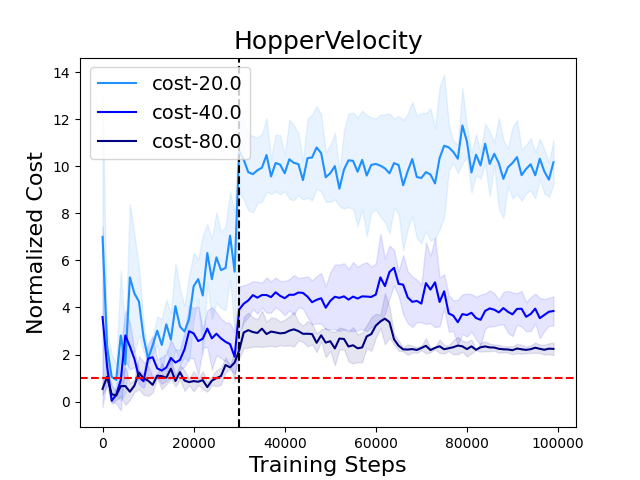}
    \end{subfigure}
    
    \begin{subfigure}[t]{0.245\textwidth}
        \centering
        \includegraphics[width=\textwidth]{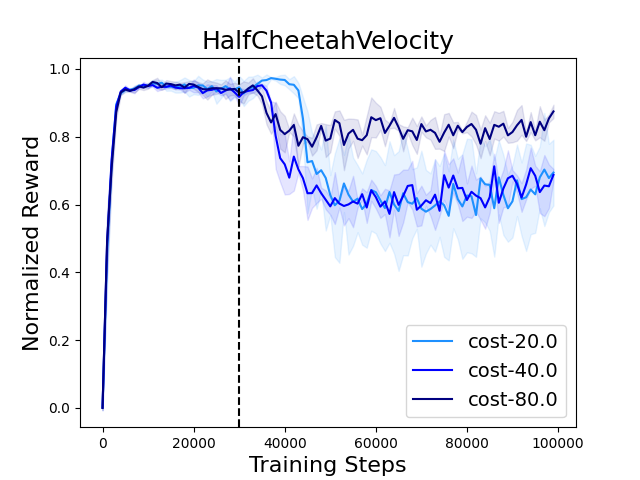}
    \end{subfigure}
    \begin{subfigure}[t]{0.245\textwidth}
        \centering
        \includegraphics[width=\textwidth]{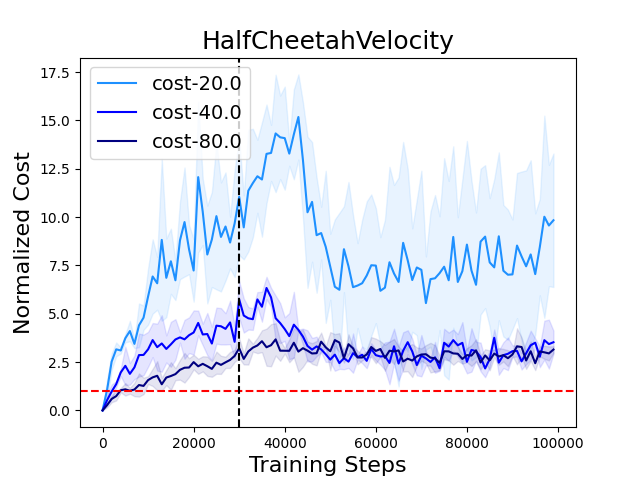}
    \end{subfigure}
    \smallskip
    \begin{subfigure}[t]{0.245\textwidth}
        \centering
        \includegraphics[width=\textwidth]{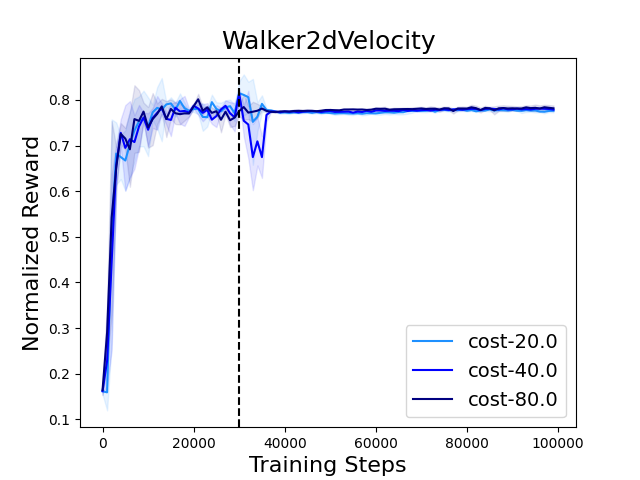}
    \end{subfigure}
    \begin{subfigure}[t]{0.245\textwidth}
        \centering
        \includegraphics[width=\textwidth]{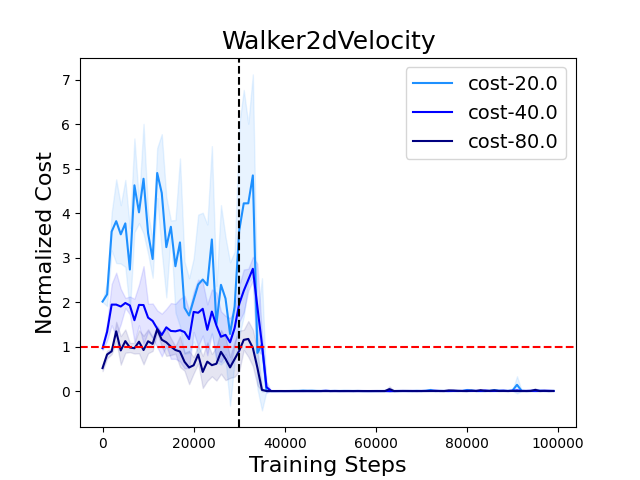}
    \end{subfigure}

    \begin{subfigure}[t]{0.245\textwidth}
        \centering
        \includegraphics[width=\textwidth]{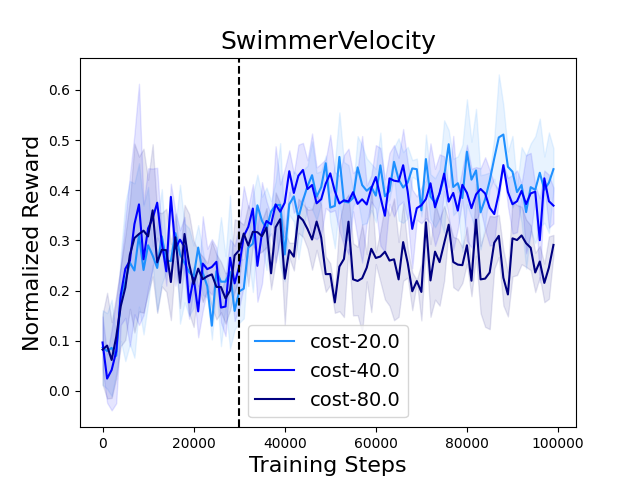}
    \end{subfigure}%
    \begin{subfigure}[t]{0.245\textwidth}
        \centering
        \includegraphics[width=\textwidth]{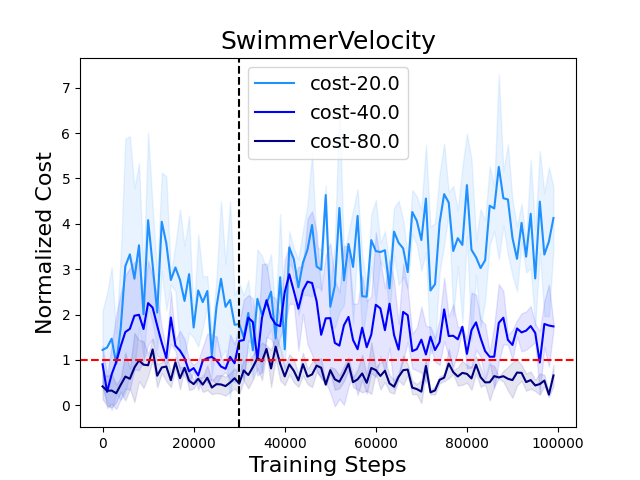}
    \end{subfigure}
    \caption{Learning curves for the 5 velocity constraint tasks in SafetyGym.}
    \label{fig:safetgym_curve3}
\end{figure*}

%%%%%%%%%%%%%%%%%%%%%%%%%%%%%%%%%%%%%%%%%%%%%%%%%%%%%%%%%%%%%%%%%%%%%%%%

%%% The acknowledgments section is defined using the "acks" environment
%%% (rather than an unnumbered section). The use of this environment 
%%% ensures the proper identification of the section in the article 
%%% metadata as well as the consistent spelling of the heading.

% \begin{acks}
% If you wish to include any acknowledgments in your paper (e.g., to 
% people or funding agencies), please do so using the `\texttt{acks}' 
% environment. Note that the text of your acknowledgments will be omitted
% if you compile your document with the `\texttt{anonymous}' option.
% \end{acks}

%%%%%%%%%%%%%%%%%%%%%%%%%%%%%%%%%%%%%%%%%%%%%%%%%%%%%%%%%%%%%%%%%%%%%%%%

%%% The next two lines define, first, the bibliography style to be 
%%% applied, and, second, the bibliography file to be used.

% \bibliographystyle{ACM-Reference-Format} 
% \bibliography{sample}

\end{document}